\documentclass[journal]{IEEEtran}
\ifCLASSINFOpdf
  \usepackage[pdftex]{graphicx}
\else
\fi
\usepackage{amsmath}
\usepackage{algorithmic}

\usepackage{array}

\usepackage{subfigure}
\usepackage{float}
\usepackage{multirow}
\usepackage{threeparttable}
\usepackage{bm}
\usepackage{color}
\usepackage{times}
\usepackage{soul}
\usepackage{url}
\usepackage[hidelinks]{hyperref}
\usepackage[utf8]{inputenc}
\usepackage{caption}
\usepackage{graphicx}
\usepackage{booktabs}
\usepackage{algorithm}
\usepackage{algorithmic}
\usepackage{bm}
\usepackage{amsmath}
\usepackage{amssymb}
\usepackage{amsthm}

\newtheorem{proposition}{Proposition}
\newcommand{\tabincell}[2]{\begin{tabular}{@{}#1@{}}#2\end{tabular}}
\begin{document}
%
\title{Personal Privacy Protection via Irrelevant Faces Tracking and Pixelation in Video Live Streaming}
%
%
%

\author{Jizhe~Zhou,~\IEEEmembership{Student Member,~IEEE,}
        Chi-Man~Pun,~\IEEEmembership{Senior Member,~IEEE}}


%
%

\markboth{Journal of \LaTeX\ Class Files,~Vol.~14, No.~8, August~2015}%
{Shell \MakeLowercase{\textit{et al.}}: Bare Demo of IEEEtran.cls for IEEE Journals}
%



\maketitle



%
\IEEEpeerreviewmaketitle

%
%
%
%
\begin{abstract}
To date, the privacy-protection intended pixelation tasks are still labor-intensive and yet to be studied. With the prevailing of video live streaming, establishing an online face pixelation mechanism during streaming is an urgency. In this paper, we develop a new method called Face Pixelation in Video Live Streaming (FPVLS) to generate automatic personal privacy filtering during unconstrained streaming activities. Simply applying multi-face trackers will encounter problems in target drifting, computing efficiency, and over-pixelation. Therefore, for fast and accurate pixelation of irrelevant people's faces, FPVLS is organized in a frame-to-video structure of two core stages. On individual frames, FPVLS utilizes image-based face detection and embedding networks to yield face vectors. In the raw trajectories generation stage, the proposed Positioned Incremental Affinity Propagation (PIAP) clustering algorithm leverages face vectors and positioned information to quickly associate the same person's faces across frames. Such frame-wise accumulated raw trajectories are likely to be intermittent and unreliable on video level. Hence, we further introduce the trajectory refinement stage that merges a proposal network with the two-sample test based on the Empirical Likelihood Ratio (ELR) statistic to refine the raw trajectories. A Gaussian filter is laid on the refined trajectories for final pixelation. On the video live streaming dataset we collected, FPVLS obtains satisfying accuracy, real-time efficiency, and contains the over-pixelation problems.
\end{abstract}
\begin{IEEEkeywords}
face pixelation, video live streaming, privacy protection, positioned incremental affinity propagation, empirical likelihood ratio
\end{IEEEkeywords}
\section{Introduction}

\begin{figure*}[tbp]
\centering
\vspace{-0.35cm}
\subfigtopskip=2pt
\subfigbottomskip=2pt
\subfigcapskip=-3pt
\subfigure
{
\includegraphics[width=3.45cm]{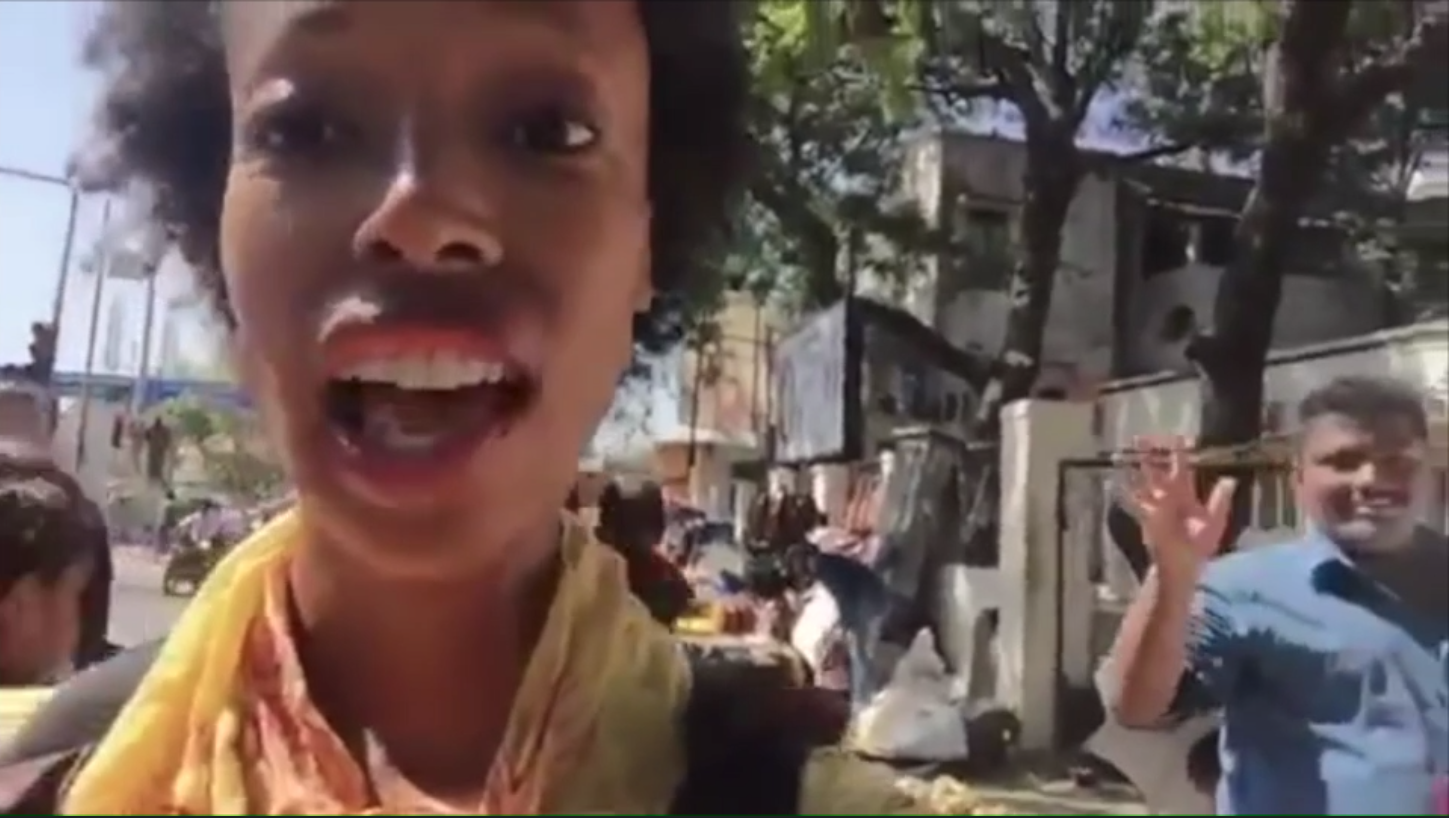}\hspace {-0.1em}
\includegraphics[width=3.45cm]{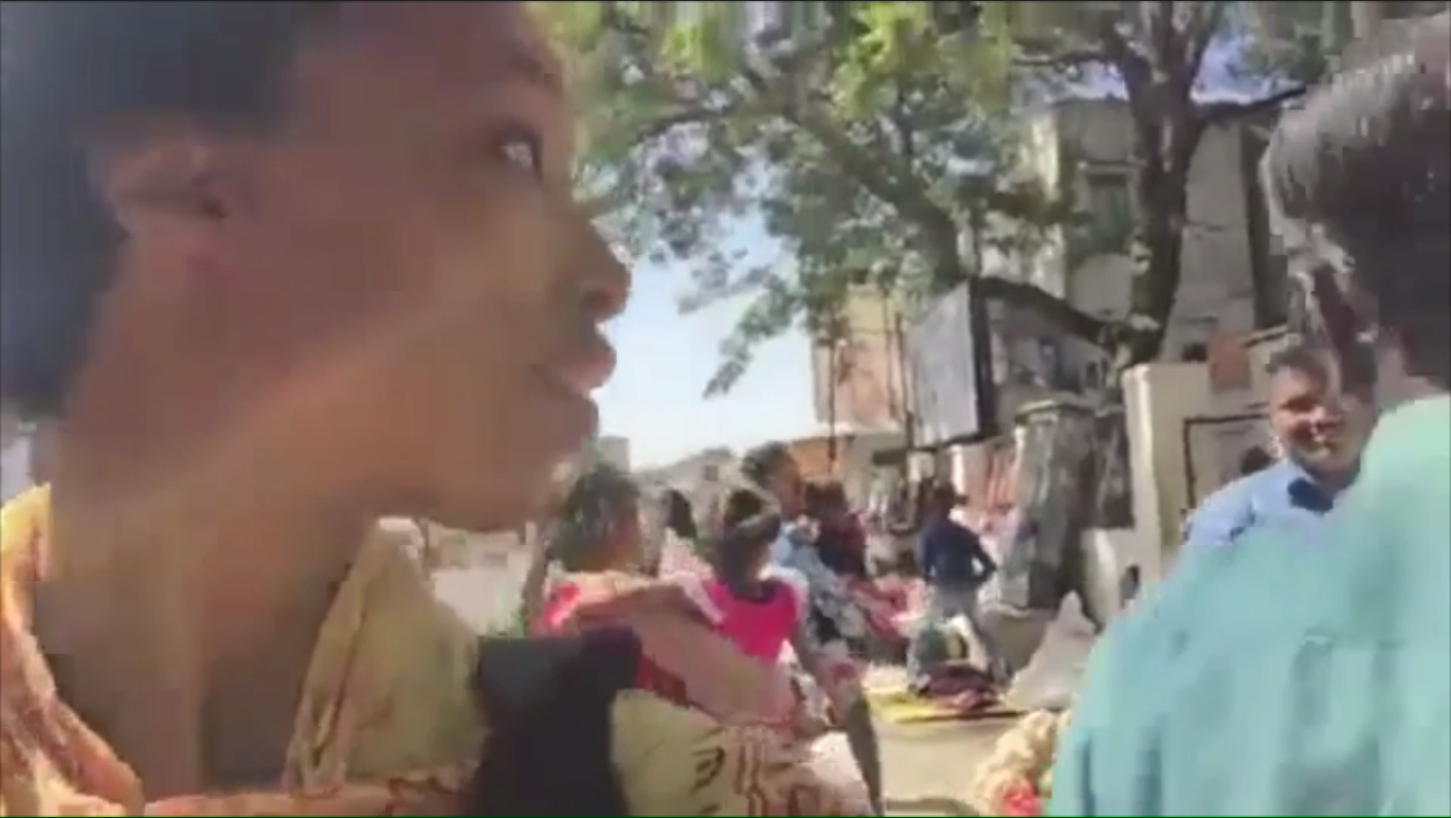}\hspace {-0.1em}
\includegraphics[width=3.45cm]{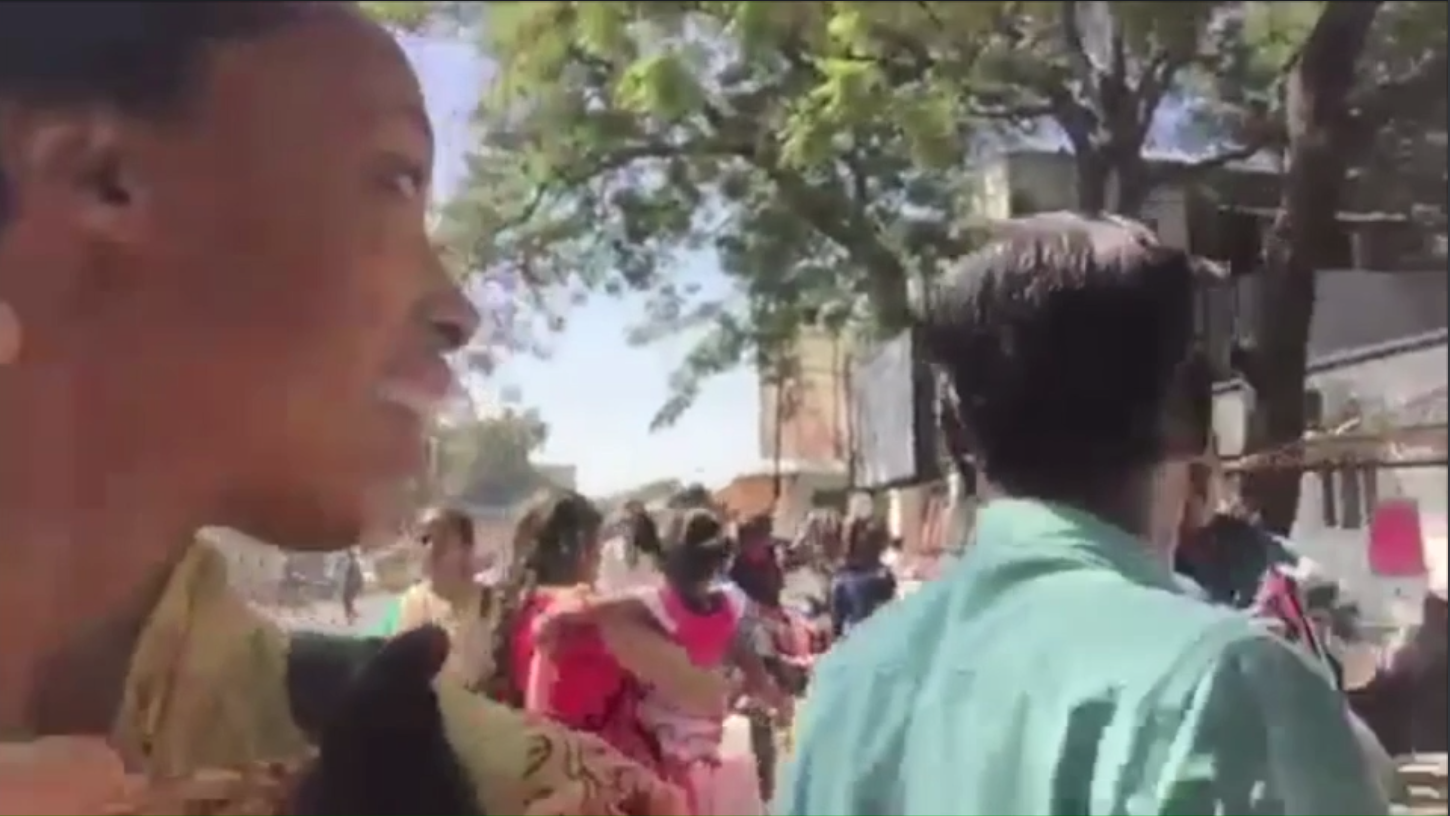}\hspace {-0.1em}
\includegraphics[width=3.45cm]{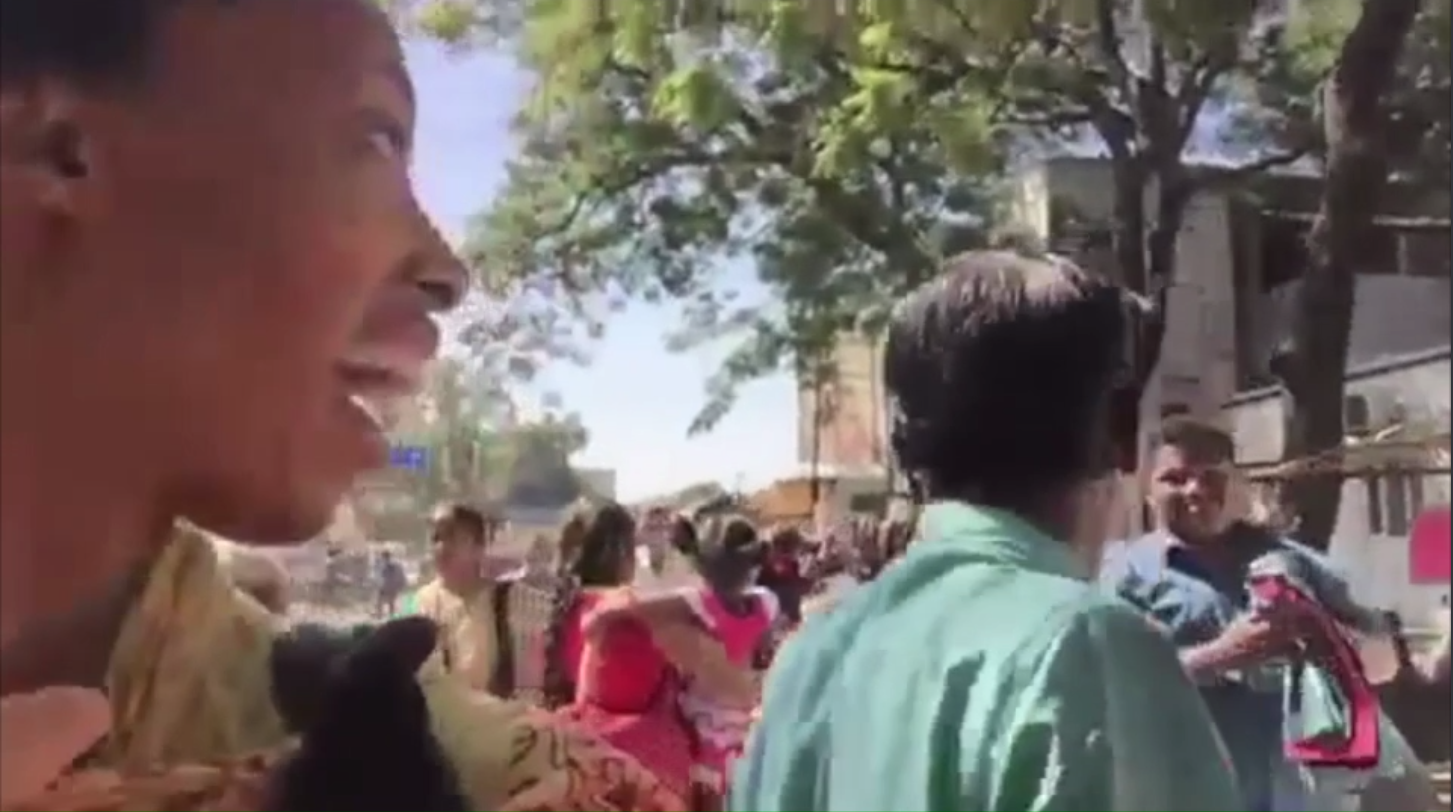}\hspace {-0.1em}
\includegraphics[width=3.46cm]{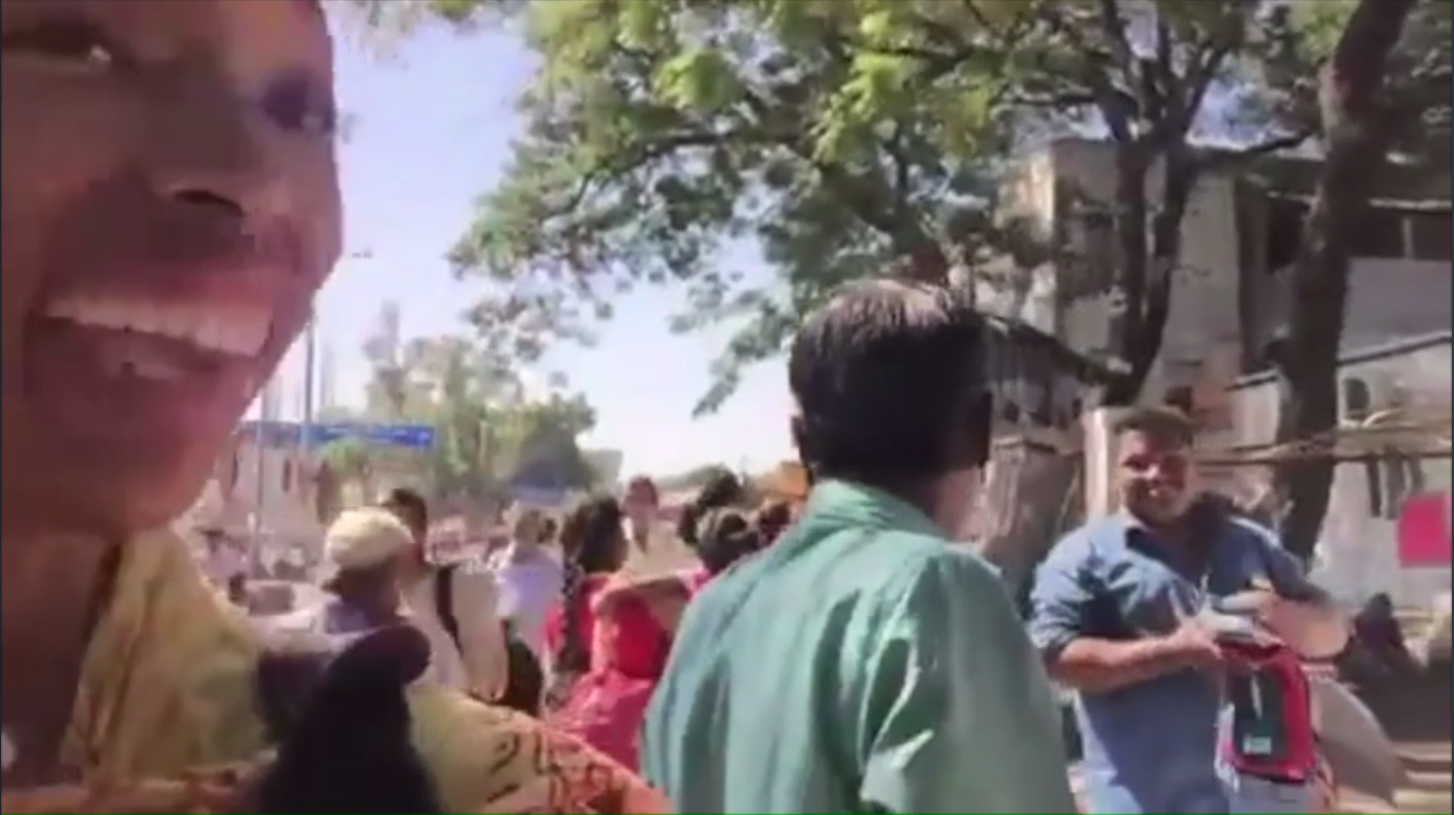}
}

\subfigure
{
\includegraphics[width=3.45cm]{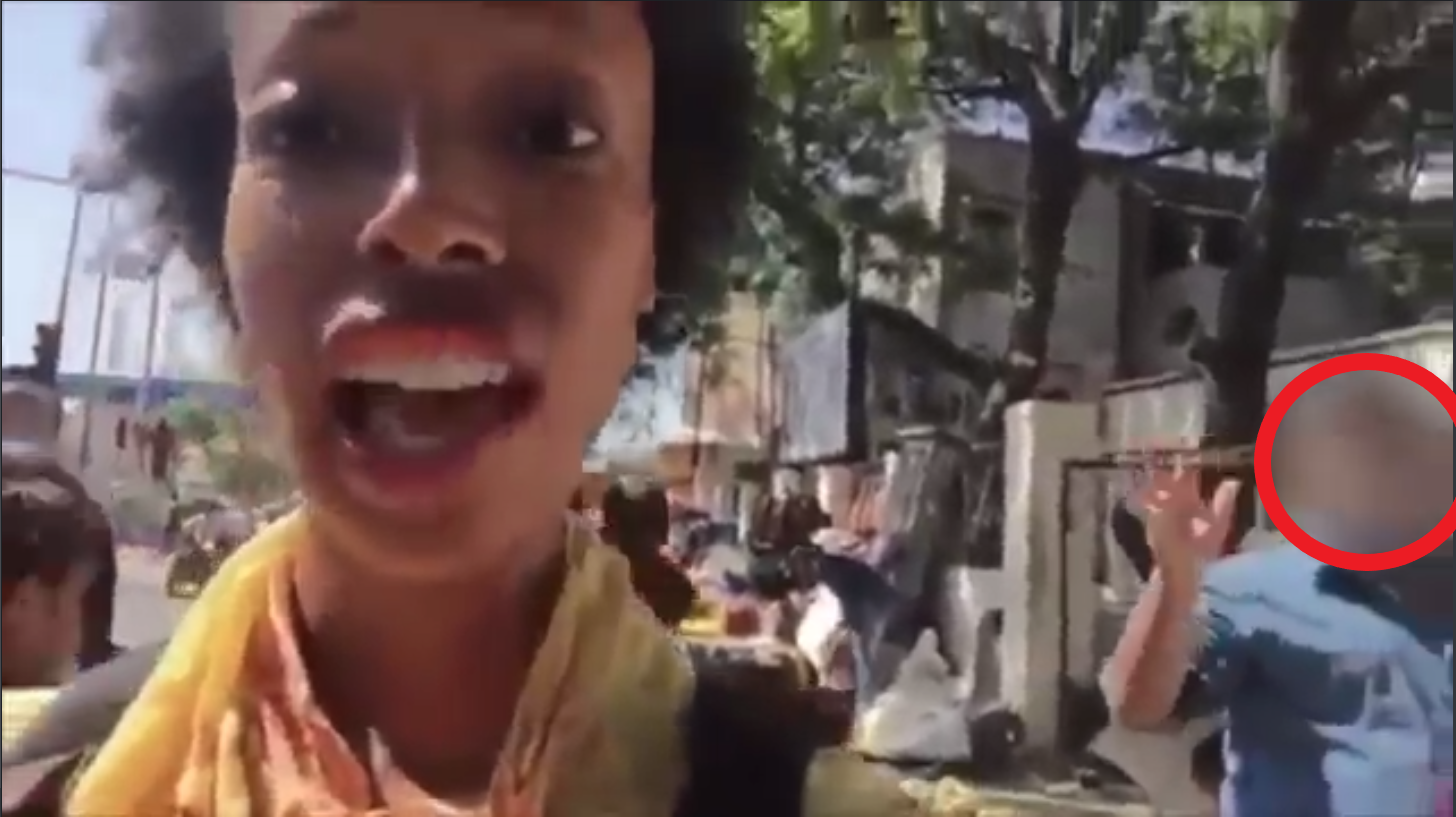}\hspace {-0.1em}
\includegraphics[width=3.45cm]{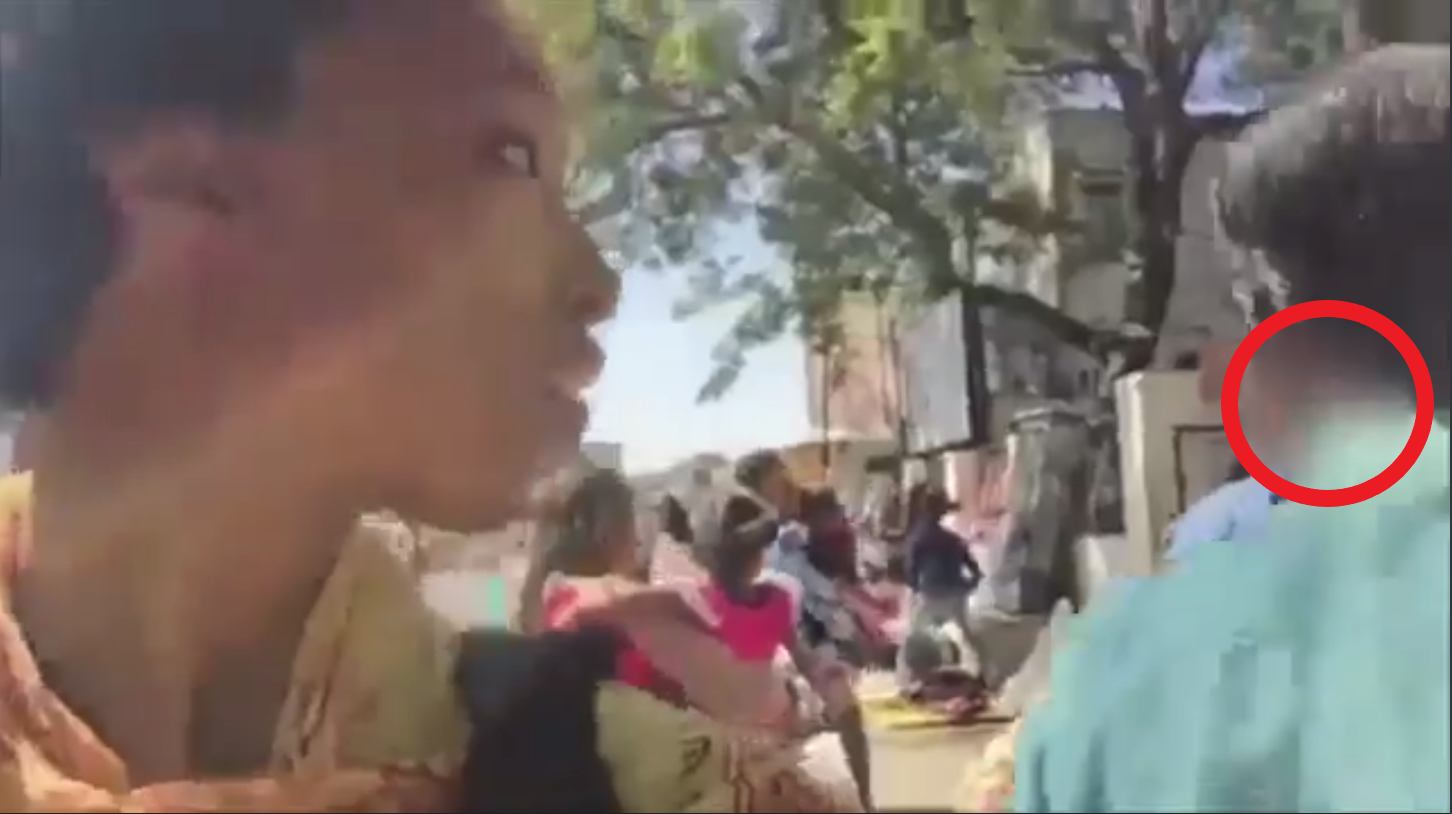}\hspace {-0.1em}
\includegraphics[width=3.45cm]{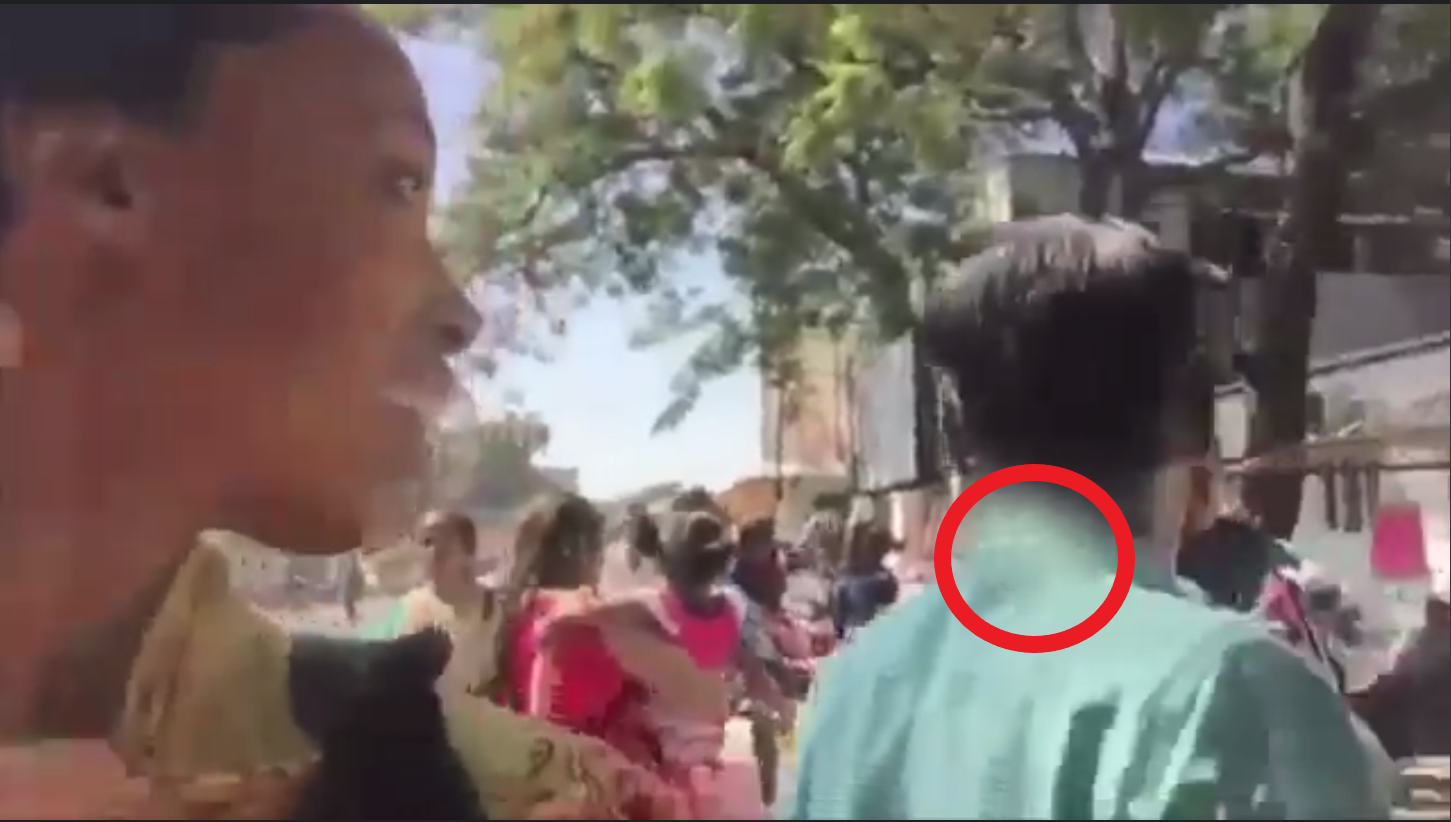}\hspace {-0.1em}
\includegraphics[width=3.45cm]{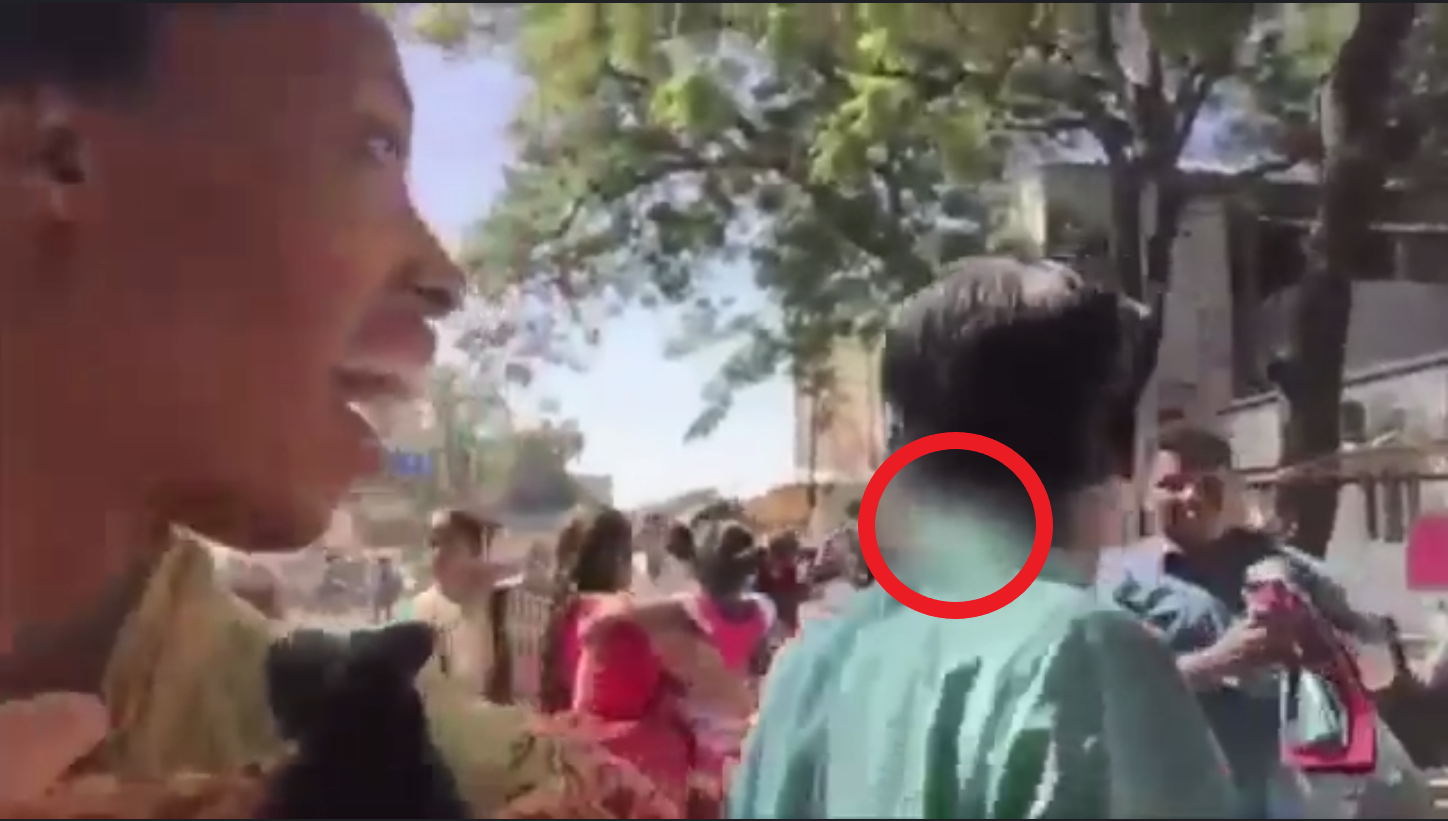}\hspace {-0.1em}
\includegraphics[width=3.45cm]{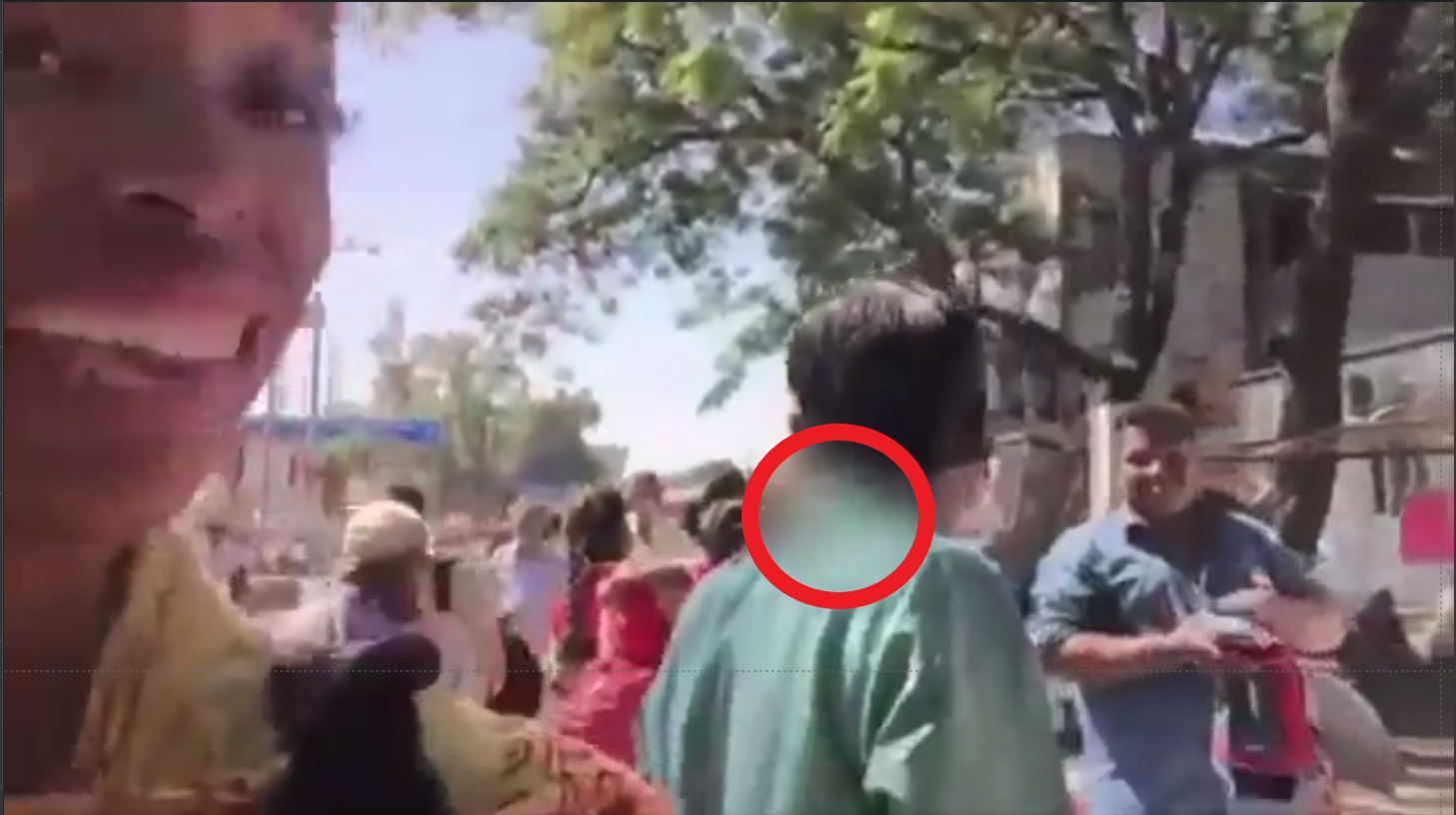}
}\hskip -1pt
\subfigure
{
\includegraphics[width=3.45cm]{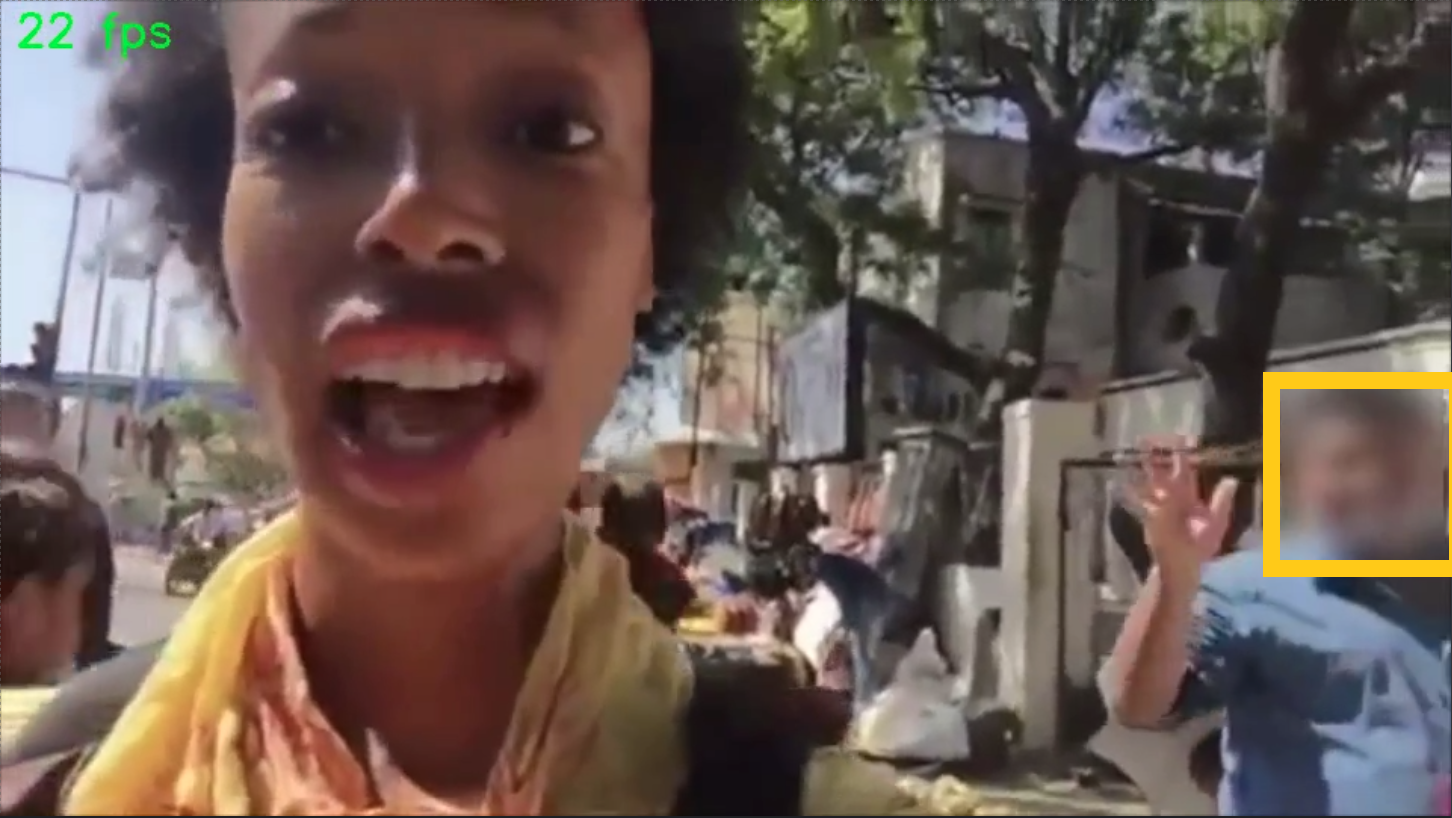}\hspace {-0.1em}
\includegraphics[width=3.45cm]{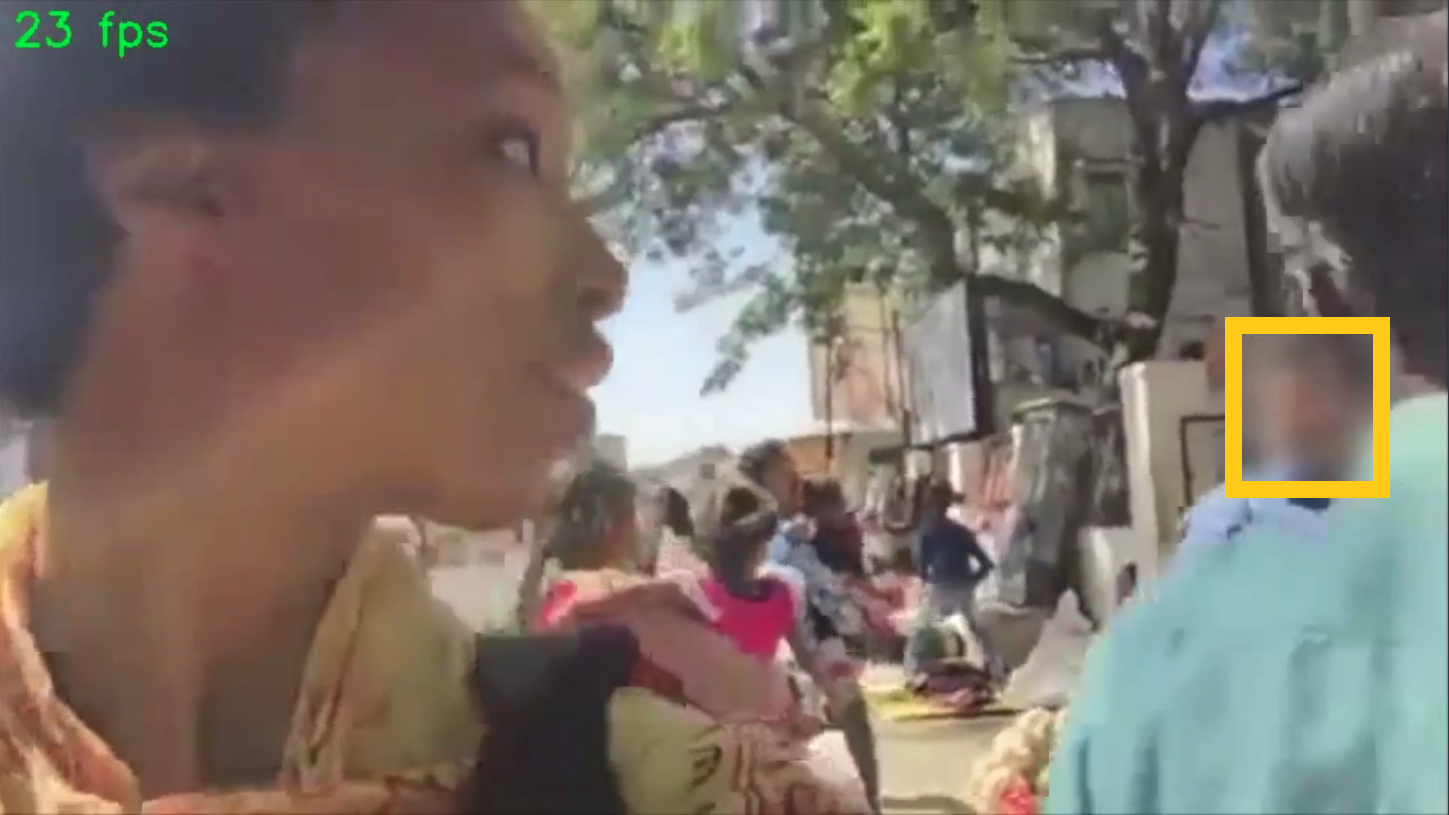}\hspace {-0.1em}
\includegraphics[width=3.45cm]{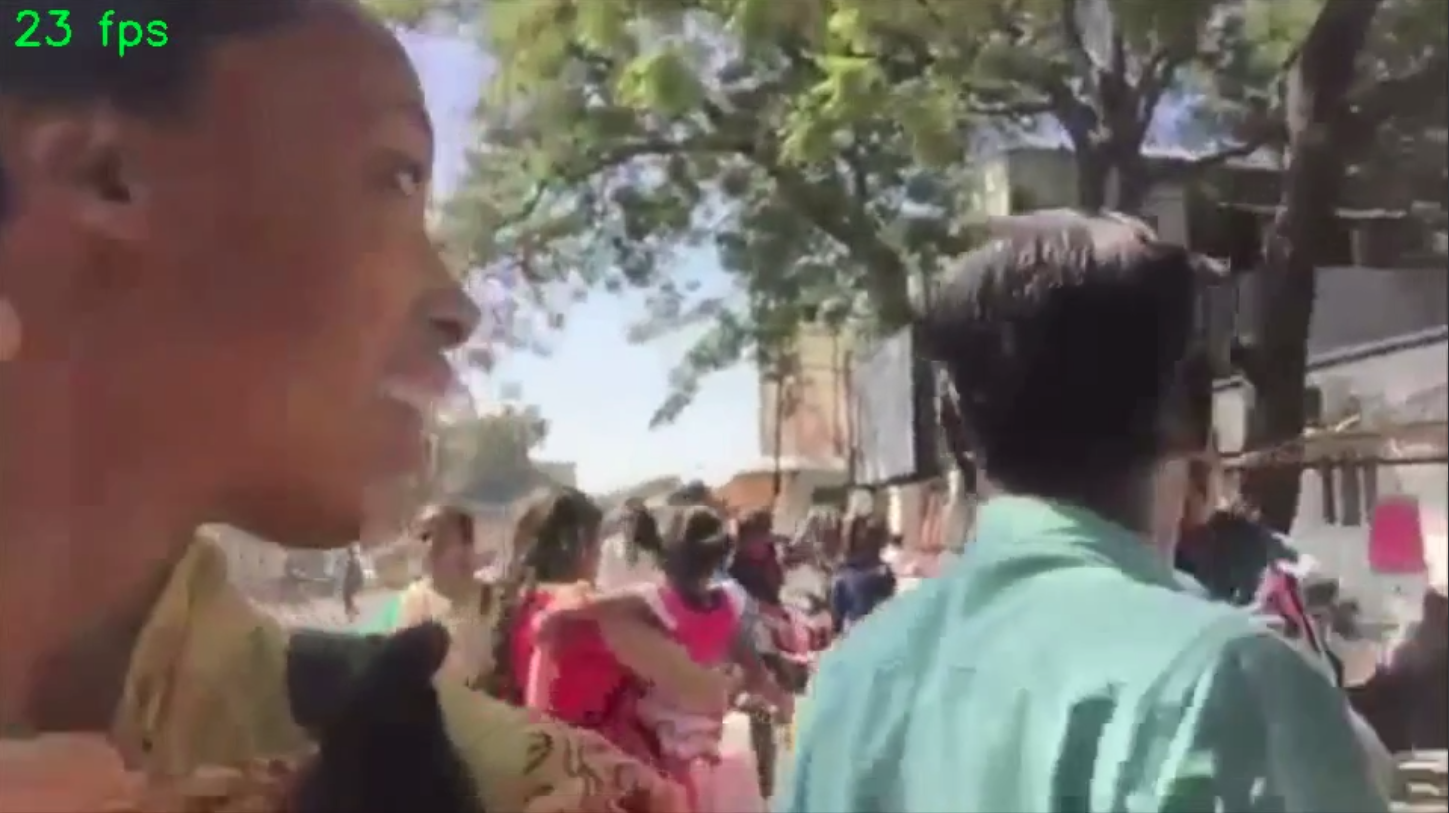}\hspace {-0.1em}
\includegraphics[width=3.45cm]{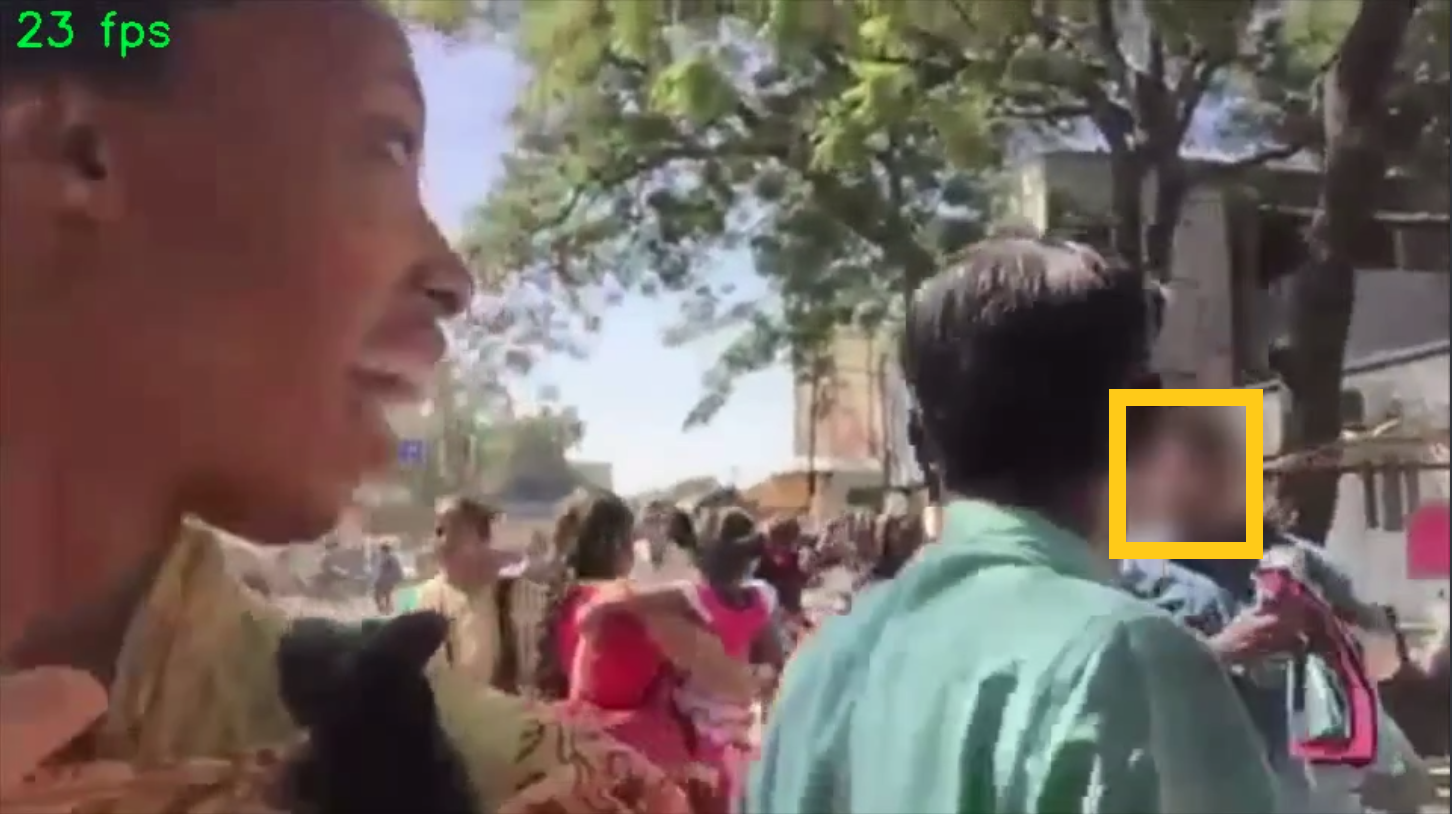}\hspace {-0.1em}
\includegraphics[width=3.45cm]{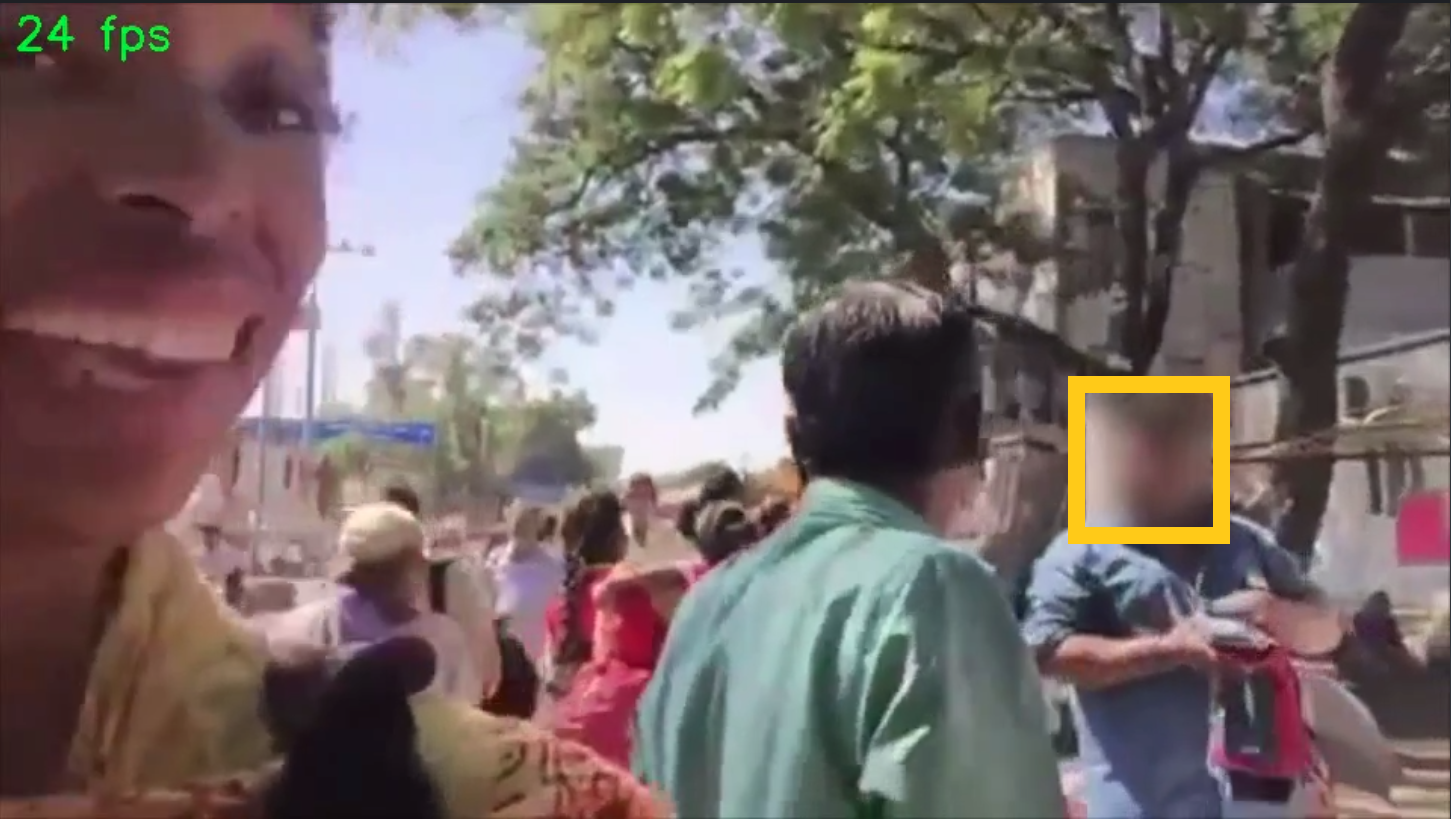}
}
\hskip -1pt

\caption{Differences on the drifting and over-pixelation issues handled by FPVLS (yellow) and tracker-based YouTube Studio tool (red). The original snapshots of the live streaming scene are listed in the upper row from left to right. The middle row with red circle mosaics is the offline pixelation results of YouTube Studio tools. The bottom row with yellow rectangular mosaics is the online pixelation results of FPVLS.}
\vspace{-0.5cm}
\label{fig1}
\end{figure*}

\par
Video live streaming has never been so prevailing as in the past two or three years. Live streaming platforms, sites, and streamers rapidly become part of our daily life. Benefit from the popularity of smart-phones and cheaper 5G networks, online video live streaming profoundly shapes our daily life~\cite{vo2017optimal}. Streaming techniques endow people with the capability to instantly record and broadcast real scenes to audiences. The primary host of live streaming activities transformed from TV stations, news bureau, and professional agencies to the ordinaries. Without the imposed censorship, these private streaming channels severely disregard the personal privacy rights~\cite{faklaris2016legal}. As a result, streamers, together with live streaming platforms, have raised a host of privacy infringement issues. Despite different privacy laws in different regions, these privacy infringements are always concerning and even committed to crimes~\cite{zimmer2017law}.

\par
Apart from the indifference of streamers, the handcrafted and labor-intensive pixelation process is in fact the main reason leads to rampant privacy violations amid streaming~\cite{stewart2016up}. Streaming platforms are therefore incapable of allocating adequate human labor on executing censorship anytime, anywhere. Being a giant hosting service provider, YouTube is already aware of the reason and published its offline face pixelation tool on the latest YouTube Creator Studio~\cite{youtube2020}. When uploading videos, creators are now required to pixelate faces of children under age 13 through YouTube Creator Studio, unless parental consent is obtained~\cite{youtube2020kid}. YouTube Live and Facebook Live experienced explosive growths during the global COVID-19 pandemic, and how to push similar privacy-protection policies in streaming is under heated discussions. As far as we know, current studies, including YouTube Studio~\cite{youtube2020} and Microsoft Azure~\cite{Juliako2020}, mainly focus on processing offline videos; whilst leave the face pixelation solution on the online video live streaming field underexplored. Thus, in this paper, we build a privacy-protection method named Face Pixelation in Video Live Streaming (FPVLS) to blur irrelevant peoples' faces amid streaming. FPVLS protects personal privacy rights as well as promotes the sound development for the video live streaming industry.

\par
The tangible reference of FPVLS is the mosaics manually allocated in TV shows. People who are unwilling to expose in front of the audiences will be facially pixelated in the show's post-production process. Hence, FPVLS can be intuitively resolved as an automized replica of the manual face pixelation process. Furthermore, the manual face pixelation process is remarkably close to the multi-face tracking (MFT) algorithms, because they both involve continuous identification of target faces on frames. Then, fine-tuning MFT algorithms on live streaming seem to be a straightforward solution for FPVLS.

\par
Follow the above interpretation, we directly migrated and tuned the multi-face trackers on our collected video live streaming dataset. However, vast tests reveal that migrated multi-face trackers suffer a severe drifting issue. Fig.~\ref{fig1} demonstrates a typical street-streaming scene from left to right. A street hawker is waving hands to the female streamer for greeting. Without the permission of the hawker, his face shall be pixelated for privacy protection. Meanwhile, a passerby in a light green shirt is passing through the streamer and the hawker. YouTube Studio, which relies on MFT algorithms, generates the middle row's mosaics circled in red. During the crossover of the passerby and the hawker (middle-three snapshots), the hawker's face mosaics drift manifestly. At the rightmost snapshot, the mosaics are even brought to the passerby's neck region and leave the hawker's face completely exposed.

\par
Such a drifting issue is mainly caused by the underperformance of the face detection algorithms in the unconstrained live-streaming scenes. State-of-the-art MFT studies adopt a similar tracking-by-detection structure~\cite{sanchez2016cascaded,danelljan2017eco,henriques2014high}, which assembles face detectors ahead of trackers. Detectors specify the bounding boxes of faces on frames and pass the results to trackers. Trackers either designate an existing tracklet or initiate a new tracklet for each detection, and further merge tracklets into an intact trajectory. The detector's outcomes primarily determine the quality of the tracklets, and trackers accept tracklets as priors in making scene-based matching or reasoning.
On the whole performance of the MFT algorithm, the detector plays an as pivotal role as the tracker. However, the detection accuracy is seldom mentioned and discussed in current MFT algorithms~\cite{sanchez2016cascaded,danelljan2017eco,henriques2014high,yu2016poi,shen2018tracklet}. Benchmark tests for these MFT algorithms are carried out in music videos, movie clips, TV sitcoms, or CCTV footage recorded by professional, high-resolution, steady, or stably moving cameras with various shot changes. Face detectors function well in these videos; accordingly, tracklets can be reliably generated. Clean face detection results and reliable primal tracklets become the acquiescence, and the major concern of current MFT is how to establish consistent, accurate associations among tracklets under frequent switching shots.

\par
In practice, live streaming scenes are commonly recorded by handheld mobile phone cameras with only a single or a few shots. Likewise, the video quality can only maintain at a relatively low level (720p/30FPS or lower). Crowded scenarios, frequent irregular motions of the camera, and abrupt broadcasting resolution conversion are everywhere in live streaming. All these harsh realities indicate that face detection networks are no longer able to yield accurate results for the generation of reliable tracklets. Once scattered or unreliable tracklets are initialized, trackers are prone to fast corruption and jeopardize the efficacy of MFT based pixelation methods. Only very few previous works~\cite{huang2008robust,yu2016poi} noticed the uncertainty caused by unreliable detection results but hardly explored solutions.

\par
Besides the drifting, the apparent over-pixelation problem can be observed in the middle snapshots of Fig.~\ref{fig1}. Introduced by the tracking algorithms and inherited by tracking based pixelation methods, the over-pixelation occurs when pixelation methods generate unnecessary and excessive mosaics for un-identifiable faces. Heavy or full occlusion and massive motion blurs are the common reasons in causing un-identifiable faces. Over-pixelation is an intrinsic problem while migrating tracking algorithms to pixelation tasks. MFT algorithms benefit from estimating the locations of temporary un-identifiable faces~\cite{zhang2014technology, shen2018tracklet}. Such an evidence accumulation process fairly contributes to constructing long enough trajectories, thereby avoiding frequent face ID-switch. Unlike tracking, in the meantime of redacting irrelevant faces out, pixelation tasks eager to preserve the audience as many originals as possible.

\par
Consequently, to accomplish the irrelevant people's faces tracking and pixelation task under streaming scenes, FPVLS adopts a brand new, clustering and re-detection based, frame-to-video framework. The core of this framework is the raw face trajectory generation stage and the trajectory refinement stage. In a nutshell, when streaming started, face detection and embedding (recognition) networks are alternately applied on frames to yield face vectors. Noises caused both by false positives in detection and variations in embedding networks are concealed in these face vectors. Thereupon, Positioned Incremental Affinity Propagation (PIAP) is proposed to cope with these noises and generate raw face trajectories. PIAP inherits the ability of clustering under ill-defined cluster numbers from classic Affinity Propagation (AP)~\cite{frey2007clustering}. We introduce the positioned information to revise affinities and endow PIAP with noise-resistance. We further employ incremental clustering to accelerate the consensus-reaching process. Besides the noises, false negatives in detection arouse intermittences in raw trajectories. The trajectory refinement stage aims to elaborately compensate the intermittences without provoking over-pixelation. We construct a proposal network to re-detect the faces in the intermittences. The two-sample test based on the Empirical Likelihood Ratio (ELR) is then applied to cull the non-faces from re-detection results. The proposal network and the two-sample test collaborate to yield the final trajectories. Lastly, the post-process allocates Gaussian filters on the refined trajectories for pixelation. An example of FPVLS is presented in the bottom row of Fig.~\ref{fig1}. The mosaics marked in yellow rectangular remain on the hawker's face after the crossover. The drifting issue is greatly alleviated. Further, as shown in the middle pictures of the third row, while the target face is fully occluded, FPVLS will not produce baffling, over-pixelated mosaics like YouTube Studio or other tracking based methods.


\par
Specifically, this paper consists of following contributions:
\begin{itemize}
    \item We build the first online pixelation framework: Face Pixelation in Video Live Streaming (FPVLS). FPVLS adopts PIAP clustering and the two-sample test based on ELR to alleviate the trajectories drifting as well as contain the over-pixelation problem.
    \item We proposed the Positioned Incremental Affinity Propagation (PIAP) clustering algorithm to generate raw trajectories upon inaccurate face vectors. The proposed PIAP spontaneously handles the cluster number generation problem, and are endowed with noise-resistant and time-saving merits through position information and increment clustering.
    \item A non-parametric imposing two-sample test based on empirical likelihood ratio (ELR) statistics is introduced. Cooperating with the proposal network, the two-sample test compensates the deep networks insufficiency through the ELR statistics, and avoid over-pixelation in the trajectory refinement. Such an error rejection algorithm can serve other face detection algorithms.
    \item We collected a diverse video live streaming dataset from the streaming platforms and manually labeled dense annotations (51040 face labels) on frames to conduct our experiments. This dataset will be released to the public and could be the benchmark dataset for future studies on video live streaming.
\end{itemize}

\par
The remaining of the paper is organized as follows. Section \uppercase\expandafter{\romannumeral2} reviews the related works; Section \uppercase\expandafter{\romannumeral3} introduces the detail methods of FPVLS; experiments, results, and discussion could be found in Section \uppercase\expandafter{\romannumeral4}; Section \uppercase\expandafter{\romannumeral5} concludes and discusses future work.

\section{Related Works}
To our best knowledge, face pixelation in video live streaming is not studied before. Referring to existing techniques, we first tried to implement FPVLS in an end-to-end manner by replacing the 2D deep CNN with its extension 3D CNN~\cite{ji20133d}. Taking an entire video as input, 3D CNN handles spatial-temporal information simultaneously and achieves great success in video segmentation~\cite{yue2015beyond,karpathy2014large} and action recognition areas~\cite{ji20133d,tran2015learning}. However, tested on the live streaming dataset we collected, 3D CNN models are quite time-consuming and extremely sensitive to video quality and length. Furthermore, without sufficient training data, as designed to cope with high-level abstractions of video contents, 3D CNN cannot precisely handle the pixelation task on the individual frame-level. Therefore, the face mosaics generated by the 3D CNN model constantly blinking during streaming.
\par
The other proper way for realization is the multi-object tracking (MOT) or multi-face tracking (MFT) algorithms we discussed in the Introduction. Current offline pixelation methods, including YouTube Studio and Microsoft Azure, are mainly implemented on MFT. Although MOT and MFT are claimed to be thoroughly studied~\cite{zhang2014technology}, their tracking accuracy is still challenging in unconstrained videos. Reviewing state-of-the-art trackers, we find two main categories: offline and online trackers. The former category assumes the object detection in all frames has already been conducted, and clean tracklets are established by linking different detection in offline mode~\cite{berclaz2011multiple}. This property of offline trackers allows global optimization for the path~\cite{zamir2012gmcp} but makes them incapable of dealing with online tasks.
\par
State-of-the-art online MOT algorithms~\cite{henriques2014high, danelljan2017eco, mueller2017context, valmadre2017end, liu2016structural} combine continuous motion prediction, online learning, and the correlation filter to find an operator that gradually updates the tracker at each timestamp. KCF~\cite{henriques2014high} and ECO~\cite{danelljan2017eco} are the representatives of this kind. Within the same shot, these trackers assume a predictable position for targets in the adjacent frames. Although this assumption is generally applicable in conventional videos, it does not hold in video live streaming due to the frequent camera shakes. Tests of such tracking algorithms are always established on high-resolution music video datasets in the absence of massive camera moves and motion blur~\cite{zhang2016tracking}. Moreover, MOT sometimes encounters the size issue when directly applied to the face tracking field. The faces' size may be so small (12*12 pixels) that the online-learned operator is not robust enough to progressively locate the time-invariant feature of the face. Therefore, the correlation filter based MOT algorithms also suffer drastic drifting issues on the face pixelation tasks.
\par
State-of-the-art MFT algorithms adopt the tracking-by-detection structure~\cite{bochinski2017high,yu2016poi,sridhar2015fast,weinzaepfel2015learning}. IOU tracker~\cite{bochinski2017high} and POI tracker~\cite{yu2016poi} are the outstanding ones. IOU tracker applies detection networks on the current frame. The bounding boxes with the most considerable IoU value will be assigned to the corresponding trajectory, referring to the tracking list. IOU tracker is agile in processing frames, but heavily relies on the detection results and does not consider the temporal connection between frames. Thereby, false positives and false negatives in detection are hard to be eliminated and cause instant drifting mosaics. POI tracker replaces the correlation filters with deep embedding networks. Online POI tracker applies detection networks in all frames and divides the tracking video into multiple disjoint segments. They then use the dense neighbors (DN) search to associate the detection responses into short tracklets in each segment. Several nearby segments are merged into a longer segment, and the DN search is applied again in each longer segment to associate existing tracklets into longer tracklets. To some degree, the POI tracker shares a cognate idea with our raw face trajectory generation stage in the video segment, detection, and embedding aspects. However, POI uses a DN search instead of clustering algorithms and do not further apply refinement. Therefore, the POI tracker will also be profoundly affected by inaccurate detection contains frequent false positives and negatives. Besides, widely accepted solutions of MOT and MFT~\cite{zhang2016joint,insafutdinov2017arttrack} in unconstrained videos assume some priors, including the number of people showing up in the video, initial positions of the people, no fast motions, no camera movements, and high recording resolution. The state-of-the-art studies on MFT~\cite{lin2018prior} manage to exclude as many priors as possible but still require the full-length video in advance for analysis.

\begin{figure*}[tbp]
\centering
\includegraphics[scale=0.22]{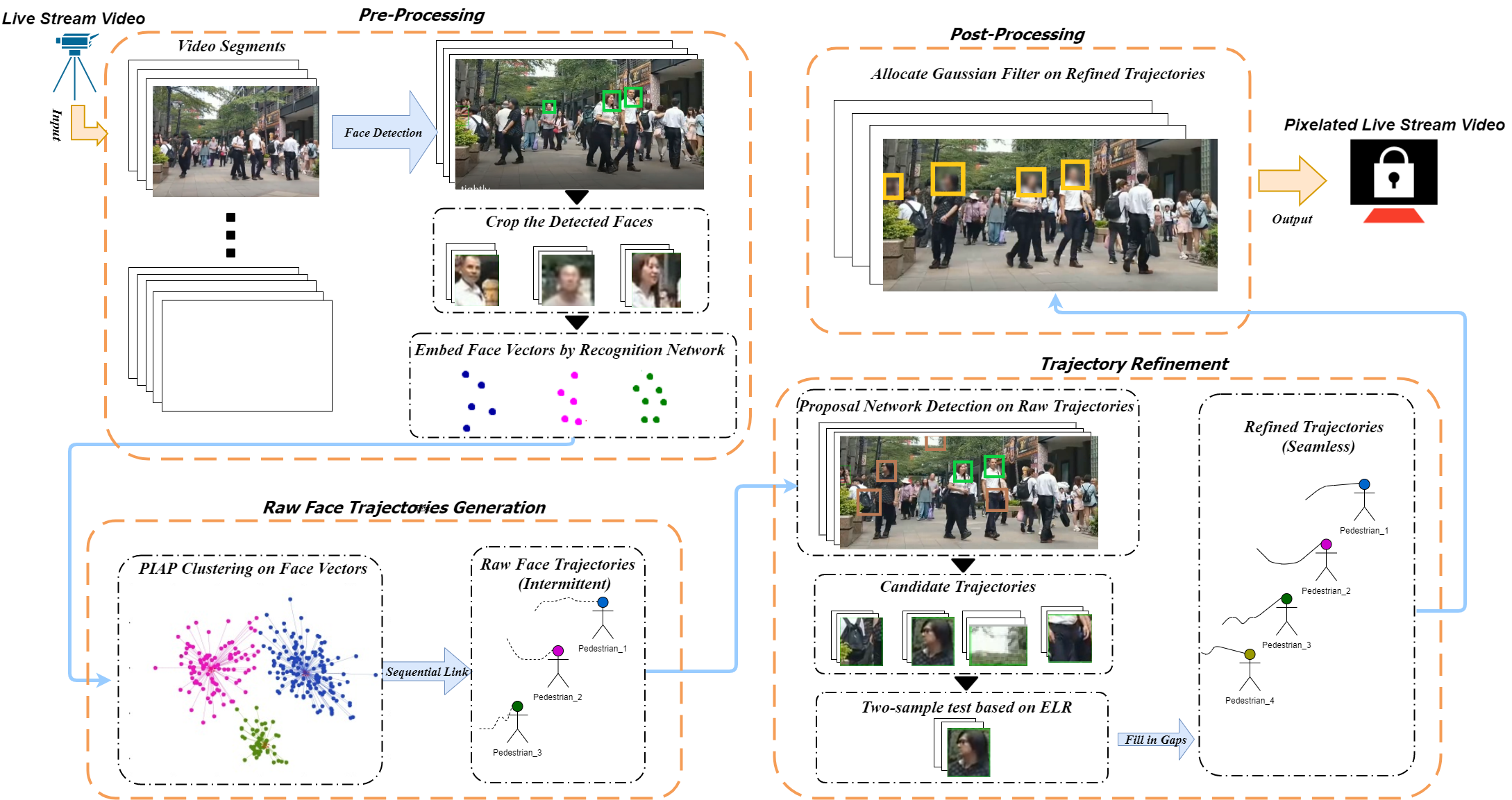}
\caption {Proposed framework of FPVLS.}
\label{fig2}
\end{figure*}

\section{Proposed Methods}
Consider the aforementioned pros and cons; we construct the FPVLS as depicted in Fig.~\ref{fig2}. Along with a preprocessing and a postprocessing stage, FPVLS owns two core stages: the raw face trajectories generation stage and the trajectory refinement stage. In Fig.~\ref{fig2}, stages are enclosed by yellow rectangular, and the pivotal intermediate results or algorithms are surrounded by dark rectangular.
\par
Specifically: $\textbf{(1) Prepocessing:}$ We design a buffer section at the beginning of each live streaming. The buffer section enables FPVLS to slice an entire video streaming into video segments, thereby launching the later trajectory refinement stage. Besides, the buffering time can be leveraged to designate the streamers' faces (like touching the phone screen). Face detection and embedding networks are also applied in this stage to generate face vectors. $\textbf{(2) Raw Face Trajectories Generation:}$ Raw face trajectories link the same person's face vectors across frames through clustering. For the clustering afoot, the number of people showing up in streaming cannot be known in advance, the face vectors contain noises, and the real-time efficiency must be met. Hence, we propose the PIAP upon AP~\cite{frey2007clustering}. Positioned information revises the affinities, thereby altering the embedding variations and excluding the noises as outliers. The incremental propagation shoves the consensus-reaching process, thereby, guarantees efficiency. $\textbf {(3) Trajectory Refinement:}$ Besides the noises, raw trajectories incur intermittences due to the false negatives in detection. A proposal network is built to propose suspicious face areas. Afterward, the two-sample test based on empirical likelihood ratio statistic is applied to deny the inappropriate bounding boxes yielded by the proposal network and fill the accepted ones into the intermittences for refinement. The discontinuity of raw trajectories is therefore elaborately vanished without provoking the over-pixelation problem. $\textbf {(4) Postprocessing.}$ Gaussian filters with various kernel values can be placed on the non-streamers' trajectories to generate the final pixelated streaming and then broadcast to the audience.

\subsection{Prepossessing}
\subsubsection{Video Segments}
FPVLS deals with video live streaming; the processing speed of FPVLS, counted by pixelation frames-per-second (FPS), shall be greater than or equal to the streaming's original broadcasting speed, which is also counted by FPS. Usually, the broadcasting FPS is constant. Therefore, if we denote the broadcasting FPS as ${FPS}$, we have pixelation FPS $\geq {FPS}$ as facts. Then, we can stack every $\mathcal{N}$ frames into a short video segment by demanding an at least $(2*\mathcal{N})$ frames buffer section at the very beginning of streaming without causing discontinuities in the broadcasting after pixelation (\textbf{Proposition 1}).

\par
Afterward, leveraging image-based face detection and recognition networks frame-by-frame, face vectors in a segment are obtained. Video segments reform the commonly hours-long video live streaming into numerous segments in seconds level. In this manner, we alleviate the conflict between accuracy and efficiency by transforming it into a broadcasting latency. Considering the primal latency brought by communication and compression, such a buffering latency is minor and may not be noticed. However, the length of the segments shall still be restrained within a rational interval to balance the performance of FPVLS and the user-experienced latency.

\begin{proposition}
  \vspace{1em}
If video segments length is $\mathcal{N}$ frames, a $(2*\mathcal{N})$ frames buffer section can ensure the broadcasting continuity after pixelation.
\end{proposition}

\begin{proof}
The recording of a video is continuous and normally the recording FPS $= {FPS}$.
FPVLS receives this video as raw input and processes on the unit of video segments. For a certain frame $Fr$ recorded at frame number $f$, FPVLS will stack $Fr$ into the segment $Sg$; then process the entire segment $Sg$; and finally broadcast $Sg$.

\par
Thus, we can have:
\begin {itemize}
\item The recording of the all the frames belonging to $Sg$ is completed at time $\{\lceil \frac{f}{\mathcal{N}} \rceil * \frac{\mathcal{N}}{FPS}\}$.\\
\item FPVLS takes at most $\{\frac{\mathcal{N}}{FPS}$\} seconds to process $Sg$.\\
\item $Sg$ is then broadcast, and extra $\{\frac{f}{FPS}- \lfloor \frac{f}{\mathcal{N}} \rfloor *\frac{\mathcal{N}}{FPS}\}$ seconds are needed to display $Fr$, as $Sg$ is broadcast from its beginning frame.
\end{itemize}
\par
The broadcasting frame number of $Fr$ after FPVLS is the sum of the above three:\newline
\[FPS*\{\lceil \frac{f}{\mathcal{N}} \rceil * \frac{\mathcal{N}}{FPS}+\frac{\mathcal{N}}{FPS}+\frac{f}{FPS}- \lfloor \frac{f}{\mathcal{N}} \rfloor *\frac{\mathcal{N}}{FPS}\}\]
\[\Rightarrow \{f + \mathcal{N} + {\mathcal{N}} * ( \lceil \frac{f}{\mathcal{N}} \rceil - \lfloor \frac{f}{\mathcal{N}} \rfloor)\}    \Rightarrow f+2*\mathcal{N} \]
That is, when segment length is $\mathcal{N}$ frames, FPVLS broadcasts any frame with an at most $(2*\mathcal{N})$ frames lag. We can directly cover this lag by demanding a $(2*\mathcal{N})$ frame long buffer section at the beginning of streaming.
\qedhere
\vspace{1em}
\end{proof}


\subsubsection{Face Detection and Recognition on Frames}
In this paper, we use MTCNN~\cite{zhang2016joint}, and CosFace~\cite{wang2018cosface} to process every frame in turn. MTCNN and CosFace are chosen for the convenience of later affinity in PIAP and the training of the proposal network. However, do note that the detection and embedding networks could be substituted by other state-of-the-art work like PyramidBox~\cite{tang2018pyramidbox} \& ArcFace~\cite{deng2019arcface}. Also, instead of indulging in comparing or tuning pre-trained face detection and embedding algorithms, we intentionally avoid using the most advanced algorithms to demonstrate the performances of FPVLS under relatively poor priors. MTCNN accepts arbitrary size inputs and detects face larger than $12*12$ pixel. CosFace reads the face detection results as input and generates $512$ dimension feature vector for each face. Face alignment is applied right after detection. Every detected face is cropped out from the frame and aligned to frontal pose through the affine transformation before embedding. Since we built the refinement stage, the detection threshold is initiated to a higher value for better suppressing the false positives.

\subsection{Raw Face Trajectories Generation}
The embedded face vectors are further clustered to connect the same face across frames and segments. DBSCAN~\cite{ester1996density}, FCM~\cite{pal1995cluster}, Affinity Propagation (AP)~\cite{frey2007clustering}, and their extensions are the top candidates for objects clustering in data streams. As noisy detection results are common in videos and the data-size for clustering is also unbalanced, the density-based DBSCAN is declined. Fuzzy logic based FCM mandatorily requires the cluster number for initialization. Despite the ability to handle the ill-defined cluster number and cluster under unbalanced data size, the remaining AP also has intrinsic noise-sensitive and time-consuming defects. Therefore, in the proposed PIAP, pair-wise affinities based on deep feature vectors are revised according to positioned information. The time to reach the consensus are significantly shortened by assigning responsibilities and availabilities to new-coming vectors according to the ones in the last segment. PIAP produces the face clusters, and an inner-cluster sequential link quickly generates the faces' raw trajectories afterward.

\subsubsection{Affinity Propagation}
The key of AP~\cite{frey2007clustering} is to select the exemplars, which are the represents of clusters, and then assign the rested nodes to its most preferring exemplars. Similar to traditional clustering algorithms, the very first step of AP is the measurement of the distance between data nodes, also called the similarities or affinities. Following the common notation in AP clustering, $i$ and $k$, ($i,k \in R^D$) are two of the data nodes and $S(i,k)$ denotes the similarity between data node ${i}$ and ${k}$, thereby indicating how well the data node $k$ is suited to be the exemplar for data point $i$. $S$ is the similarity matrix stores the similarities between every two nodes. $S(i,k)$ is the element on row $i$, column $k$ of $S$. Similar notations are also used in below. The diagonal of the similarity matrix $S(k,k)$ denotes the preference of selecting $k$ as an exemplar. $S(k,k)$ is called as the preference value $p(k)$ in some studies~\cite{wang2013multi}.
\par
A series of messages are then passed among all data nodes to reach the agreement on exemplars' selection. These messages are the responsibilities $R(i,k)$ and availabilities $A(i,k)$. $R$ and $A$ are the responsibility and availability matrices. Node $i$ passes $R(i,k)$ to its potential exemplar $k$, indicating the current willingness of $i$ choosing $k$ as its exemplar considering all the other potential exemplars. Correspondingly, $k$ responds $A(i,k)$ to $i$, implying the current willingness of $k$ accepting $i$ as its member considering all the other potential members. Then, $R(i,k)$ is updated according to $A(i,k)$, and a new iteration of message passing is issued. The consensus of the sum-product message passing process ($R(i,k)$ and $A(i,k)$ remain the same after an iteration) stands for the final agreement of all nodes in exemplars' selection and clusters' association is reached. The sum of $R(i,k)$ and $A(i,k)$ can directly give the fitness for choosing $k$ as the exemplar of $i$. Apart from the ill-defined cluster number, AP is not sensitive to the initialization settings; selects real data nodes as exemplars; allows asymmetric matrix as input; is more accurate when measured in the sum of squared errors. Therefore, considering the subspace distribution (measured by the least squares regression), AP is effective in generating robust and accurate clustering results for high dimension data like face vectors.
\par
We initialize both of the responsibility and availability matrices to zero for the commencement, and following~\cite{frey2007clustering}, the computation of $R$ and $A$ are:

\begin{equation}
R(i,k) \leftarrow S(i,k) - \max_{{k'}, s.t.{k'} \neq k}\{A(i,k')+S(i,{k'})\}
\label{eq1}
\end{equation}
\begin{equation}
A(i,k) \leftarrow \min\bm{\{} {0, R(k,k)+ \sum_{{i'},{i'}\notin {\{i,k\}}} \max \{0, R({i'},k)\} \bm{\}}}
\label{eq2}
\end{equation}
Equation~(\ref{eq3}) is used to fill in the elements on the diagonal of the availability matrix:
\begin{equation}
A(k,k) \leftarrow \sum_{i', s.t. i' \neq k}\max\{0, R(i',k) \}
\label{eq3}
\end{equation}
Update responsibilities and availabilities according to~(\ref{eq1}),~(\ref{eq2}) and~(\ref{eq3}) till convergence, then
the criterion matrix $C$, which holds the exemplars as well as clustering results, is simply the sum of the availability matrix and responsibility matrix at each location.
The highest value in each row of $C$ is designated as the exemplar. Naturally, as each row stands for a data node, data nodes that share the same exemplar are in the same cluster. Therefore, the exemplars $\hat{c}=\{c_1,...c_i,...c_n\}, \hat{c}\in C$ can be computed as:
\begin{equation}
c_{i}=\arg \max_{k} \{A(i,k)+R(i,k)\}
\label{eq5}
\end{equation}

\subsubsection{Positioned Incremental Affinity Propagation}
\par
(\ref{eq1})-(\ref{eq5}) are the original equation for classic AP. In most trackers, the critical assumption is that the motion of the same face between consecutive frames is minor within a single shot. Correlation filters, IOU trackers, or other online learning networks are all working under such an assumption. However, in live streaming videos, the frequent camera shakes and the resolution conversion invalid this assumption. The only fact that helps exclude outliers from the classic AP algorithm is that any face can only appear in one position in a single frame. Therefore, faces belonging to the same frame should have small enough similarities. Cosine similarity is then brought for measurements. Let $j$ stands for the other face vectors that belong to the same frame as $i$. The similarity matrix can be computed as:
\begin{equation}
   S(i,k)=
    \left\{
             \begin{array}{lr}
             \frac{i \cdot k}{\Vert i \Vert \Vert k \Vert} -1,  \quad\mbox{ if }\ k \notin{j}\\

          \\
             -1, \quad \mbox{ if } \ k \in j
             \end{array}
    \right.
\label{eq6}
\end{equation}
$S(i,k) \in [-1,0]$, since the negative similarity is more decent for the convergence of message passing.
\par
We intentionally avoid directly manipulating the vectors' value like in some other face clustering works~\cite{cao2015diversity}. Although the embedding network is not perfect and will surely generate outliers due to many factors, it is still somehow reliable considering the face vectors' subspace distribution. Instead of any rush tuning on face vectors, we solely force the similarities to the minimum if faces come from a single frame.
\par
Moreover, $p(k)$ which is the diagonal of the similarity matrix $S$, reflects how suitable is $k$ considered as the potential exemplar for itself. $p(k)$ is directly assigned to $R(k,k)$ according to~(\ref{eq1}). Therefore, in~(\ref{eq2}), the left side of the equation is the evidence of how much $k$ can stands for itself ($R(k,k)$) and how much $k$ can stands for the others ($\sum_{{i'},{i'}\notin {\{i,k\}}} \max\{0, R({i'},k)\}$). Every positive responsibilities of $k$ contributes to $A(i,k)$. As we mentioned above, any face can only have one position in a single frame. $A(i,k)$ should not count $j$'s ($j \in k$) choice as part of $i$'s availability. One step further, the choice of $j$ will actually repel $i$ in availability and equation~(\ref{eq2}) shall be rewritten as:
\begin{equation}
    \begin{aligned}
      &A(i,k) \leftarrow  \min\bm{\{}0, R(k,k)+ \\
             &\sum_{{i'},{i'}\notin {\{i,j,k\}}} \max\{0, R({i'},k)\}   -\sum_{{i'},{i'}\in {\{j\}}} \max\{0, R({i'},k)\}\bm{\}}
           \end{aligned}
\label{eq7}
\end{equation}
(\ref{eq6}) and (\ref{eq7}) ensure $i$ and $j$ are mutual exclusive in the clustering process. If $i$ variants and aggregates to another cluster contains $j$, they will repel each other during the affinity propagation process. The significant drop in their availabilities will cause the corresponding responsibilities also drastically reduced. This chain reaction will not stop until one of them is revised to another cluster through the message passing. Also, as $S(i,j)=-1$, responsibility in~(\ref{eq1}) guarantees $j$ will not be selected as the potential exemplar of $i$ during initialization.
\par
Besides the position information, we employ incremental clustering to address the efficiency issue. The key challenge in building incremental AP is that the previously computed data nodes have already connected through nonzero responsibilities and availabilities. However, the newly arrived ones remain zero (detached). Therefore, the data nodes coming at different timestamps also stay at varying statuses with disproportionate relationships. Simple continuously rolling their affinity at each timestamp cannot reach the solution of incremental AP.
\par
In our case, video frames are strongly correlated data. According to the previous segment, we could assign a proper value to newly-arrived faces in adjacent segments without affecting the clustering purity. The face vectors belonging to a particular person are supposed to stay closer to each other. That is, the same person's faces should gather within a small area in the feature space. Thus, our incremental AP algorithm is proposed based on the fact that: if two detected faces are in adjacent segments and refer to one person, they should be clustered into the same group as well as have similar responsibilities and availabilities. Rather than starting with zero, we can assign the same responsibilities and availabilities to the newly-coming vectors and considerably trim the message-passing process. Such a fact is not well considered in past studies of incremental affinity propagation~\cite{cao2015diversity,sun2014incremental}.
\par
Following the commonly used notations, the similarity matrix is denoted as $S_{t-t'}$ at $t-t'$ time with ($M_{t-t'}*M_{t-t'}$) dimension, where $(t-t')= \frac{\mathcal{N}}{FPS}$. And the responsibility matrix and availability matrix at time $t-t'$ are $R_{t-t'}$ and $A_{t-t'}$ with a same dimension as $S_{t-t'}$. Mark the closest face vector for the newly arrived $i$ as $i'$, we have:
\begin{equation}
    i'=\arg \max_{i', i'\leq M_{t-t'}} \left \{S(i,i')\right \}
\label{eq9}
\end{equation}
Then, the incremental updation of $R_{t}$ according to $R_{t-t'}$ and $A_{t}$ according to $A_{t-t'}$ are:
\begin{equation}
    R_{t}(i,k)=
    \left\{
             \begin{array}{lr}
             R_{t-t'}(i,k), \quad i\leq M_{t-t'}, k\leq M_{t-t'}\\
             \\
             R_{t-t'}(i',k), \quad i > M_{t-t'}, k\leq M_{t-t'}\\
             \\
             R_{t-t'}(i,k'), \quad i\leq M_{t-t'}, k > M_{t-t'}\\
             \\
             0, \quad i > M_{t-t'}, k > M_{t-t'}
             \end{array}
\right.
\label{eq8}
\end{equation}
\begin{equation}
    A_{t}(i,k)=
    \left\{
             \begin{array}{lr}
             A_{t-t'}(i,k), \quad i\leq M_{t-t'}, k\leq M_{t-t'}  \\
             \\
             A_{t-t'}(i',k), \quad i > M_{t-t'}, k\leq M_{t-t'}\\
             \\
             A_{t-t'}(i,k'), \quad i\leq M_{t-t'}, k > M_{t-t'}\\
             \\
             0, \quad i > M_{t-t'}, k > M_{t-t'}
             \end{array}
    \right.
\label{eq10}
\end{equation}
$R_{t}(i,k)$ and $A_{t}(i,k)$ are the responsibilities and availabilities of newly arrived face vectors at time $t$. $M_{t-t'}$ stands for the amount of faces at time $t-t'$. Note that the dimension of all the matrices is increasing with time.

\par
Denote $Z_{p}^{q}=\{Z_{1}^{q},Z_{2}^{q},Z_{3}^{q}...Z_{p}^{q}\}$ as the set of all $p$ face vectors extracted in segment $q$ at time $t$. The full process of PIAP algorithm is summarized as \textbf{Algorithm 1}.
\begin{algorithm}[htbp]
\caption{Positioned Incremental Affinity Propagation}
\label{alg:algorithm}
\textbf{Input}: $R_{t-t'}$,$A_{t-t'}$,$C_{t-t'}, Z_{p}^{q}$\\
\textbf{Output}: $R_{t},A_{t},C_{t}$
\begin{algorithmic}[1] 
\WHILE {not end of a live-streaming}
\STATE Compute similarity matrix according to~(\ref{eq6}).
\IF {the first video segment of a live-stream}
\STATE Assign zeros to all responsibilities and availabilities.
\ELSE
\STATE Compute responsibilities and availabilities for $Z_p^q$ according to equation~(\ref{eq9}),~(\ref{eq8}) and~(\ref{eq10}).
\STATE Extend responsibilities matrix $R_{t-t'}$ to $R_{t}$, and availabilities $A_{t-t'}$ to $A_{t}$.
\ENDIF
\STATE Message-passing according to equation~(\ref{eq1}),~(\ref{eq7}),~(\ref{eq3}) and~(\ref{eq5}) until convergence.
\ENDWHILE
\end{algorithmic}
\end{algorithm}

\par
Handled by our proposed PIAP algorithm, we can quickly cluster faces in new segments; exclude the outliers, and discover newly spotted faces.  Unlike newly spotted faces, the false positives are quickly isolated by other real faces while clustering, thereby not forming any long and steady enough trajectories. As a result, the scattered faces referring to the same person could be linked together sequentially to form our raw face trajectories.

\subsection{Trajectory Refinement}
If we blur non-streamers' faces instantly on raw trajectories, fast blinking mosaics can be observed everywhere. Owing to inadequate training samples, poor streaming quality, and noisy scenes, false negatives are manifest in frame-wise face detection results. Therefore, the created raw trajectories are always intermittent and incur blinking mosaics. Regarding little breaks that only span over a few frames in the raw trajectory, we can quickly smooth them through interpolation. The real contest remains on the gaps accumulated frame-after-frame by false negatives in detection.
\par
Direct smoothing on gaps leads to massive over-pixelation problems as we cannot tell whether the miss-detection or heavy occlusions cause the gap. To eliminate the false negatives as well as dodge over-pixelation, we bring in the trajectory refinement stage. Concretely, the trajectory refinement stage grounds on a proposal network and the corresponding two-sample test based on empirical likelihood ratio (ELR) statistics to identify the miss-detection from a bunch of suspicious face areas. The proposal network is a relatively shallow CNN aiming to propose suspicious face areas on the gap frames at a high recall rate. We cull the non-faces from proposed face areas according to the two-sample test based on ELR since \textbf{(i)} it refrains from the computation burden of another embedding network and catches the extra richness of the detection; \textbf{(ii)} it does not require the proposal network results to comply with a Gaussian distribution, thereby, is not parametric-imposing. Combing the two-sample test with the proposal network, we deny the inappropriate bounding boxes and receive the proper ones for refinements.

\subsubsection{The Proposal Network}
\par
The proposal network shares some mutual characteristics with the cascaded face detection networks. The cascaded architecture introduces multiple deep convolution networks to predict faces in a coarse-to-fine manner. The first several networks in the architecture also dedicate to proposing suspicious face areas just as our proposal network does~\cite{zhang2016joint, chen2016supervised, qin2016joint}. For the sake of simplicity, here, we peel off the first network from MTCNN~\cite{zhang2016joint} to construct our proposal network. MTCNN leverages three cascaded deep convolution networks to predict faces and their landmark locations. The very first P-net in MTCNN and our proposal network reach the same goal.
\par
Moreover, the loss function of the proposal network aims at the highest recall rate. However, ordinary face detection networks always focus on the highest accuracy. The higher recall rate suggests more false positives in results. Therefore, if the proposal network is built and trained in a single-handed manner, it may incur the low precision problem. That is, when measured through the IoU of bounding boxes, the proposal network trends to produce many overlapped false positives. MTCNN is trained as an entirety, and the performance of the P-net is contained by the latter R-net and O-net to achieve high accuracy. Thus, P-net retains a high recall rate and will not incite the precision problem. The structure of the proposal network is shown in Fig.~\ref{fig3}.

\begin{figure}[htbp]
\centering
\includegraphics[scale=0.29]{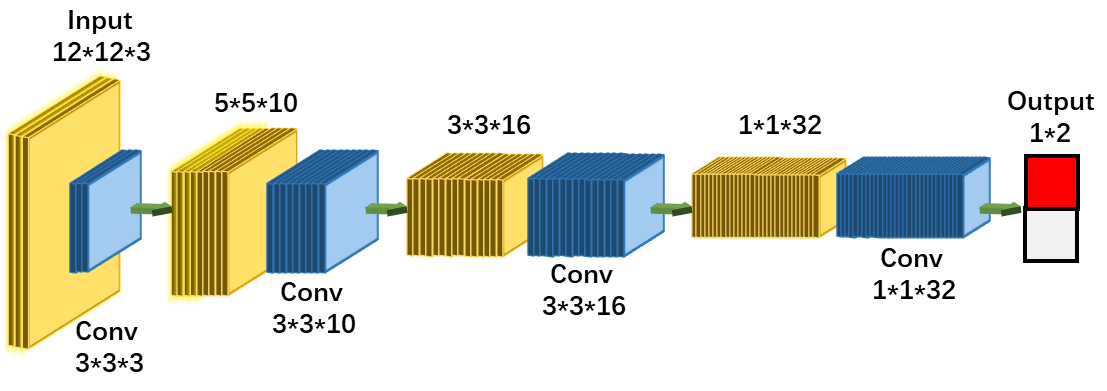}
\caption{The structure of the proposal network. The input and network layers are depicted in yellow. The convolution kernels are rendered in blue. The output is a binary classification result, and the output of the last but one layer is used for the trajectory refinement.}
\vspace{-1em}
\label{fig3}
\end{figure}

\par
If any miss-detection is captured through the proposal network and received by the two-sample test, they will be filled back into the raw trajectories. Afterward, gaps are to be patched into tiny breaks, and the interpolation will reboot to smooth the trajectories. In this manner, we can recursively refine the trajectories seamlessly. Moreover, this manner does not request the proposal net to spot every miss-detection. We thereby can raise the proposal network detection threshold and use Non-Maximum Suppression (NMS)~\cite{neubeck2006efficient} to suppress the false positives under a high recall rate as well.

\subsubsection{Two-sample test Based on Empirical Likelihood Ratio}
Sending the gap areas of the raw trajectory to the proposal network, we can get some of the miss-detection along with other false positives as results. Associate through Hungarian algorithm~\cite{bewley2016simple}, all the proposed faces form candidate trajectories, which correspond to orange dash lines in Fig. 4. Generally, with false positives yielded by the
proposal network, there is likely to have more than one candidate trajectories. Mark an observation of a raw trajectory as the sequence $z=\{z_1,z_2,...,z_m\}$, and mark one of its possible candidate trajectories as the sequence $z'=\{z'_1,z{'}_2,...,z{'}_{m^{'}}\}(m'< m)$. $z$ is yielded by MTCNN, and likewise, $z'$ is yielded by the foremost network in MTCNN. We use the second last layer's output in the proposal network as the vectors of $z'$ and trace back the output of the same layer for $z$. Thereby, as the last but one layer is the size of 1*1*32 in Fig.~\ref{fig3}, $z$ and $z'$ are $32*m$ and $32*m'$ matrices.

\par
The expense of high recall rate is the massively generated false positives. A direct method to exclude these false positives and identify the rest is to re-train another face embedding network and feed the results to PIAP once again. Notwithstanding the above method might be theoretically capable, we cannot ensure the real-time efficiency since the number of proposed faces could be substantial. Furthermore, there is also no guarantee that the deep network will project false positives as outliers because the embedding network is mainly trained on the face dataset. Other than embedding networks, the Siamese net~\cite{bertinetto2016fully, guo2017learning, wang2018learning} might be applicable. Siamese net is designed to identify the slight differences between two similar inputs through a completely weight-sharing, dual structure. However, attribute to its shared weight structure as well, the Siamese net is extremely susceptible to noise. Therefore, it is incapable of sieving the lightly-occluded faces out of the proposed suspicious faces.

\begin{figure}[htbp]
\centering
\includegraphics[scale=0.31]{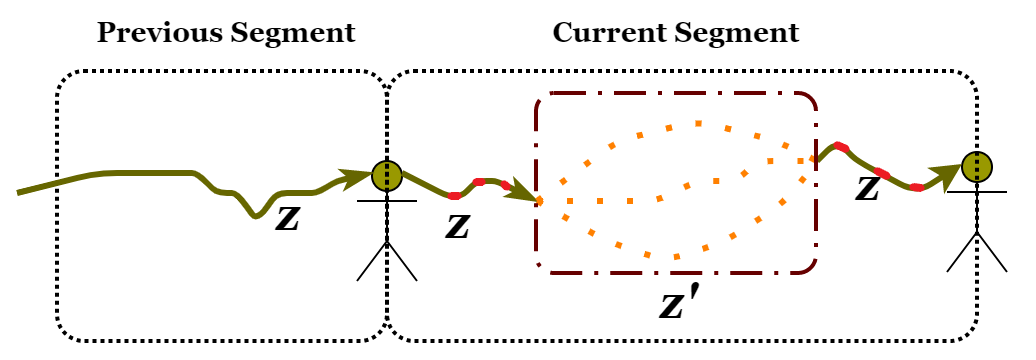}
\caption{$z$ (solid line) and all the possible $z'$ (orange dash lines). Orange dots are the suspicious faces proposed by the proposal network. Red areas on $z$ are the breaks recovered by interpolation.}
\vspace{-1em}
\label{fig4}
\end{figure}

\par
In Fig.~\ref{fig4}, the solid line corresponding to $z$ indicates the seamless trajectories after interpolation. Red areas on the solid line are the breaks recovered by interpolation. Under well-established face detection networks, for the correct or acceptable candidate trajectory $z'$, $|z|>|z'|$ should hold. Meaning the number of miss-detected faces should be less than the number of detected faces in raw trajectories. If we put the detection network insufficiency aside, $z'$ are caused by objective factors including distortions, rotations, illumination changes, and so on. Therefore, we can infer that the original network could somehow detect the miss-detected faces after some reverse-transformation of distortions, rotations, illumination changes, and so on. Substantially, under ideal conditions, $z'$ is expected to be captured directly by the detection network; however, some noises $\varepsilon_{z^{'}}$ stop deep CNN from functioning. Detection under these noises requires the help of the proposal net. Hence, we can deduce the $z'$ from $z$ according to:

\begin{equation}
   z'=f(z,\theta) + \varepsilon_{z^{'}}
\label{eq11}
\end{equation}
where $f$ is the transformation function with the hyperparameter $\theta$, and independent noise term $\varepsilon_{z^{'}} \sim N(0, \sigma_{z^{'}}^2 I)$. $I$ is the identity matrix.
\par
That is, as depicted in Fig. 4, the trajectories refinement is just like the process of linking dense dots (proposed faces) into dash lines (candidate trajectories) and then into a solid line (refined trajectories). Some state-of-the-art studies in MFT~\cite{lin2018prior} fields introduced a similar concept for trajectories optimization, but assume both $z$ \& $z'$ comply with a Gaussian distribution. They assign $f=f_G$ where $f_G$ is a Gaussian Process (GP) model and infers the correlation between tracklets through Maximum-Likelihood Estimation (MLE) on $\theta$.

\par
The GP model is sufficient in most offline tracking cases~\cite{lin2018prior, hirscher2016multiple, hou2007real, yang2013incremental}. However, in video live streaming, trajectories are refined on short video segments, containing only a small number of frames. Namely, we cannot impose $f=f_G$ because: (\textbf{i}) the amount of proposed faces within a segment is not significantly larger than the dimension number of the face vectors; (\textbf{ii}) the amount of proposed faces within a segment is insufficient in guaranteeing the proposed face sample complies with a Gaussian distribution. Inspired by~\cite{ciuperca2016empirical}, for better robustness and avoiding parametric-imposing, we propose a two-sample test to reject the proposed false positives based on empirical likelihood ratio (ELR) statistics.

\par
The goal of the two-sample test is to determine whether two distributions are different based on samples. In our case, we test whether the candidate trajectory and the corresponding raw face trajectory are two face samples that come from the same population (distribution). The statistical null hypothesis for the two-sample test is $\mathcal{H}_0: P_z=P_{z'}$. $\mathcal{H}_0$ stands if and only if $z$ and $z'$ are from the same distribution. Otherwise, $z$ and $z'$ are different at a $\alpha$ significance level, and we should reject $\mathcal{H}_0$ and the corresponding candidate trajectory. Denote the domain of face vectors is a compact set defined in a Gaussian Reproducing Kernel Hilbert Space (RKHS) $\mathit{H}$ with reproducing kernel $\mathcal{F}$. Following the common Maximum Mean Discrepancy (MMD), we measure the distribution $z$ and $z'$ by embedding them into such a Gaussian RKHS~\cite{fukumizu2004dimensionality}. Denote the mean embedding of distribution $z$ and $z'$ in $\mathit{H}$ as $\mu z$ and $\mu z'$, the question of finding optimal estimators of MMD in kernel-based two-sample tests is an estimate of statistically optimal estimators in the construction of these kernel-based tests. After the projection function $f \in \mathcal{F}$, MMD is represented by the supremum of the mean embedding of $z$ and $z'$ :
\begin{equation}
  MMD_{m,m'}[\mathcal{F}, z, z'] := \sup_{f \in \mathcal{F}} (\frac{1}{m}\sum_{i=1}^{m}[f(x_i)]-\frac{1}{m'}\sum_{i=1}^{m'}[f(y_i)] )
\label{eq13}
\end{equation}
$x_i$ and $y_j$ are the elements in $z$ and $z'$ after the projection. Referring~\cite{ciuperca2016empirical, ding2019linear}, $\mathcal{H}_0$ can be written as:
\begin{equation}
    \mathcal{H}_0: P_z=P_{z'}\Leftrightarrow  MMD_{m,m'}[\mathcal{F}, z, z']^2=0 \\ 
\label{eq14}
\end{equation}
According to equation~(\ref{eq13}), the right side of~(\ref{eq14}) is:
\begin{equation}
\begin{aligned}
 &MMD_{m,m'}[\mathcal{F}, z, z']^2= \langle \mu z-\mu z', \mu z- \mu z'    \rangle_\mathit{H} \\
 &=\frac{1}{m(m-1)}\sum_{i \ne j}^{m}k(x_i,x_j)+ \frac{1}{m'(m'-1)}\sum_{i \ne j}^{m'}k(y_i,y_j)\\
 &-\frac{2}{mm'}\sum_{i,j=1}^{m,m'}k(x_i,y_j)
    \end{aligned}
\label{eq15}
\end{equation}
$k(x_i,y_j)$ stands for the Gaussian kernel function of $\mathit{H}$.
\par
As $m'< m$, if we pick the nearest (in time) $m'$ face vectors in $z$, and according to~(\ref{eq15}), set $h(z,z')=\{\frac{1}{m(m-1)}\sum_{i \ne j}^{m}k(x_i,x_j)+ \frac{1}{m'(m'-1)}\sum_{i \ne j}^{m'}k(y_i,y_j)-\frac{2}{mm'}\sum_{i,j=1}^{m,m'}k(x_i,y_j)\}$ for simplification, the linear time unbiased estimator of MMD is proposed in~\cite{gretton2012kernel}. The unbiased estimator $MMD_{ub}$ of $z$ and $z'$ functions similarly as the MLE, but does not presume $z$ and $z'$ are densely sampled under a multi-variant Gaussian Process.
\begin{equation}
MMD_{ub}[\mathcal{F}, z, z'] = \frac{1}{\lfloor m'/2 \rfloor} \sum_{i=1}^{{\lfloor m'/2 \rfloor}}h(z_{2i-1},z_{2i})
\label{eq16}
\end{equation}
\par
Denote $h_i=h(z_{2i-1},z_{2i})$ for further simplification. $z$ and $z'$ are i.i.d, which implies $h_i$ are also i.i.d observations. i.e., $\bar h =0$, where $\bar h$ is the mean value. Hence, we introduce the empirical likelihood ratio (ELR) statistic~\cite{ding2019linear,owen2001empirical} $p_i$ to solve the two-sample test of $\mathcal{H}_0: \bar{h}=0$ as:
\begin{equation}
    L(p_i)= \sup_{p_i}\{\prod_{i=1}^{m'}p_i, \big\vert \sum_{i=1}^{m'} p_i=1, \sum_{i}^{m'} p_ih_i=0 \}
\label{eq17}
\end{equation}
where $p_i$ subjects to the normalization and $\bar{h}=0$ constrains. $\sum_{i}^{N} p_ih_i=0$ ensures the empirical mean of $h$ equals zero. Therefore,  any subtle difference between $z$ and $z'$ can be captured by the empirical likelihood ratio statistic $p_i$ through the pairwise discrepancy $MMD_{ub}$. The explicit solution of~(\ref{eq17}) is attained from a Lagrange multiplier argument $\lambda$ according to~\cite{owen2001empirical}:
\begin{equation}
    \begin{aligned}
        & p_i=\frac{1}{m}\frac{1}{1+\lambda h_i}\ \\s.t.
        &\sum_{i=1}^{m}\frac{h_i}{1+\lambda h_i} =0
     \end{aligned}
\label{eq18}
\end{equation}
\par
Through~(\ref{eq17}) and~(\ref{eq18}), the supremum is transferred to the ELR test $T_{ELR}$ for statistic $p_i$. If the null hypothesis $\mathcal{H}_0$ holds, we have:
\begin{equation}
    \begin{aligned}
        T_{ELR}=-2\log L(p_i) \stackrel{d}{\rightarrow} \mathcal{X}_{1}^{2} \\
    \end{aligned}
\label{eq19}
\end{equation}
where $\mathcal{X}_{1}^{2}$ is the chi-square distribution with 1 degree of freedom according to Wilk's Theorem, and $T_{ELR}$ converges in this distribution. Otherwise, we will reject the null hypothesis $\mathcal{H}_0$ when $T_{ELR} \geq \mathcal{X}_{\alpha}^{2}$, and $\mathcal{X}_{\alpha}^{2}$ is the confidence interval:
\begin{equation}
    Pr(\mathcal{X}_{1}^{2} \geq \mathcal{X}_{\alpha}^{2}) = \alpha
\label{eq20}
\end{equation}
\par
The rejection threshold of~(\ref{eq20}) could be computed through the off-the-shelf python package. Therefore, leveraging $p_i$ as the new parametric feature to eliminate the false positives raised by the proposal network is extremely time-saving. The cost of getting  $p_i$ is solely related to the computation of $T_{ELR}$, and according to~(\ref{eq15}), this cost is linear to the sample size. Our two-sample test algorithm based on empirical likelihood ratio statistic achieves impressing results in real tests.
\par
With the efforts of~(\ref{eq19}) and~(\ref{eq20}), the received $z'$ is added to $z$'s trajectory. Once again, fast interpolation applies to fill up the tiny breaks that still exist and then generate the refined trajectory.

\subsection{Postprocessing}
The final pixelation happens on the refined non-streamers' trajectories through Gaussian filters. The choice of different Gaussian kernels can offer different pixelation styles. Pixelated video segments will directly be broadcast to the audience. An overall procedure of FPVLS is described in \textbf{Algorithm 2}.
\begin{algorithm}[H]
\caption{Overall Procedures of the FPVLS}
\label{alg:algorithm}
\begin{flushleft}
\textbf{Input}:Raw video live streaming\\
\textbf{Output}:The pixelated video live streaming
\end{flushleft}
\begin{algorithmic}[1] 
\WHILE {not end of a live stream}
\STATE Stack the streaming frames into video segments.
\STATE Use the face detection and embedding networks to generate face vectors within a segment.
\IF {is the buffer section}
\STATE Designate (for example, touching the phone screen) the streamers' faces.
\ELSE
\STATE PIAP clustering on face vectors to form the raw face trajectories.
\STATE Interpolation on raw face trajectories to recover the breaks.
\STATE The proposal network detection and two-sample test based on ELR to compensate the detection lost.
\STATE Filling the gaps with detection lost and interpolation again to form refined trajectories.
\ENDIF
\STATE Gaussian Filter blurs the non-streamers' trajectories and broadcast to audience.
\ENDWHILE
\end{algorithmic}
\end{algorithm}

\section{Experiments and Discussions}

\begin{table*}[h]
\caption{Details of the live video streaming dataset}
\centering
\begin{tabular}{|c||cccc||cccc|}
\hline
{Dataset} &\tabincell{c}{Quantity\\ of Videos} & Category  & Resolution & \tabincell{c}{People Occurred}  &Frames & \tabincell{c}{Total\\Face Labels*}   &{\tabincell{c}{Streamers'\\Face Labels*}} &{\tabincell{c}{Irrelevant People's\\Face Labels*}}  \\
\hline\hline
$HS$    &4 &  a,b,c,d       & 720p/1080p   &$\gg$ 2     &4133  &17366 &7989 &9377 \\

$LS$    &8 &  a,b,c,d      & 360p/480p   & $\gg$ 2    &4680  &16112  &7090  &9022 \\

$LN$     &4 &  a,b,c,d       &360p/480p   & $\le$2    &5692 &9849  &5038 &4811\\

$HN$    &4 &  a,b,c,d      &720p       & $\le$2    &4867  &7713  &3956 &3757\\
\hline
\end{tabular}
\begin{tablenotes}
      \small
      \item $\qquad$*: the value amounts the occurrences of the faces. Meaning the same face is repeatedly counted in different frames.
    \end{tablenotes}
\label{tbl1}
\end{table*}

\subsection{Dataset}
As there are no available datasets and benchmark tests for reference, we collected and built a video live streaming video dataset from YouTube and Facebook platforms and manually labeled dense annotations (51040 face labels and 26967 irrelevant people's or privacy-sensitive face labels). We followed the notation paradigm of MOT15~\cite{leal2015motchallenge} and MOT 17~\cite{milan2016mot16}, which are benchmark datasets in tracking. The major difference is we de-associate the heavily occluded and other temporarily un-identifiable faces from their trajectories, and mark them with the 'over-pixelation' label. On the contrary, the tracking datasets give credits to trackers that can produce continuous trajectories under heavy occlusion.
\par
In total, 20 streaming video fragments come from 4 different streaming categories (a: street interview; b: dancing; c: street streaming; d: flash activities) are collected. As demonstrated in Table~\ref{tbl1}, these labeled video fragments are further divided into four groups according to the broadcasting resolution and the number of people showing up in streaming. Live streaming videos having at least 720p resolution are marked as high-resolution $H$, and the rest are low-resolution $L$. Similarly, the live streaming videos contain more than two people are sophisticated scenes $S$, and the rest are naive ones $N$.

\begin{table*}[h]
\centering
\caption{Pixelation results on collected video live streaming dataset}
\begin{tabular}{|c||cccc||cccc||c|c|c|c|}
\hline
\multicolumn{1}{|c||}{\multirow{2}{*}{Method}} &{MFPA$\uparrow$}  &{MFPA$\uparrow$}  &{MFPA$\uparrow$}  &{MFPA$\uparrow$} &{MFPA$\uparrow$}  &{MFPP$\uparrow$} &{MFPP$\uparrow$}
	\multirow{2}{*}&{MFPP$\uparrow$}
    &\multicolumn{4}{c|}{Entire Dataset}\\
    \cline{10-13}
    &($HS$)&($LS$)&($HN$)&($LN$)&($HS$)&($LS$)&($HN$)&($LN$)&
    \tabincell{c}{MP$\uparrow$\\(frames)} &{OPR$\downarrow$} &{FPS$\uparrow$} &{NoP}\\
\hline \hline
YouTube\cite{youtube2020}     &0.45  &0.42 &0.68 & 0.47 &0.53 &0.47 &0.77 &0.63 &238    & 0.31   & $N/A$ &$\times$\\
Azure\cite{Juliako2020}      &0.43  &0.47 &0.64  &0.56 &0.50 &0.53 &0.70 &0.68   &203  & 0.30    & $N/A$  &$\times$\\
KCF\cite{henriques2014high}     &0.35 &0.32 &0.38 & 0.31  &0.41 &0.40 &0.44 &0.40    &113  &0.28  &\textbf{$>$100} &\checkmark\\
ECO\cite{danelljan2017eco}        &0.27 &0.28 &0.34 & 0.31  &0.37 &0.37 &0.41 &0.40    &148  &0.24 &{30\textasciitilde 50} &\checkmark\\
IOU\cite{bochinski2017high}  &0.58 &0.53 &0.60 & 0.57 &0.68 &0.64 &0.71 &0.68   &241 &0.23 &{30\textasciitilde 50} &\checkmark\\
POI\cite{yu2016poi} &0.62 &0.61 &0.70 & 0.67 &0.70 &0.70 &0.83 &0.81  &275   &0.27 &{30\textasciitilde 50} &\checkmark\\
\hline\hline
\tabincell{c}{FPVLS$^1$ \\(without refinement)} & 0.56  &0.52  &0.67 &0.64 &0.68 &0.64 &0.80 &0.77   &254 &\textbf{0.09} & 20\textasciitilde40 &\checkmark\\
\hline\hline
\tabincell{c}{FPVLS$^2$\\(Base CNNs)} & \textbf{0.68}  &\textbf{0.67}  &\textbf{0.72} &0.70  &\textbf{0.80} &\textbf{0.77} &\textbf{0.86} &0.83&342 &0.12 & 20\textasciitilde40 &\checkmark\\
\hline\hline
\tabincell{c}{FPVLS$^3$\\(Advanced CNNs)} & \textbf{0.68}  &0.66  &\textbf{0.72} &\textbf{0.71} &\textbf{0.80} &\textbf{0.77} &\textbf{0.86} &\textbf{0.85} &\textbf{367} &0.12 & 20\textasciitilde40 &\checkmark\\

\hline
\end{tabular}
\begin{tablenotes}
      \small
      \item $\uparrow$ $\&$ $\downarrow$ stand for the higher the better and the lower the better. FPVLS$^1$ adopts base CNN models and doesn't activate the trajectories refinements stage. FPVLS$^2$ adopts base CNN models of MTCNN\cite{zhang2016joint}+CosFace\cite{wang2018cosface}. FPVLS$^3$ adopts advanced CNN models of PyramidBox\cite{tang2018pyramidbox}+ArcFace\cite{deng2019arcface}.
    \end{tablenotes}
\label{tbl2}
\end{table*}

\subsection{Parameters}
All the parameters remain unchanged for entire live-streaming video tests. $\mathcal{N}=150$ for every segment contains 150 frames. 10 seconds for buffering as $2*\mathcal{N}/30$. CosFace resizes the cropped face to $112*96$; weight decay is 0.0005. The learning rate is 0.005 for initialization and drops at every 60K iterations. The damping factor for PIAP is default 0.5 as it in AP. Detection threshold for MTCNN is [0.7,0.8,0.9]. The proposal network inherent the threshold from MTCNN, and the resize factor is [0.702]; the IoU rate for NMS is [0.7]. $\alpha=0.95$ for ELR statistic.

\subsection{Experiments Breadth}
As we are the first work on the face pixelation in the video live streaming field, we can hardly find similar methods for comparison. Therefore, we dived into potential algorithms in the online MOT and MFT fields and migrated proper ones to the pixelation tasks. In short, thorough experiments are conducted by \textbf{(i)} employing currently applied commercial pixelation methods, including YouTube Studio~\cite{youtube2020} and Microsoft Azure~\cite{Juliako2020}; \textbf{(ii)} migrating proper state-of-the-art MOT to face pixelation, that is the KCF~\cite{henriques2014high} and ECO~\cite{danelljan2017eco}; \textbf{(iii)} migrating proper state-of-the-art MFT to face pixelation, that is the IOU~\cite{bochinski2017high} and POI~\cite{yu2016poi}; \textbf{(iv)} adopting FPVLS and FPVLS with more advanced based CNNs.

\begin{table*}[h]
\centering
\caption{FPVLS performances under various segment length}
\begin{tabular}{|c||cccc||cccc||c|c|c|}
\hline
\multicolumn{1}{|c||}{\multirow{2}{*}{Segment Length $(\mathcal{N})$}} &{MFPA$\uparrow$}  &{MFPA$\uparrow$}  &{MFPA$\uparrow$}  &{MFPA$\uparrow$} &{MFPA$\uparrow$}  &{MFPP$\uparrow$} &{MFPP$\uparrow$}
	\multirow{2}{*}&{MFPP$\uparrow$}
    &\multicolumn{3}{c|}{Entire Dataset}\\
    \cline{10-12}
    &($HS$)&($LS$)&($HN$)&($LN$)&($HS$)&($LS$)&($HN$)&($LN$)&
    \tabincell{c}{MP$\uparrow$\\(frames)} &{OPR$\downarrow$} &\tabincell{c}{Buffering\\ Time**}\\
\hline \hline
IOU*\cite{bochinski2017high}&0.58 &0.53 &0.60 & 0.57 &0.68 &0.64&0.71 &0.68&241&0.23&N/A\\
\hline \hline
\tabincell{c}{0 frames\\(without refinement)} & 0.56  &0.52  &0.67 &0.64 &0.68 &0.64 &0.80 &0.77   &254 &\textbf{0.09} & 0s\\
\hline \hline
30 frames       &0.61 &0.59 &0.68 & 0.65  &0.72 &0.68 &0.82 &0.78  &278 &0.20 &2s\\
POI*\cite{yu2016poi} &0.62 &0.61 &0.70 & 0.67 &0.70 &0.70 &0.83 &0.81  &275 &0.27 &N/A\\
\hline\hline
60 frames       &0.63 &0.63 &0.70 & 0.68  &0.75 &0.72 &0.83 &0.80  &278 &0.16 &4s\\
90 frames       &0.65 &0.65 &0.71 & 0.68  &0.77 &0.73 &0.85 &0.80  &301 &0.14 &6s\\
120 frames      &0.67 &0.67 &0.72 & 0.69  &0.79 &0.75 &0.86 &0.81  &325 &0.13 &8s\\
\hline\hline
150 frames      &0.68 &0.67 &0.72 & 0.70  &0.80 &0.77 &0.86 &0.83  &342 &0.12 &10s\\
\hline\hline
180 frames      &0.68 &0.67 &0.72 & 0.70  &0.80 &0.77 &0.86 &0.83  &342 &0.12 &12s\\
210 frames      &0.68 &0.67 &0.72 & 0.70  &0.81 &0.78 &0.86 &\textbf{0.84}  &342 &0.12 &14s\\
240 frames      &0.68 &0.67 &\textbf{0.73} & 0.70  &0.81 &\textbf{0.78} &\textbf{0.87} &\textbf{0.84}  &342 &0.12 &16s\\
270 frames      &\textbf{0.69} &\textbf{0.68} &\textbf{0.73} & 0.70  &\textbf{0.82} &\textbf{0.79} &\textbf{0.87} &\textbf{0.84}  &\textbf{360} &0.11 &18s\\
300 frames      &\textbf{0.69} &\textbf{0.68} &\textbf{0.73} & \textbf{0.71}  &\textbf{0.82} &\textbf{0.79} &\textbf{0.87} &\textbf{0.84}  &\textbf{360} &0.11 &20s\\
\hline
\end{tabular}
\begin{tablenotes}
      \small
      \item *: IOU and POI tracker are listed here for comparisons. The segment or buffer section does not exist in their algorithm.
      \item **: the buffer section length is twice of the segment length (\textbf{Proposition 1}), and the default FPS is 30.
    \end{tablenotes}
\label{tbl3}
\end{table*}

\subsection{Evaluation Metrics}
Experiments are afoot on the collected dataset with all the aforementioned algorithms. Referring to the widely accepted metrics in the tracking field~\cite{bernardin2008evaluating}, we propose the Multi-Face Pixelation Accuracy (MFPA), Multi-Face Pixelation Precision (MFPP), and Most Pixelated (MP) frames to indicate the overall performance of an algorithm. Moreover, the Over-Pixelation Ratio (OPR), and the ability to handle the Number of People (NoP) are the other two customized metrics for pixelation tasks. OPR represents the degree of over-pixelation and is reversely related to preserving original pictures to the audiences. OPR is the vital metric in making up the one-sidedness brought by evaluations solely based on the accuracy (MFPA) and precision (MFPP). Higher MFPA and MFPP values of an algorithm may not stand for absolutely better performances on the pixelation tasks unless it still can contain the OPR at a low level. NoP is crucial as we cannot ask for frequent human interventions during the pixelation of live streaming. Thereby, NoP is the key to bring the pixelation algorithm into real applications. This set of metrics is grounded on the essence of face pixelation tasks and fits other privacy-protection-related pixelation tasks as well.

\par
Specifically:
\begin{equation}{\nonumber}
MFPA=1- {\frac{\sum_{t}(m_t+fp_t+mm_t)}{\sum_{t} {g_t}}}
\end{equation}
where $m_t$, $fp_t$, $mm_t$, $g_t$ correspond to the missed pixelation, false positives in pixelation, miss-matched pixelation, and total pixelation labeled in frame $t$. \begin{equation}
\nonumber
MFPP={\frac{\sum_{i,t}d_{i,t} }{\sum_{t} {c_t}}}
\end{equation}
where $d_{i,t}$ is distance between the pixelation of face $i$ and its ground truth on frame $t$. $d_{i,t}$ is measured through the bounding box overlapping ratio in our paper; therefore, the higher, the better. $c_t$ is the total number of matched pixelation in frame $t$. For all the face tracks, MP is the highest length of consecutively and correctly pixelated frame sequence. The other metric OPR indicates the degree of over-pixelation problem by:
\begin{equation}\nonumber
OPR=\frac{\sum_{t}( op_t )} {\sum_{t} {g_t}}
\end{equation}
where the $op_t$ is matched over-pixelation in frame $t$ according to the 'over-pixelation' label.

\subsection{Experiment Results}
\par
Pixelation results of the methods mentioned in \textit{\textbf{C}} are evaluated through the set of metrics proposed in \textit{\textbf{D}}, on the live streaming dataset built in \textit{\textbf{A}}. The results are listed in Table~\ref{tbl2}. YouTube and Azure are offline pixelation methods that adopt tracker based algorithms without calibration. KCF and ECO here are pre-trained with face datasets and optimized for face pixelation tasks. Although ECO applies higher-dimensional features than KCF, the MFPA of ECO is lower than KCF. The problem for the combination of the correlation filters and online learning strategy is any noisy samples yielded by correlation filters profoundly affect the learning network. For state-of-the-art tracking-by-detection methods, IOU and POI trackers get quite attractive results in comparison. POI tracker has two versions. The online version with Mask R-CNN is used here. POI tracker shares some cognate theories with FPVLS on detection, embedding, and the following trajectory generation. Self-evidently, FPVLS without the trajectory refinement stage somehow achieves similar performances with the POI tracker.
\par
Overall, FPVLS achieves better performance in MFPA, MFPP, MP, and OPR on every single subdataset. Particularly, FPVLS noticeably increases the MP metric on the entire dataset, indicating FPVLS yields longer pixelation. Referring to the last two rows of Table~\ref{tbl2}, FPVLS with advanced CNNs substitutes the base face detection and embedding networks with more advanced deep networks. PyramidBox~\cite{tang2018pyramidbox} gains around 3\% accuracy promotion comparing to MTCNN\cite{zhang2016joint} on their benchmark tests. Likewise, ArcFace~\cite{deng2019arcface} owns 1\%-7\% accuracy promotion comparing to CosFace\cite{wang2018cosface} on their benchmark tests. However, advanced CNN models did not bring FPVLS significant enough boosts on any five metrics. This proves our FPVLS framework is robust to the selection of CNN models. On the other hand, if worse enough detection and embedding networks are adopted, FPVLS will end up in a mess according to our further tests.

\begin{table*}[htbp]
\centering
\caption{Different gains brought by two-sample test based on ELR and the Gaussian Process model.}
\begin{tabular}{|c||c||c|c|c|c|c|}
\hline
\tabincell{c}{Segment Length $(\mathcal{N})$} &\tabincell{c}{Proposal Network\\Recall Rate}   &Methods & \tabincell{c}{MFPA Gain*$\uparrow$}  &\tabincell{c}{MFPP Gain*$\uparrow$}   &\tabincell{c}{MP Gain*$\uparrow$}   &\tabincell{c}{OPR Gain*$\downarrow$ }   \\
\hline\hline
\multirow{2}{*}{30} &\multirow{2}{*}{0.33}
& {$T_{ELR}$}  & {0.05} &{0.04} & {30}  & {0.11}  \\
\cline{3-7}
& &GP & {-0.02} &{-0.03} & {0}  &{0.16}   \\
\hline\hline

\multirow{2}{*}{90} &\multirow{2}{*}{0.33}
& {$T_{ELR}$}  & {0.08} &{0.07} & {57}  & {0.05}  \\
\cline{3-7}
& &GP & {0.03} &{0.03} & {19}  &{0.08}   \\
\hline\hline

\multirow{2}{*}{150} &\multirow{2}{*}{0.32}
& {$T_{ELR}$}  & {0.10} &{0.08} & {88}  & {0.03}  \\
\cline{3-7}
& &GP & {0.05} &{0.04} & {31}  &{0.08}   \\
\hline\hline

\multirow{2}{*}{210} &\multirow{2}{*}{0.33}
& {$T_{ELR}$}  & {0.10} &{0.09} & {88}  & {0.03}  \\
\cline{3-7}
& &GP & {0.08} &{0.07} & {57}  &{0.05}   \\
\hline\hline

\multirow{2}{*}{300} &\multirow{2}{*}{0.32}
& {$T_{ELR}$}  & {0.11} &{0.10} & {94}  & {0.02}  \\
\cline{3-7}
& &GP & {0.10} &{0.10} & {69}  &{0.04}   \\
\hline
\end{tabular}
\begin{tablenotes}
      \small
      \item $\qquad\quad\quad$ $T_{ELR}$ is the two-sample test based on ELR, GP is the Gaussian Process model.
      \item $\qquad\quad\quad$ *: Gains are the boosts brought by the $T_{ELR}$ or the GP model through the trajectories refinement stage .
    \end{tablenotes}
\label{tbl5}

\end{table*}

\subsection{Ablation Study}
The performance of FPVLS is the joint efforts of the raw face trajectory generation and the trajectory refinement stages. Our ablation study focuses on exploring the promotion brought by each of the stages. Therefore, we discuss the effects of parameters (preprocessing), clustering algorithms (raw trajectories generation), and the two-sample test based on ELR (trajectory refinement) in sequence.

\par
\textbf{Parameters.}
Except for the buffering length, the rest parameters own to pre-trained models of face detection and embedding networks. Several sets of parameters are published for the pre-trained models. In fact, the parameter tuning affects less than $1\%$ on their benchmark tests, while the base CNN model substitution in Table~\ref{tbl2} introduces $3-7\%$ boosts. So as FPVLS is robust to CNN model selections, it is also robust to parameter selections for pre-trained models.
\par
The buffer section length is a user-defined setting in our method. Ideally, the longer buffering section helps accuracy, but lower the user experience. The performances of FPVLS with different segment length $\mathcal{N}$ is shown in Table~\ref{tbl3}. On our entire dataset, there is just a slight difference (+0.011 in MFPA) when increasing segment length from 5s to 10s. When 4s or less is used, we incur a drop up to -0.12 in MFPA. We adopt the 5s (150 frames) segment length to achieve the shortest buffering time with the smallest drawbacks, thereby balancing the efficiency and user experience. Besides, the network capacity and video compression introduce primal latency in seconds level. Live streaming requires low-latency when remote face-to-face interaction happens. In online video live streaming, the only unilateral interaction is typing live comments. Therefore, if set to a reasonable value (commonly less than 10 seconds), the buffering lag is transparent to both streamers and audiences.
\par
Moreover, we move the metrics of IOU and POI trackers into Table~\ref{tbl3} for comparisons. FPVLS, even with the segment length of 0 frames (completely without the trajectory refinement stage), surpasses the IOU tracker, and FPVLS with the segment length of 30 frames (2s buffering time) reaches comparable MFPA and MFPP value with the POI tracker. FPVLS, with the segment length of 60 frames (4s buffering time), surpasses the POI tracker. POI tracker applies DN search~\cite{du2016online}, which excludes the outliers and smooths the trajectories according to conjunctive tracklets. POI retains higher MFPA and MFPP value than the FPVLS with extremely short segment length (0 or 30 frames). However, the POI tracker's notably higher OPR metric implies the DN search also leads to massive over-pixelation while smoothing the trajectories. That means, compared with the state-of-the-art MFT-based pixelation methods, FPVLS has better performance even when a very low buffer section is set.

\par
\textbf{Subdatasets.}
According to Table~\ref{tbl2} and Table~\ref{tbl3}, the High-resolution Naive-scene ($HN$) subdataset is the one favorable to all the tested algorithms. All these algorithms get their best performance on the $HN$ subdataset. Regarding Table~\ref{tbl3}, when $\mathcal{N}$ increases linearly, on the $LN$ and $HN$ subdataset, the changing rate of MFPP is more flattened than that of MFPA. Moreover, the changing rate of MFPP also quickly reaches the turning point on these two subdatasets. In other words, the number of people showing up in the videos ($S$ or $N$) has a more substantial impact on MFPP than the resolution ($H$ or $L$). On the other side, MFPA is influenced simultaneously by the resolution and the number of people showing up in videos.

\begin{table}[htbp]
  \vspace{1em}
\centering
\caption{Clustering accuracy and efficiency}
\begin{tabular}{|c||c|c|c|c|c|}
\hline
{Dataset} &{Method} &Purity& {Accuracy}  &\tabincell{c}{Cluster number\\(Clustered/Truth)}   &Time(s)        \\
\hline\hline
\multirow{3}{*}{$HS$}
& {AP*}  & {0.81} &{0.79} & {22/17}  & {1.82}   \\
\cline{2-6}
& {PAP}  & {0.86} &{0.83} &{18/17}& {1.82}   \\
\cline{2-6}
& {PIAP}  & {0.86}&{0.83} & {18/17} & {0.08}   \\
\hline\hline
\multirow{3}{*}{$LS$}
& {AP*}  & {0.72} &{0.70} & {64/43}& {3.05}   \\
\cline{2-6}
& {PAP}  & {0.79} &{0.73}& {49/43}& {3.05}   \\
\cline{2-6}
& {PIAP}  & {0.78} &{0.73}& {49/43}& {0.19}   \\
\hline\hline
\multirow{3}{*}{$HN$}
& {AP*}  & {0.86} &{0.79} & {8/6} & {0.75}   \\
\cline{2-6}
& {PAP}  & {0.91} &{0.89}& {6/6}& {0.75}   \\
\cline{2-6}
& {PIAP}  & {0.90}&{0.89}& {6/6} & {0.04}   \\
\hline\hline
\multirow{3}{*}{$LN$}
& {AP*}  & {0.89}&{0.86} & {11/7}& {0.80}   \\
\cline{2-6}
& {PAP}  & {0.92}&{0.90}& {9/7} & {0.80}   \\
\cline{2-6}
& {PIAP}  & {0.91}&{0.89}& {9/7} & {0.03}   \\
\hline
\end{tabular}
\begin{tablenotes}
      \small
      \item *: The detection results are cleaned before fed to AP.
    \end{tablenotes}
\label{tbl4}
\vspace{1em}
\end{table}

\begin{figure*}[htbp]
\centering
\vspace{-0.35cm}
\subfigtopskip=2pt
\subfigbottomskip=2pt \subfigcapskip=-3pt
\subfigure[]{
\includegraphics[width=3.45cm]{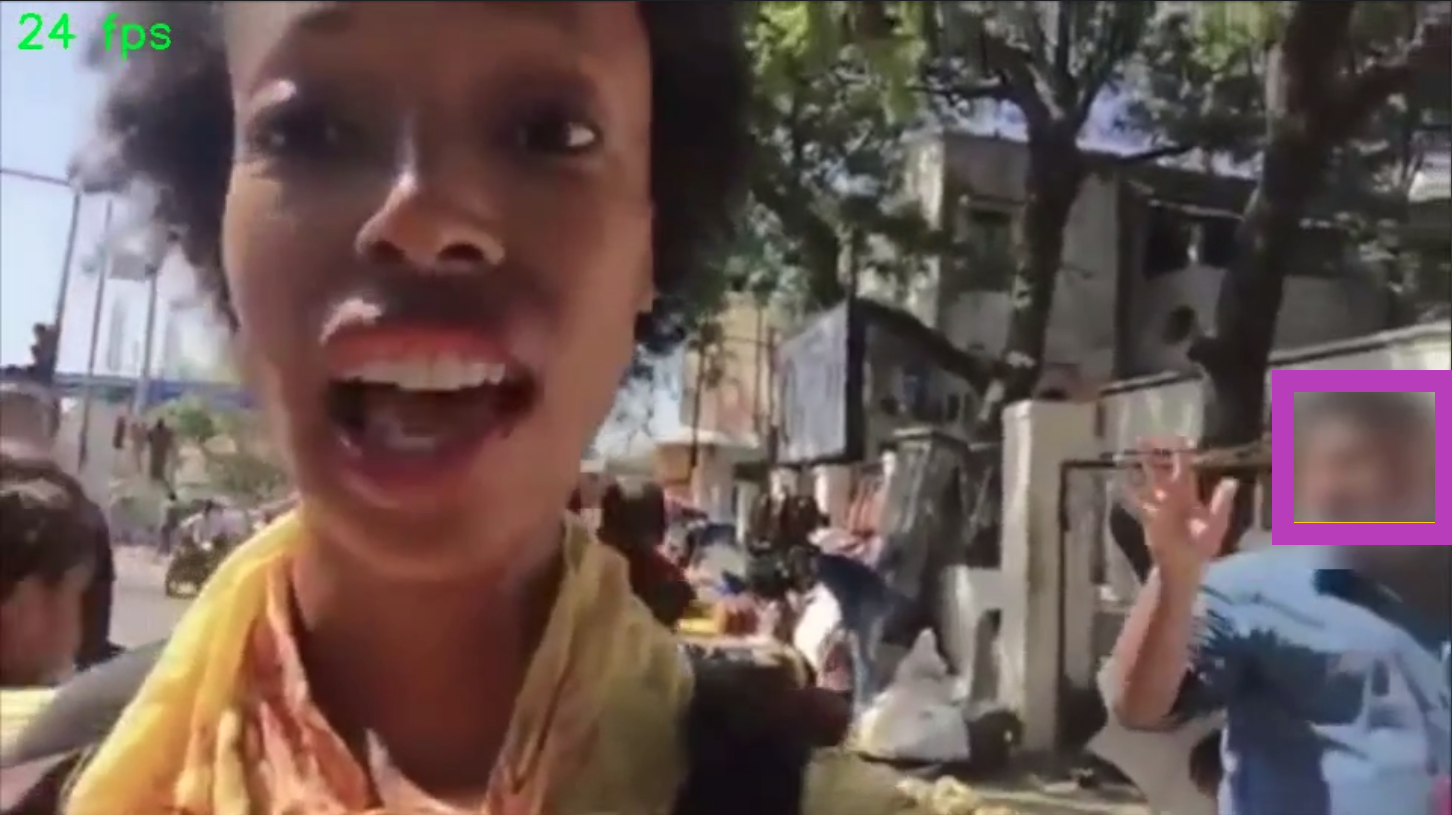}
\includegraphics[width=3.45cm]{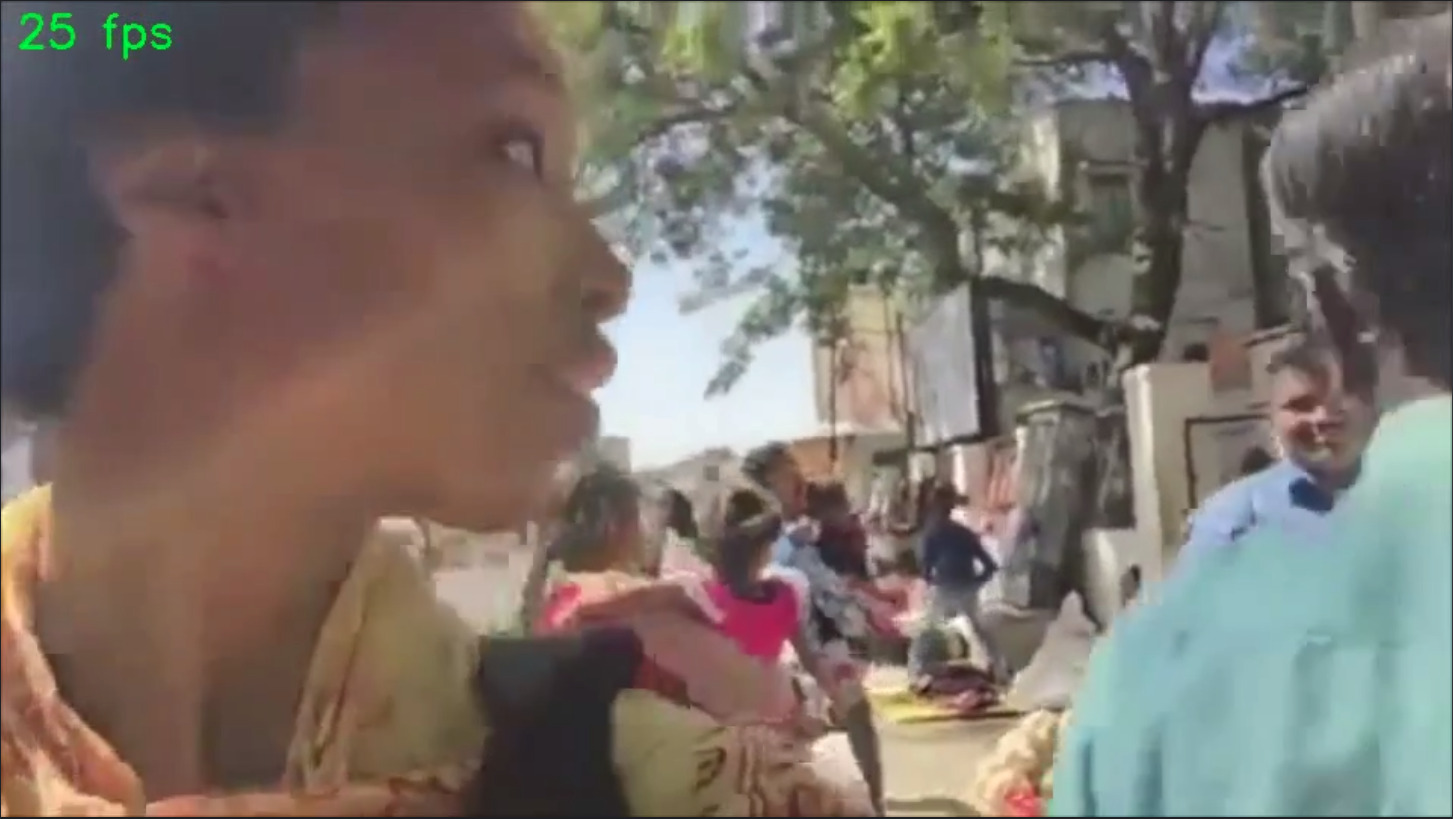}
\includegraphics[width=3.45cm]{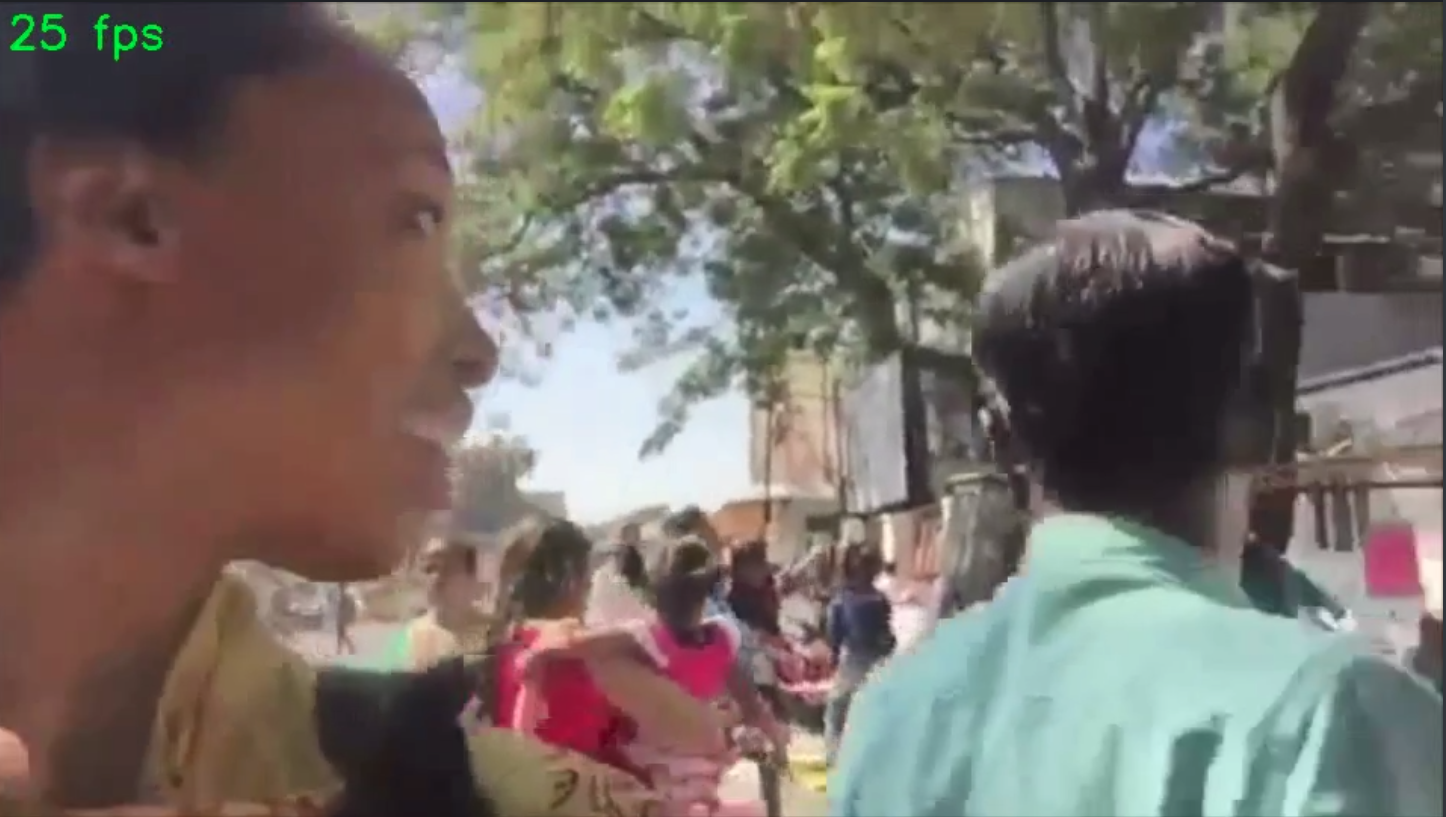}
\includegraphics[width=3.45cm]{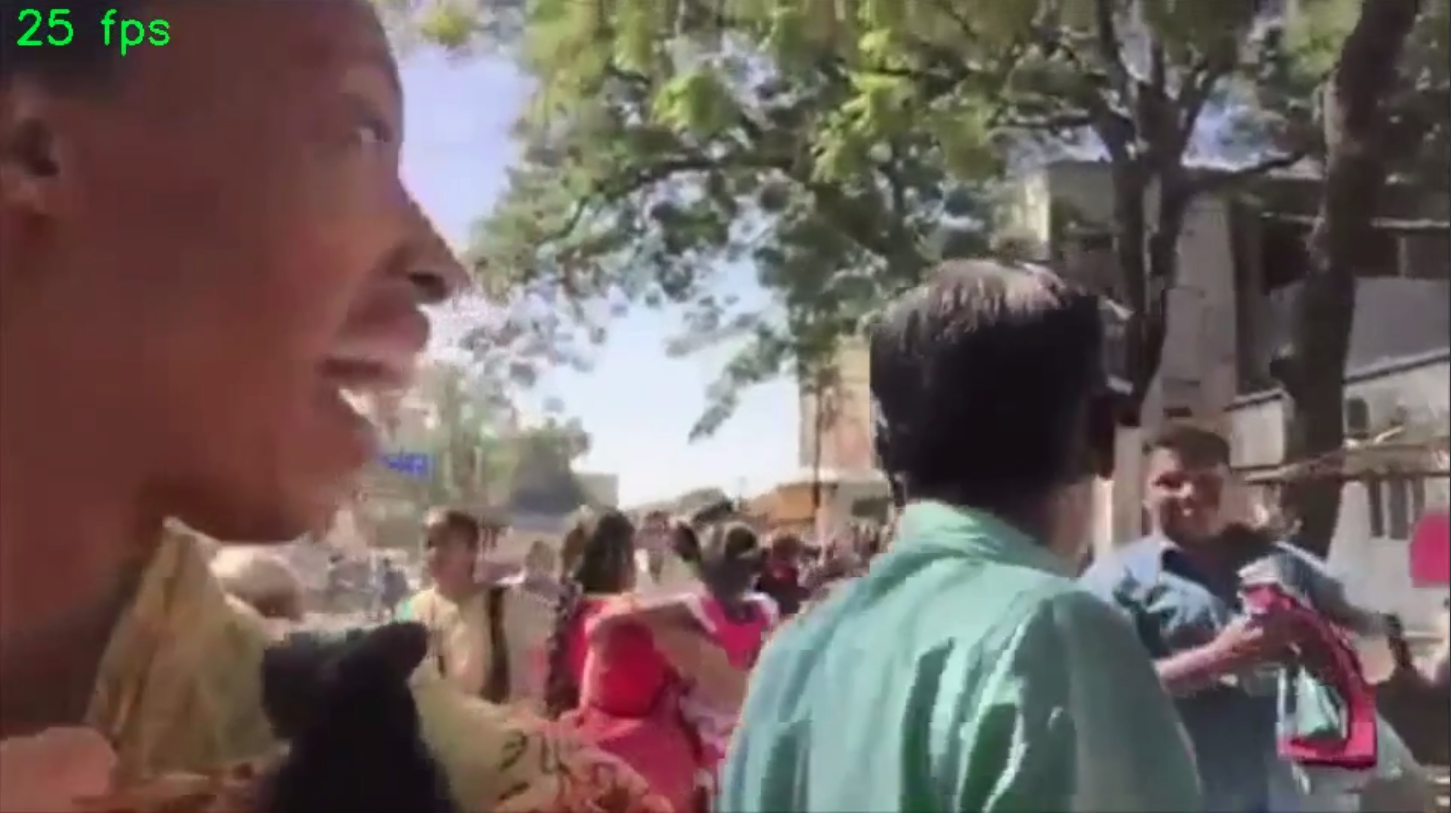}
\includegraphics[width=3.45cm]{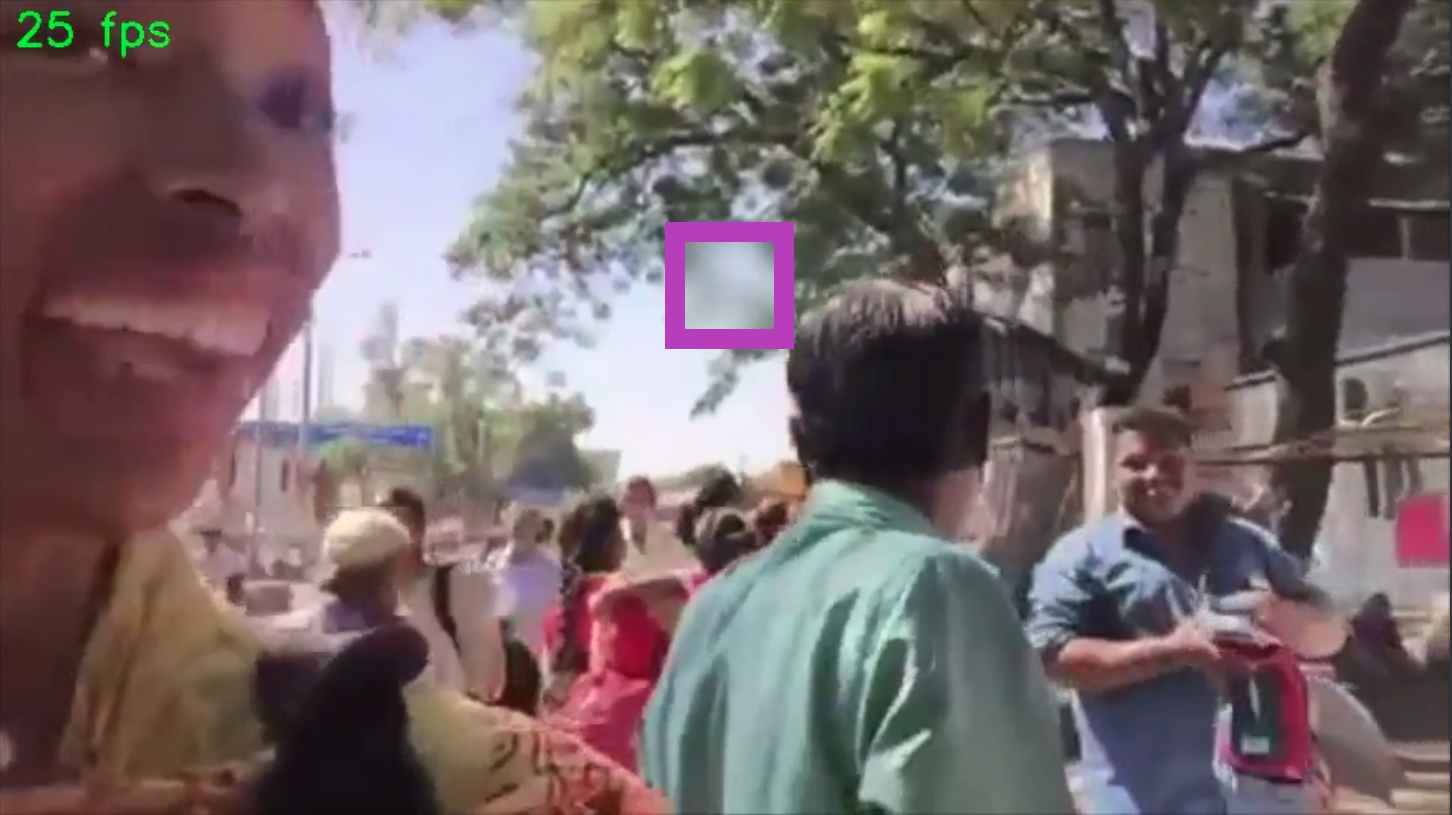}
}\vskip -2pt
\quad
\subfigure[]{
\includegraphics[width=3.45cm]{INT/42.png}
\includegraphics[width=3.45cm]{INT/43.png}
\includegraphics[width=3.45cm]{INT/44.png}
\includegraphics[width=3.45cm]{INT/45.png}
\includegraphics[width=3.45cm]{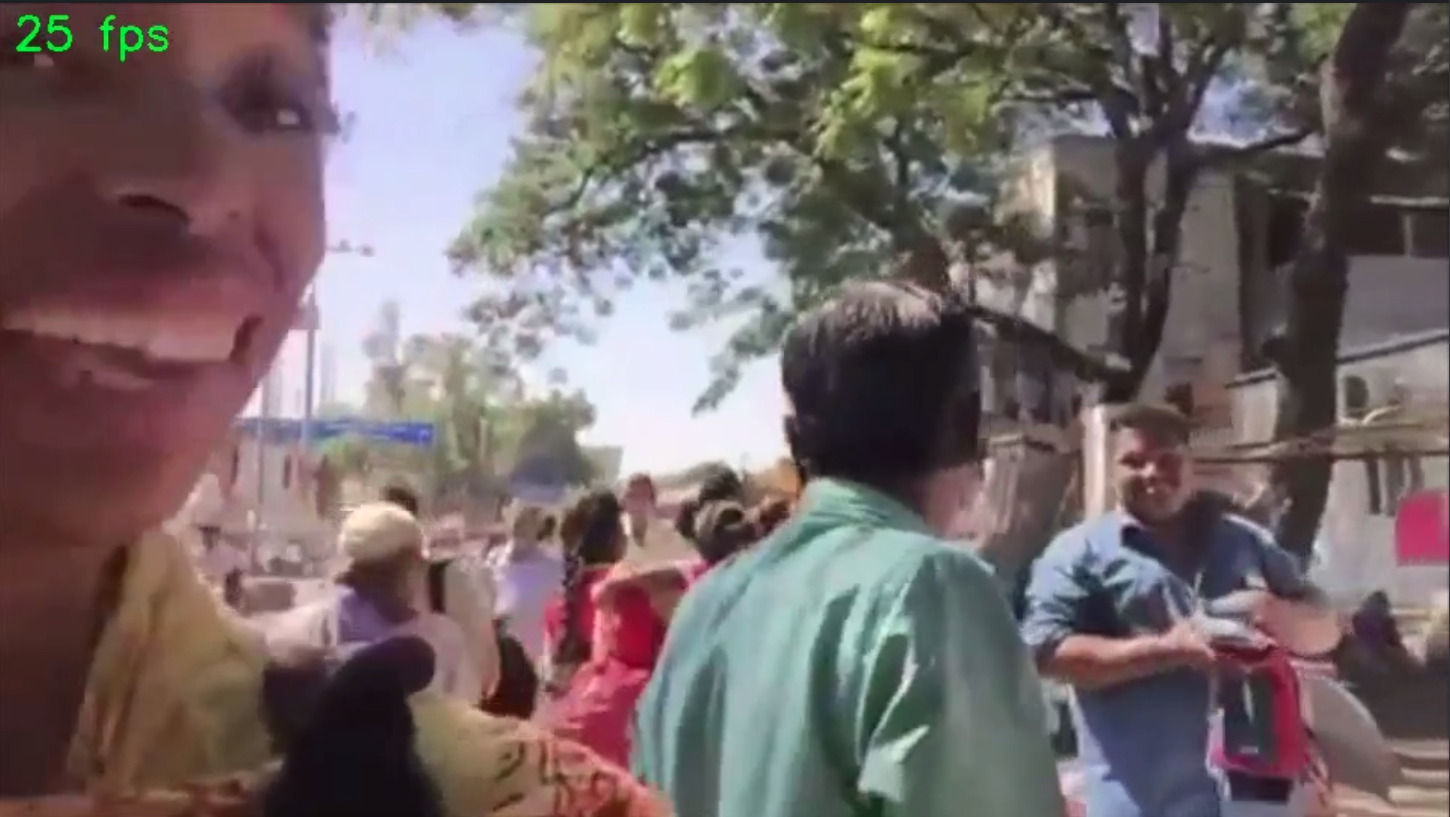}
}\vskip -2pt
\quad
\subfigure[]{
\includegraphics[width=3.45cm]{INT/42.png}
\includegraphics[width=3.45cm]{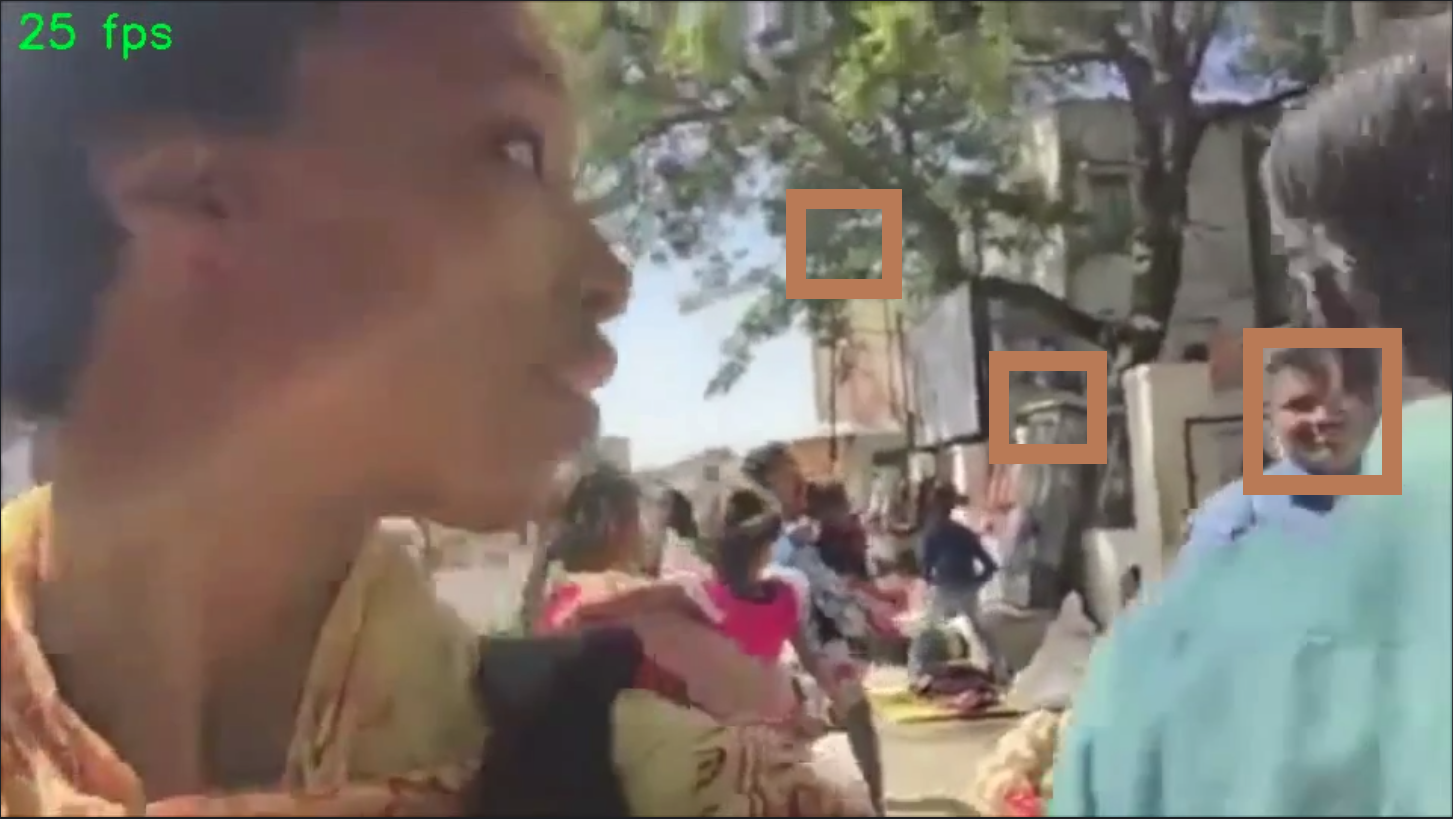}
\includegraphics[width=3.45cm]{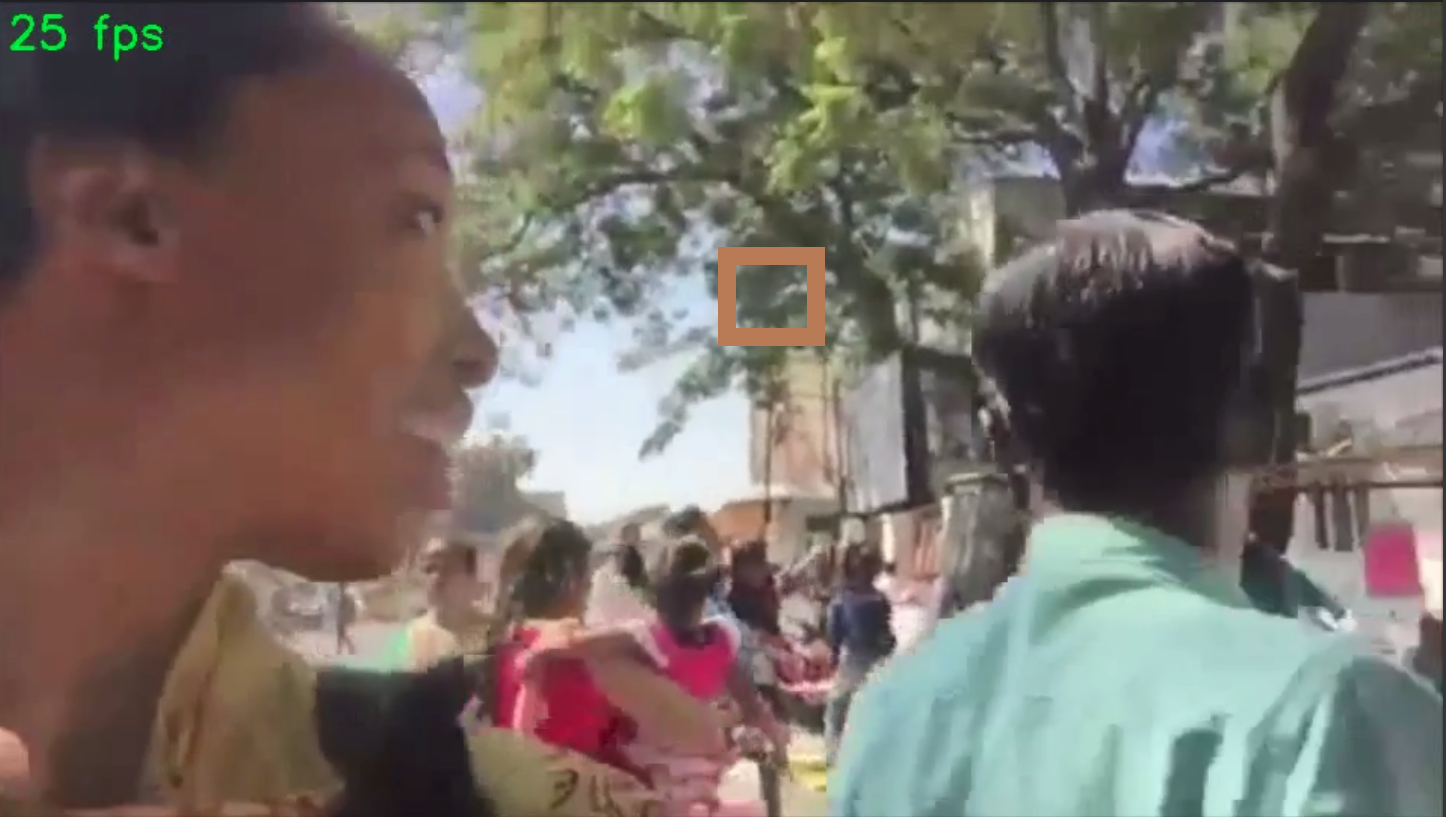}
\includegraphics[width=3.45cm]{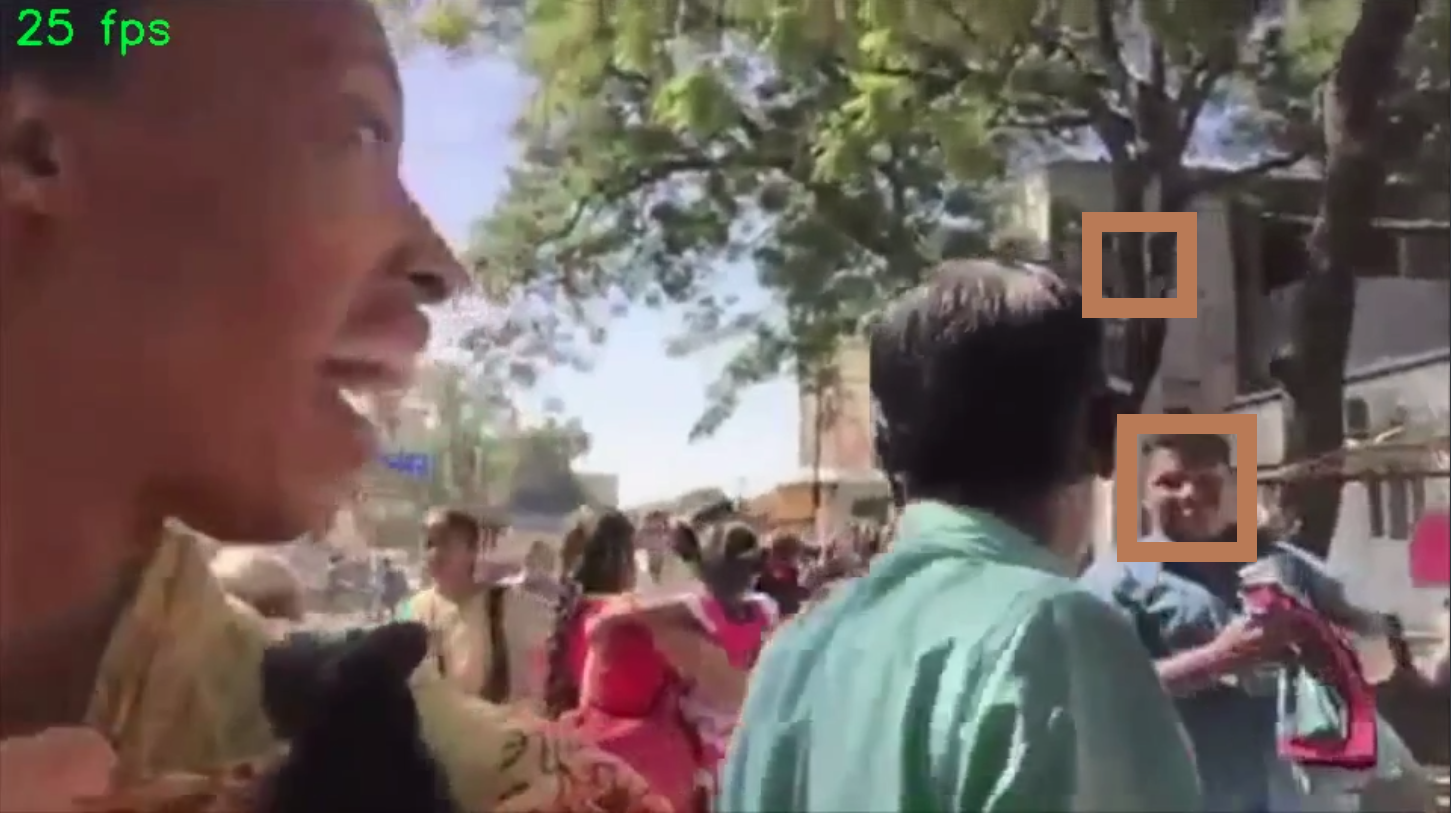}
\includegraphics[width=3.45cm]{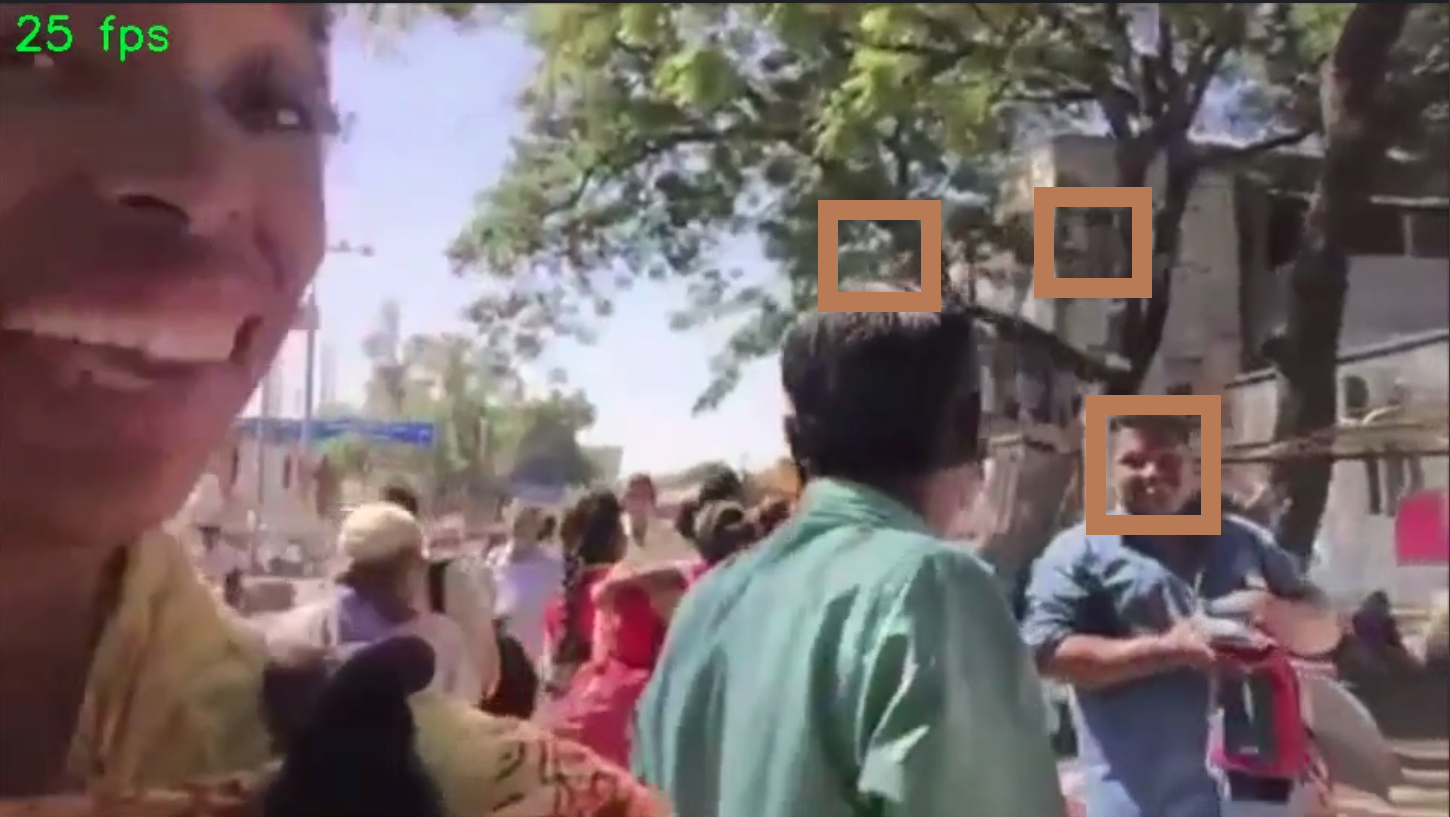}
}\vskip -2pt
\quad\subfigure[]{
\includegraphics[width=3.45cm]{INT/42.png}
\includegraphics[width=3.45cm]{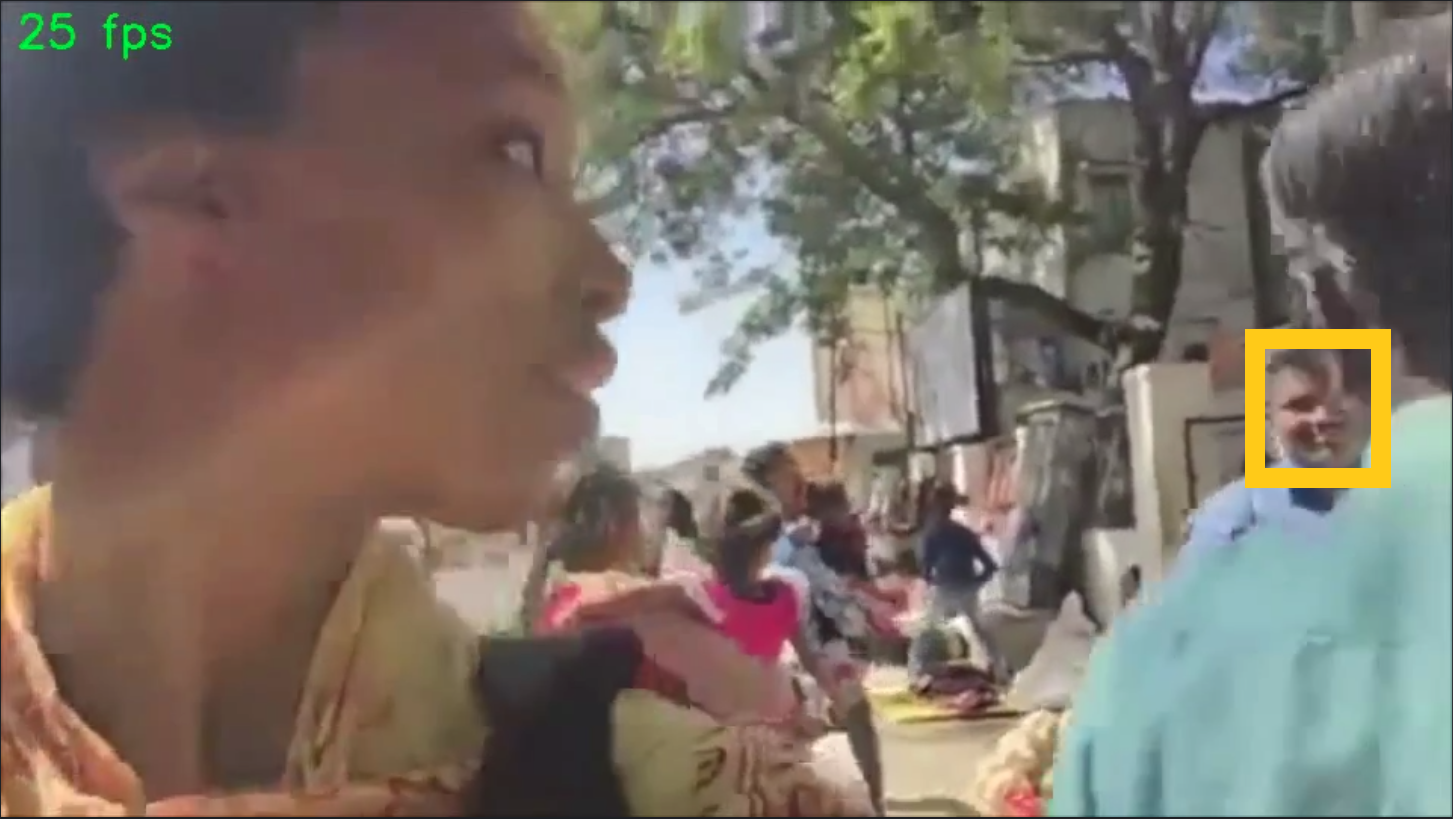}
\includegraphics[width=3.45cm]{INT/44.png}
\includegraphics[width=3.45cm]{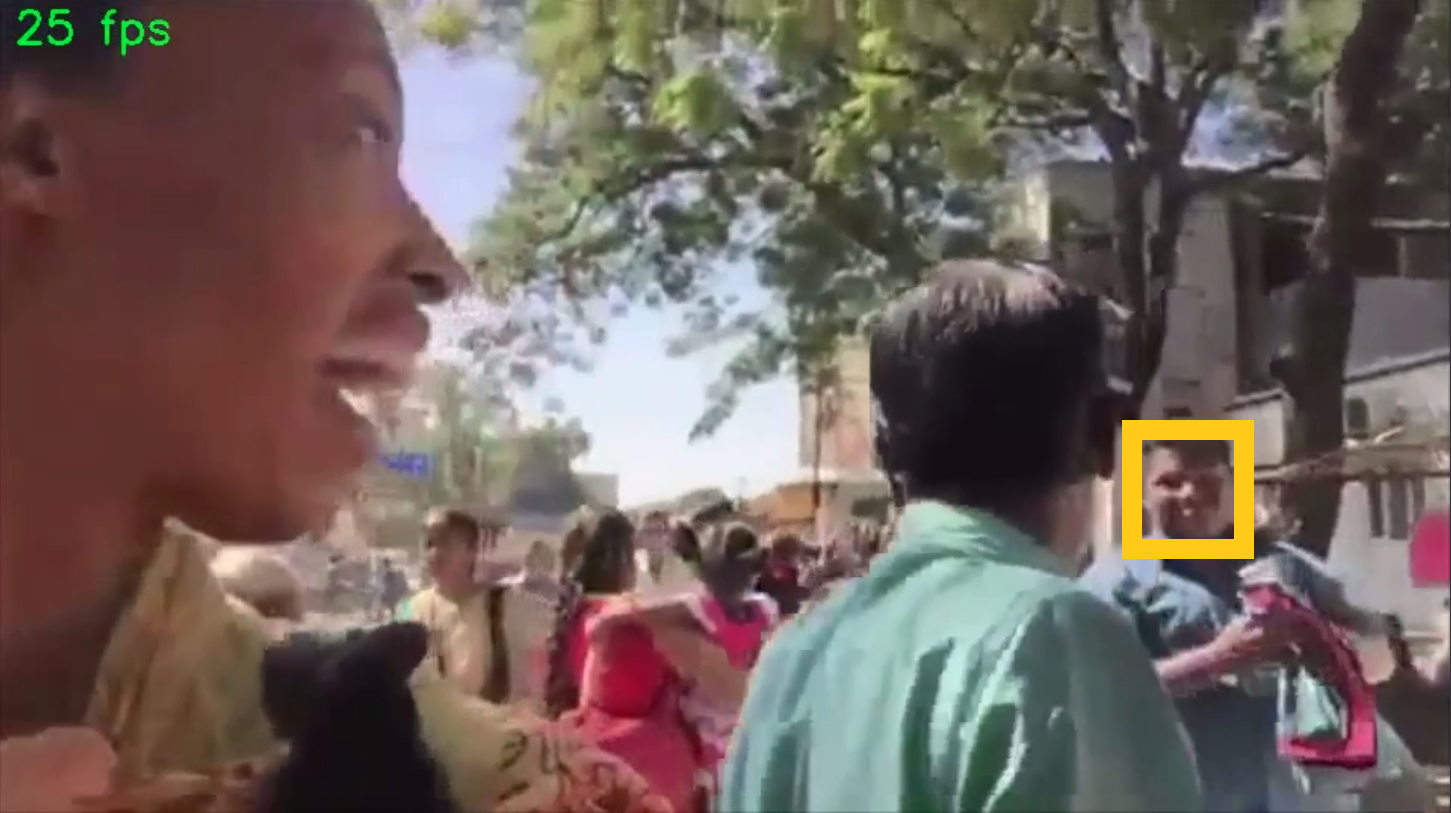}
\includegraphics[width=3.45cm]{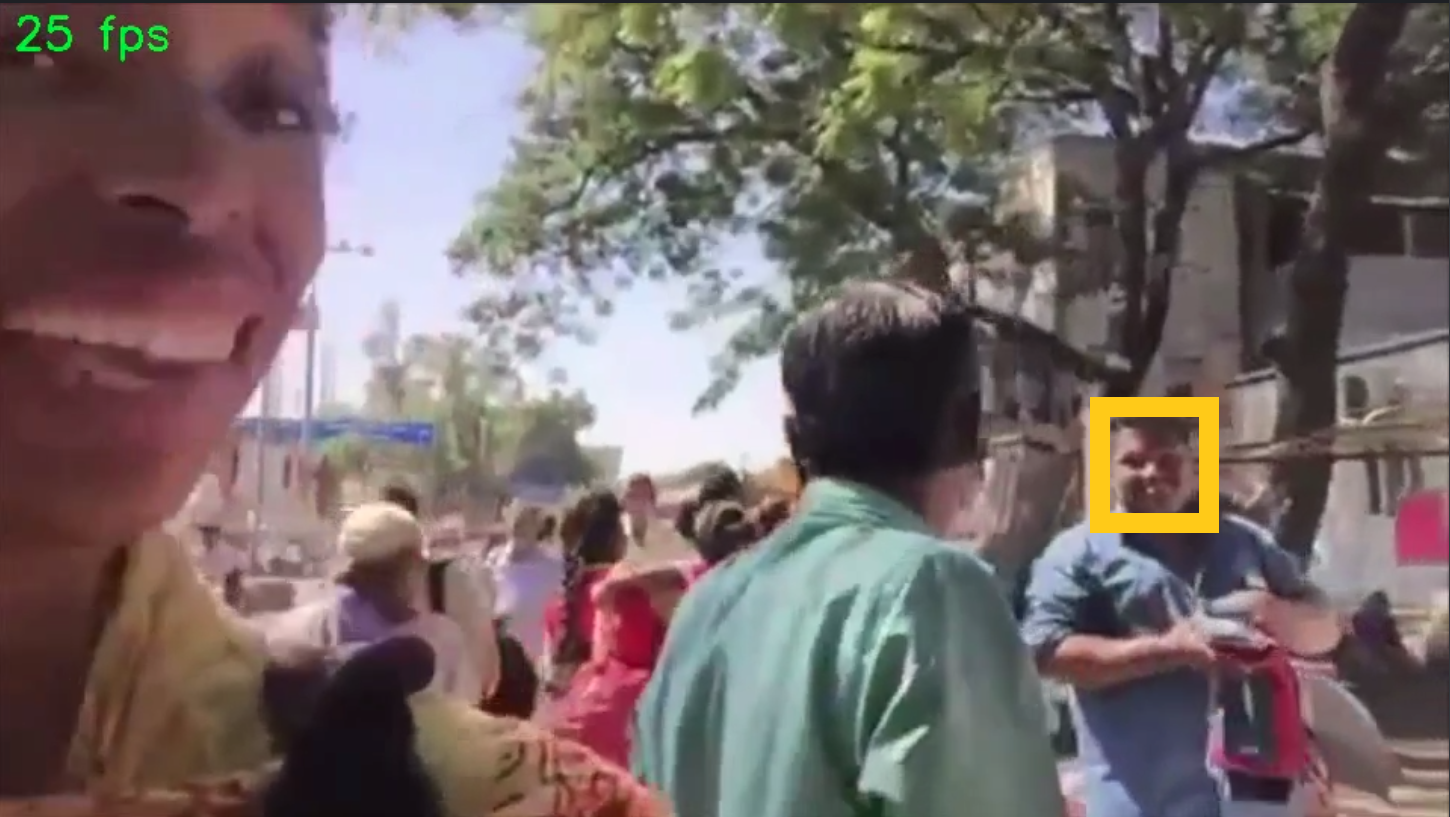}
}
\caption{
The promotion brought by each stage of FPVLS. (a) is the live video pixelated solely through face detection and embedding networks. (b) is processed on the result of (a) through PIAP clustering. (c) processes the results of (b) through the proposal network. (d) is the result of processing (c) through the two-sample test based on ELR.}

\label{fig5}
\end{figure*}

\par
\textbf{Clustering Algorithm.}
We built Table~\ref{tbl4} to investigate the different clustering performances among classic AP, PAP (AP with positioned information), and PIAP on the collected live streaming dataset. In such a way, we can figure out how the PIAP facilitates the performance of the entire FPVLS. As no comparable or similar clustering algorithm handles object clustering under noises and ill-defined cluster numbers, we conduct longitudinal comparisons among AP, PAP, and PIAP under $\mathcal{N}$=150. If detection full of false positives is fed to AP directly, the noise-sensitive AP will produce meaningless indexes for reference. Since PIAP excludes false positives as outliers, we trace back the raw vectors of faces remaining in PIAP clustering results. These raw face vectors are sent to AP, and in Table~\ref{tbl4}, AP only handles the embedding variations. Comparing AP with PAP, the positioned affinities significantly boost the clustering purity, indicating the embedding variations are fixed to a large extent. According to PIAP and PAP results, our incremental message passing algorithm solves the time-consuming problem of AP almost without affecting the purity. The value of the Time column is the seconds needed to reach consensus.

\par
\textbf{Two-sample test based on ELR \& GP model.}
We also explicitly explore the metrics gains or boosts brought by the other core stage of FPVLS: the trajectory refinement. We look into the differences in metrics gains while applying the two-sample test based on ELR and the Gaussian Process (GP) model. The trajectory refinement is the downstream stage of the raw trajectories generation stage. Therefore, the comparison is conducted under the same segment length to vanish the basic raw trajectories' discrepancies. Furthermore, the proposal networks recall rate is also examined in case the behavior of the proposal network varies under different segment lengths.

\par
Table~\ref{tbl5} is the result. Differences in five metrics when applying the proposed two-sample test based on ELR and the commonly used Gaussian Process model are shown. The recall rate of the proposal networks is not affected by the value of the segment length. Therefore, the raw input of the two-sample test and GP model is generally the same. Different metrics' gains are the boosts brought by the trajectories refinement stage using the two-sample test or GP model. The two-sample test based on ELR is denoted as $T_{ELR}$ in the table. $T_{ELR}$ outperforms the GP model on every metric under every segment length. The GP model even claims negative gains in MFPA and MFPP under extremely short segment length, where the initial assumption that proposed faces comply with a Gaussian distribution cannot hold.

\begin{figure*}[htbp]
\centering
\includegraphics[scale=0.30]{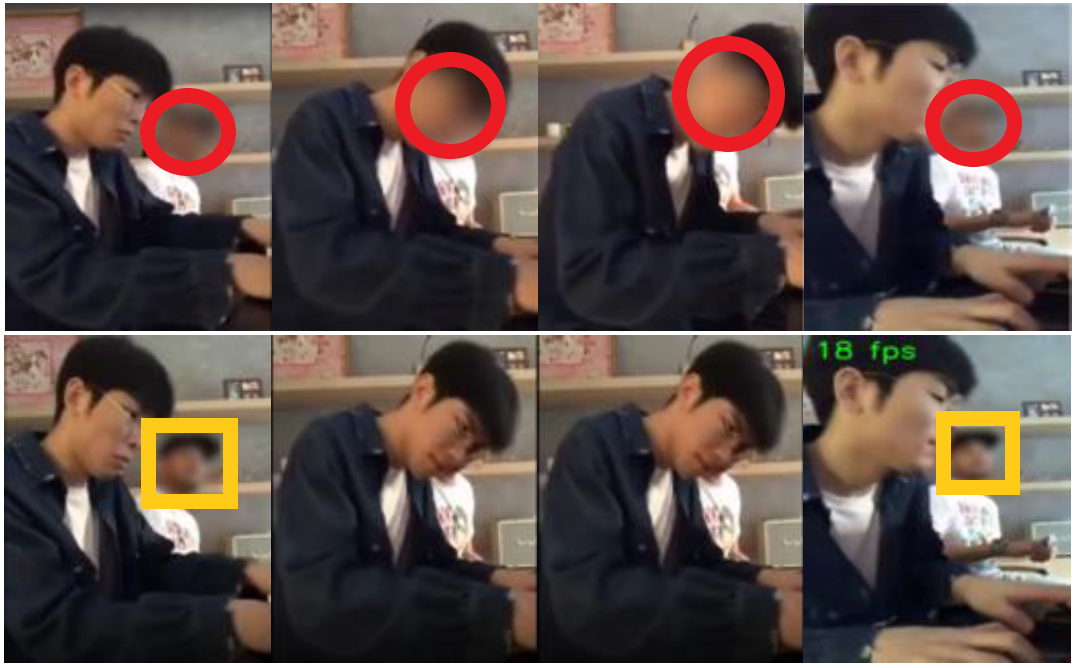}
\includegraphics[scale=0.30]{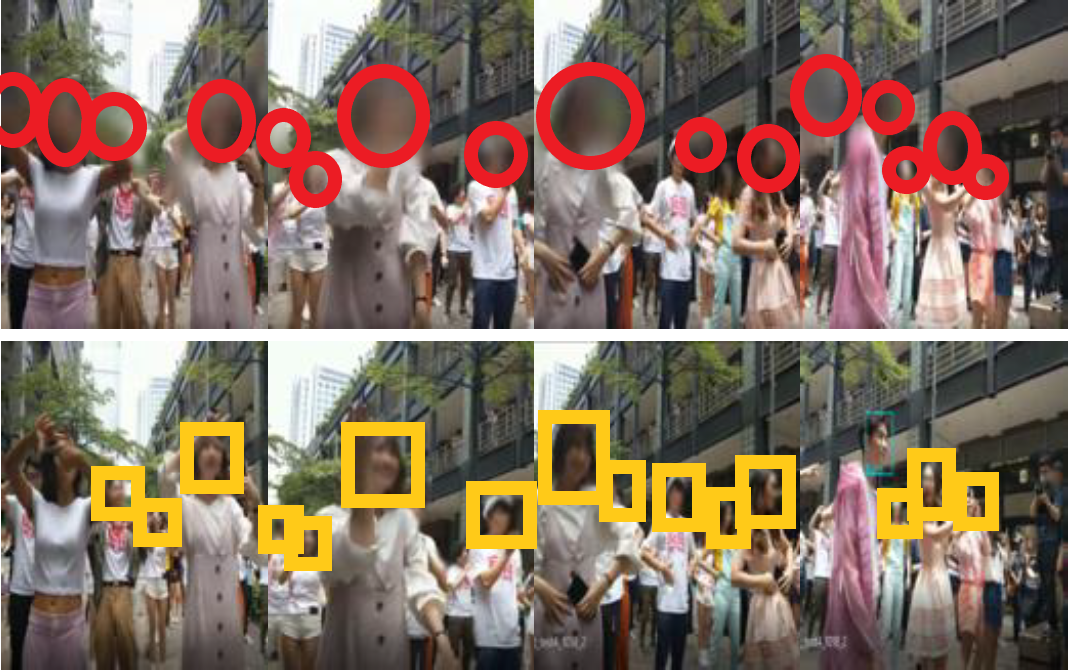}
\caption{Pixelation comparisons between FPVLS (yellow) and the POI tracker (red) in low-resolution scenes.}
\label{fig6}
\end{figure*}

\subsection{Quantitative Analysis}
Fig.~\ref{fig5} shows the specific function of each algorithm in a real streaming scene. Fig.~\ref{fig5}(a) is the same scene that we demonstrated in Fig.~\ref{fig1} pixelated solely through face detection and recognition networks. For comparison, we feed around 1000 pictures of the streamer in advance to pre-train the recognition network. The classification result of the recognition network is directly used for pixelation. In the rightmost snapshot of Fig.~\ref{fig5}(a), the detection and recognition network failed to capture the hawker's face and produced a false positive in purple rectangular. Also, in the second and the fourth snapshots, the detection network failed to locate the hawker's face. Then, we fix the results of (a) through PIAP clustering. The rightmost snapshot in (b) shows PIAP excluded the false positives caused by the detection network. (b) corresponds to the results after the raw trajectories generation stage. Further, the proposal network is applied to raw trajectories to retrieve the detection lost. In (c), the proposal network yielded the lost faces along with some other non-face areas marked in brown rectangles. The ELR statistics managed to cull the non-faces through the two-sample test in (d). The result in (d) is fully processed by FPVLS except for the final pixelation. The effect of both the raw face trajectory generation and trajectory refinement stages are explicit and significant regarding (a)-(d).

\par
More qualitative results\footnote{More video results on:\href{https://knightzjz.github.io/}{https://FPVLS.github.io}} are shown in Fig.~\ref{fig6} to demonstrate the performance of FPVLS under noisy, low-resolution video live streaming scenes. POI tracker generates mosaics in red circles, and FPVLS generates those in yellow rectangular. The left-upper row shows the same over-pixelation problem also occurs in the POI tracker. In red circles, excessive and puzzling mosaics are placed on the frontal streamer's faces while the crossover of two people's faces happens. However, the temporary invisible non-streamer's face will not be pixelated by FPVLS, therefore, we leave the streamer's face clean. The right-side rows show a crowded flash-event, which is hard to be dealt with. Limited by tracklets, trackers cannot instantly recover from the drifting in noisy scenes. Besides other misplaced mosaics, the tracker blurs the face of the streamer who wears a pink trench coat. Although not perfect, FPVLS manages to pixelate most of the faces correctly, especially leaves the streamer's face (marked in blue rectangular) un-pixelated.

\subsection{Efficiency}
The main cost of FPVLS is on the face embedding algorithm. The embedding for one face takes 10-15ms on our i7-7800X, dual GTX1080, 32G RAM machine. The face detection, including compensation for a frame, takes 3ms. Another time-costing part is the initialization of PIAP; the initialization process takes 10-30ms depending on the case. Each incremental propagation loop takes 3-5ms. FPVLS generally satisfies the real-time efficiency requirement, and under extreme circumstances that contain many faces, we can reduce the sampling rate of the video frames to improve the efficiency of FPVLS.
\section{Conclusions}
Leveraging PIAP clustering to generate raw face trajectories and the two-sample test based on ELR to refine the raw trajectories, FPVLS manages to accomplish the irrelevant people's face tracking and pixelation problem in video live streaming. PIAP deals with the ill-defined cluster number issue and endows the classic AP with noise-resistance and time-saving merits. The two-sample test based on empirical likelihood ratio statistics outperforms the widely used Gaussian process in the trajectory refinement. With most advanced mobile phone chips like Apple A12, FPVLS could be deployed on smartphones and brought into real applications without much burden.
\par
As we discussed in section \uppercase\expandafter{\romannumeral4}, the low-resolution videos are still quite challenging because deep networks can drop to a deficient performance and break the robustness of FPVLS under low-resolution scenarios. Our future work will focus on the improvement of accuracy and robustness under low resolution streaming scenes.
\section{Acknowledgment}
The authors would like to deliver special thanks to Sin-Teng Wong for her attentive work on the video live streaming dataset annotations.
\par
This work was partly supported by the University of Macau under Grants: MYRG2018-00035-FST and MYRG2019-00086-FST, and the Science and Technology Development Fund, Macau SAR (File no. 0034/2019/AMJ, 0019/2019/A).

\bibliographystyle{IEEEtran}
\bibliography{IEEEabrv,tiffs}

\begin{thebibliography}{10}
\providecommand{\url}[1]{#1}
\csname url@samestyle\endcsname
\providecommand{\newblock}{\relax}
\providecommand{\bibinfo}[2]{#2}
\providecommand{\BIBentrySTDinterwordspacing}{\spaceskip=0pt\relax}
\providecommand{\BIBentryALTinterwordstretchfactor}{4}
\providecommand{\BIBentryALTinterwordspacing}{\spaceskip=\fontdimen2\font plus
\BIBentryALTinterwordstretchfactor\fontdimen3\font minus
  \fontdimen4\font\relax}
\providecommand{\BIBforeignlanguage}[2]{{%
\expandafter\ifx\csname l@#1\endcsname\relax
\typeout{** WARNING: IEEEtran.bst: No hyphenation pattern has been}%
\typeout{** loaded for the language `#1'. Using the pattern for}%
\typeout{** the default language instead.}%
\else
\language=\csname l@#1\endcsname
\fi
#2}}
\providecommand{\BIBdecl}{\relax}
\BIBdecl

\bibitem{vo2017optimal}
N.-S. Vo, T.~Q. Duong, H.~D. Tuan, and A.~Kortun, ``Optimal video streaming in
  dense 5g networks with d2d communications,'' \emph{IEEE Access}, vol.~6, pp.
  209--223, 2017.

\bibitem{faklaris2016legal}
C.~Faklaris, F.~Cafaro, S.~A. Hook, A.~Blevins, M.~O'Haver, and N.~Singhal,
  ``Legal and ethical implications of mobile live-streaming video apps,'' in
  \emph{Proceedings of the 18th International Conference on Human-Computer
  Interaction with Mobile Devices and Services Adjunct}.\hskip 1em plus 0.5em
  minus 0.4em\relax ACM, 2016, pp. 722--729.

\bibitem{zimmer2017law}
F.~Zimmer, K.~J. Fietkiewicz, and W.~G. Stock, ``Law infringements in social
  live streaming services,'' in \emph{International Conference on Human Aspects
  of Information Security, Privacy, and Trust}.\hskip 1em plus 0.5em minus
  0.4em\relax Springer, 2017, pp. 567--585.

\bibitem{stewart2016up}
D.~R. Stewart and J.~Littau, ``Up, periscope: Mobile streaming video
  technologies, privacy in public, and the right to record,'' \emph{Journalism
  \& Mass Communication Quarterly}, vol.~93, no.~2, pp. 312--331, 2016.

\bibitem{youtube2020}
\BIBentryALTinterwordspacing
Y.~Help, \emph{Blur your videos}, 2020 (accessed Apr 17th, 2020). [Online].
  Available: \url{https://support.google.com/youtube/answer/9057652?hl=en}
\BIBentrySTDinterwordspacing

\bibitem{youtube2020kid}
\BIBentryALTinterwordspacing
Y.~O. Blog, \emph{Better protecting kids' privacy on YouTube}, January 6, 2020
  (accessed Apr 17th, 2020). [Online]. Available:
  \url{https://youtube.googleblog.com/2020/01/betterprotecting-kids-privacy-on-YouTube.html}
\BIBentrySTDinterwordspacing

\bibitem{Juliako2020}
\BIBentryALTinterwordspacing
Juliako, \emph{Redact faces with Azure Media Analytics}, 2020 (accessed Apr
  17th, 2020). [Online]. Available:
  \url{https://docs.microsoft.com/en-us/azure/media-services/previous/media-services-face-redaction}
\BIBentrySTDinterwordspacing

\bibitem{sanchez2016cascaded}
E.~S{\'a}nchez-Lozano, B.~Martinez, G.~Tzimiropoulos, and M.~Valstar,
  ``Cascaded continuous regression for real-time incremental face tracking,''
  in \emph{European Conference on Computer Vision}.\hskip 1em plus 0.5em minus
  0.4em\relax Springer, 2016, pp. 645--661.

\bibitem{danelljan2017eco}
M.~Danelljan, G.~Bhat, F.~Shahbaz~Khan, and M.~Felsberg, ``Eco: Efficient
  convolution operators for tracking,'' in \emph{Proceedings of the IEEE
  conference on computer vision and pattern recognition}, 2017, pp. 6638--6646.

\bibitem{henriques2014high}
J.~F. Henriques, R.~Caseiro, P.~Martins, and J.~Batista, ``High-speed tracking
  with kernelized correlation filters,'' \emph{IEEE transactions on pattern
  analysis and machine intelligence}, vol.~37, no.~3, pp. 583--596, 2014.

\bibitem{yu2016poi}
F.~Yu, W.~Li, Q.~Li, Y.~Liu, X.~Shi, and J.~Yan, ``Poi: Multiple object
  tracking with high performance detection and appearance feature,'' in
  \emph{European Conference on Computer Vision}.\hskip 1em plus 0.5em minus
  0.4em\relax Springer, 2016, pp. 36--42.

\bibitem{shen2018tracklet}
H.~Shen, L.~Huang, C.~Huang, and W.~Xu, ``Tracklet association tracker: An
  end-to-end learning-based association approach for multi-object tracking,''
  \emph{arXiv preprint arXiv:1808.01562}, 2018.

\bibitem{huang2008robust}
C.~Huang, B.~Wu, and R.~Nevatia, ``Robust object tracking by hierarchical
  association of detection responses,'' in \emph{European Conference on
  Computer Vision}.\hskip 1em plus 0.5em minus 0.4em\relax Springer, 2008, pp.
  788--801.

\bibitem{zhang2014technology}
T.~Zhang and H.~M. Gomes, ``Technology survey on video face tracking,'' in
  \emph{Imaging and Multimedia Analytics in a Web and Mobile World 2014}, vol.
  9027.\hskip 1em plus 0.5em minus 0.4em\relax International Society for Optics
  and Photonics, 2014, p. 90270F.

\bibitem{frey2007clustering}
B.~J. Frey and D.~Dueck, ``Clustering by passing messages between data
  points,'' \emph{science}, vol. 315, no. 5814, pp. 972--976, 2007.

\bibitem{ji20133d}
S.~Ji, W.~Xu, M.~Yang, and K.~Yu, ``3d convolutional neural networks for human
  action recognition,'' \emph{IEEE transactions on pattern analysis and machine
  intelligence}, vol.~35, no.~1, pp. 221--231, 2013.

\bibitem{yue2015beyond}
J.~Yue-Hei~Ng, M.~Hausknecht, S.~Vijayanarasimhan, O.~Vinyals, R.~Monga, and
  G.~Toderici, ``Beyond short snippets: Deep networks for video
  classification,'' in \emph{Proceedings of the IEEE conference on computer
  vision and pattern recognition}, 2015, pp. 4694--4702.

\bibitem{karpathy2014large}
A.~Karpathy, G.~Toderici, S.~Shetty, T.~Leung, R.~Sukthankar, and L.~Fei-Fei,
  ``Large-scale video classification with convolutional neural networks,'' in
  \emph{Proceedings of the IEEE conference on Computer Vision and Pattern
  Recognition}, 2014, pp. 1725--1732.

\bibitem{tran2015learning}
D.~Tran, L.~Bourdev, R.~Fergus, L.~Torresani, and M.~Paluri, ``Learning
  spatiotemporal features with 3d convolutional networks,'' in
  \emph{Proceedings of the IEEE international conference on computer vision},
  2015, pp. 4489--4497.

\bibitem{berclaz2011multiple}
J.~Berclaz, F.~Fleuret, E.~Turetken, and P.~Fua, ``Multiple object tracking
  using k-shortest paths optimization,'' \emph{IEEE transactions on pattern
  analysis and machine intelligence}, vol.~33, no.~9, pp. 1806--1819, 2011.

\bibitem{zamir2012gmcp}
A.~R. Zamir, A.~Dehghan, and M.~Shah, ``Gmcp-tracker: Global multi-object
  tracking using generalized minimum clique graphs,'' in \emph{European
  Conference on Computer Vision}.\hskip 1em plus 0.5em minus 0.4em\relax
  Springer, 2012, pp. 343--356.

\bibitem{mueller2017context}
M.~Mueller, N.~Smith, and B.~Ghanem, ``Context-aware correlation filter
  tracking,'' in \emph{Proceedings of the IEEE Conference on Computer Vision
  and Pattern Recognition}, 2017, pp. 1396--1404.

\bibitem{valmadre2017end}
J.~Valmadre, L.~Bertinetto, J.~Henriques, A.~Vedaldi, and P.~H. Torr,
  ``End-to-end representation learning for correlation filter based tracking,''
  in \emph{Proceedings of the IEEE Conference on Computer Vision and Pattern
  Recognition}, 2017, pp. 2805--2813.

\bibitem{liu2016structural}
S.~Liu, T.~Zhang, X.~Cao, and C.~Xu, ``Structural correlation filter for robust
  visual tracking,'' in \emph{Proceedings of the IEEE Conference on Computer
  Vision and Pattern Recognition}, 2016, pp. 4312--4320.

\bibitem{zhang2016tracking}
S.~Zhang, Y.~Gong, J.-B. Huang, J.~Lim, J.~Wang, N.~Ahuja, and M.-H. Yang,
  ``Tracking persons-of-interest via adaptive discriminative features,'' in
  \emph{European conference on computer vision}.\hskip 1em plus 0.5em minus
  0.4em\relax Springer, 2016, pp. 415--433.

\bibitem{bochinski2017high}
E.~Bochinski, V.~Eiselein, and T.~Sikora, ``High-speed tracking-by-detection
  without using image information,'' in \emph{2017 14th IEEE International
  Conference on Advanced Video and Signal Based Surveillance (AVSS)}.\hskip 1em
  plus 0.5em minus 0.4em\relax IEEE, 2017, pp. 1--6.

\bibitem{sridhar2015fast}
S.~Sridhar, F.~Mueller, A.~Oulasvirta, and C.~Theobalt, ``Fast and robust hand
  tracking using detection-guided optimization,'' in \emph{Proceedings of the
  IEEE Conference on Computer Vision and Pattern Recognition}, 2015, pp.
  3213--3221.

\bibitem{weinzaepfel2015learning}
P.~Weinzaepfel, Z.~Harchaoui, and C.~Schmid, ``Learning to track for
  spatio-temporal action localization,'' in \emph{Proceedings of the IEEE
  international conference on computer vision}, 2015, pp. 3164--3172.

\bibitem{zhang2016joint}
Z.~Zhang, P.~Luo, C.~C. Loy, and X.~Tang, ``Joint face representation
  adaptation and clustering in videos,'' in \emph{European conference on
  computer vision}.\hskip 1em plus 0.5em minus 0.4em\relax Springer, 2016, pp.
  236--251.

\bibitem{insafutdinov2017arttrack}
E.~Insafutdinov, M.~Andriluka, L.~Pishchulin, S.~Tang, E.~Levinkov, B.~Andres,
  and B.~Schiele, ``Arttrack: Articulated multi-person tracking in the wild,''
  in \emph{Proceedings of the IEEE Conference on Computer Vision and Pattern
  Recognition}, 2017, pp. 6457--6465.

\bibitem{lin2018prior}
C.-C. Lin and Y.~Hung, ``A prior-less method for multi-face tracking in
  unconstrained videos,'' in \emph{Proceedings of the IEEE Conference on
  Computer Vision and Pattern Recognition}, 2018, pp. 538--547.

\bibitem{wang2018cosface}
H.~Wang, Y.~Wang, Z.~Zhou, X.~Ji, D.~Gong, J.~Zhou, Z.~Li, and W.~Liu,
  ``Cosface: Large margin cosine loss for deep face recognition,'' in
  \emph{Proceedings of the IEEE Conference on Computer Vision and Pattern
  Recognition}, 2018, pp. 5265--5274.

\bibitem{tang2018pyramidbox}
X.~Tang, D.~K. Du, Z.~He, and J.~Liu, ``Pyramidbox: A context-assisted single
  shot face detector,'' in \emph{Proceedings of the European Conference on
  Computer Vision (ECCV)}, 2018, pp. 797--813.

\bibitem{deng2019arcface}
J.~Deng, J.~Guo, N.~Xue, and S.~Zafeiriou, ``Arcface: Additive angular margin
  loss for deep face recognition,'' in \emph{Proceedings of the IEEE Conference
  on Computer Vision and Pattern Recognition}, 2019, pp. 4690--4699.

\bibitem{ester1996density}
M.~Ester, H.-P. Kriegel, J.~Sander, and X.~Xu, ``Density-based spatial
  clustering of applications with noise,'' in \emph{Int. Conf. Knowledge
  Discovery and Data Mining}, vol. 240, 1996, p.~6.

\bibitem{pal1995cluster}
N.~R. Pal and J.~C. Bezdek, ``On cluster validity for the fuzzy c-means
  model,'' \emph{IEEE Transactions on Fuzzy systems}, vol.~3, no.~3, pp.
  370--379, 1995.

\bibitem{wang2013multi}
C.-D. Wang, J.-H. Lai, C.~Y. Suen, and J.-Y. Zhu, ``Multi-exemplar affinity
  propagation,'' \emph{IEEE transactions on pattern analysis and machine
  intelligence}, vol.~35, no.~9, pp. 2223--2237, 2013.

\bibitem{cao2015diversity}
X.~Cao, C.~Zhang, H.~Fu, S.~Liu, and H.~Zhang, ``Diversity-induced multi-view
  subspace clustering,'' in \emph{Proceedings of the IEEE conference on
  computer vision and pattern recognition}, 2015, pp. 586--594.

\bibitem{sun2014incremental}
L.~Sun and C.~Guo, ``Incremental affinity propagation clustering based on
  message passing,'' \emph{IEEE Transactions on Knowledge and Data
  Engineering}, vol.~26, no.~11, pp. 2731--2744, 2014.

\bibitem{chen2016supervised}
D.~Chen, G.~Hua, F.~Wen, and J.~Sun, ``Supervised transformer network for
  efficient face detection,'' in \emph{European Conference on Computer
  Vision}.\hskip 1em plus 0.5em minus 0.4em\relax Springer, 2016, pp. 122--138.

\bibitem{qin2016joint}
H.~Qin, J.~Yan, X.~Li, and X.~Hu, ``Joint training of cascaded cnn for face
  detection,'' in \emph{Proceedings of the IEEE Conference on Computer Vision
  and Pattern Recognition}, 2016, pp. 3456--3465.

\bibitem{neubeck2006efficient}
A.~Neubeck and L.~Van~Gool, ``Efficient non-maximum suppression,'' in
  \emph{18th International Conference on Pattern Recognition (ICPR'06)},
  vol.~3.\hskip 1em plus 0.5em minus 0.4em\relax IEEE, 2006, pp. 850--855.

\bibitem{bewley2016simple}
A.~Bewley, Z.~Ge, L.~Ott, F.~Ramos, and B.~Upcroft, ``Simple online and
  realtime tracking,'' in \emph{2016 IEEE International Conference on Image
  Processing (ICIP)}.\hskip 1em plus 0.5em minus 0.4em\relax IEEE, 2016, pp.
  3464--3468.

\bibitem{bertinetto2016fully}
L.~Bertinetto, J.~Valmadre, J.~F. Henriques, A.~Vedaldi, and P.~H. Torr,
  ``Fully-convolutional siamese networks for object tracking,'' in
  \emph{European conference on computer vision}.\hskip 1em plus 0.5em minus
  0.4em\relax Springer, 2016, pp. 850--865.

\bibitem{guo2017learning}
Q.~Guo, W.~Feng, C.~Zhou, R.~Huang, L.~Wan, and S.~Wang, ``Learning dynamic
  siamese network for visual object tracking,'' in \emph{Proceedings of the
  IEEE international conference on computer vision}, 2017, pp. 1763--1771.

\bibitem{wang2018learning}
Q.~Wang, Z.~Teng, J.~Xing, J.~Gao, W.~Hu, and S.~Maybank, ``Learning
  attentions: residual attentional siamese network for high performance online
  visual tracking,'' in \emph{Proceedings of the IEEE conference on computer
  vision and pattern recognition}, 2018, pp. 4854--4863.

\bibitem{hirscher2016multiple}
T.~Hirscher, A.~Scheel, S.~Reuter, and K.~Dietmayer, ``Multiple extended object
  tracking using gaussian processes,'' in \emph{2016 19th International
  Conference on Information Fusion (FUSION)}.\hskip 1em plus 0.5em minus
  0.4em\relax IEEE, 2016, pp. 868--875.

\bibitem{hou2007real}
S.~Hou, A.~Galata, F.~Caillette, N.~Thacker, and P.~Bromiley, ``Real-time body
  tracking using a gaussian process latent variable model,'' in \emph{2007 IEEE
  11th International Conference on Computer Vision}.\hskip 1em plus 0.5em minus
  0.4em\relax IEEE, 2007, pp. 1--8.

\bibitem{yang2013incremental}
C.~Yang, L.~Bruzzone, R.~Guan, L.~Lu, and Y.~Liang, ``Incremental and
  decremental affinity propagation for semisupervised clustering in
  multispectral images,'' \emph{IEEE Transactions on Geoscience and Remote
  Sensing}, vol.~51, no.~3, pp. 1666--1679, 2013.

\bibitem{ciuperca2016empirical}
G.~Ciuperca and Z.~Salloum, ``Empirical likelihood test for high-dimensional
  two-sample model,'' \emph{Journal of Statistical Planning and Inference},
  vol. 178, pp. 37--60, 2016.

\bibitem{fukumizu2004dimensionality}
K.~Fukumizu, F.~R. Bach, and M.~I. Jordan, ``Dimensionality reduction for
  supervised learning with reproducing kernel hilbert spaces,'' \emph{Journal
  of Machine Learning Research}, vol.~5, no. Jan, pp. 73--99, 2004.

\bibitem{ding2019linear}
L.~Ding, Z.~Liu, Y.~Li, S.~Liao, Y.~Liu, P.~Yang, G.~Yu, L.~Shao, and X.~Gao,
  ``Linear kernel tests via empirical likelihood for high-dimensional data,''
  in \emph{Proceedings of the AAAI Conference on Artificial Intelligence},
  vol.~33, 2019, pp. 3454--3461.

\bibitem{gretton2012kernel}
A.~Gretton, K.~M. Borgwardt, M.~J. Rasch, B.~Sch{\"o}lkopf, and A.~Smola, ``A
  kernel two-sample test,'' \emph{Journal of Machine Learning Research},
  vol.~13, no. Mar, pp. 723--773, 2012.

\bibitem{owen2001empirical}
A.~B. Owen, \emph{Empirical likelihood}.\hskip 1em plus 0.5em minus 0.4em\relax
  Chapman and Hall/CRC, 2001.

\bibitem{leal2015motchallenge}
L.~Leal-Taix{\'e}, A.~Milan, I.~Reid, S.~Roth, and K.~Schindler, ``Motchallenge
  2015: Towards a benchmark for multi-target tracking,'' \emph{arXiv preprint
  arXiv:1504.01942}, 2015.

\bibitem{milan2016mot16}
A.~Milan, L.~Leal-Taix{\'e}, I.~Reid, S.~Roth, and K.~Schindler, ``Mot16: A
  benchmark for multi-object tracking,'' \emph{arXiv preprint
  arXiv:1603.00831}, 2016.

\bibitem{bernardin2008evaluating}
K.~Bernardin and R.~Stiefelhagen, ``Evaluating multiple object tracking
  performance: the clear mot metrics,'' \emph{EURASIP Journal on Image and
  Video Processing}, vol. 2008, pp. 1--10, 2008.

\bibitem{du2016online}
D.~Du, H.~Qi, W.~Li, L.~Wen, Q.~Huang, and S.~Lyu, ``Online deformable object
  tracking based on structure-aware hyper-graph,'' \emph{IEEE Transactions on
  Image Processing}, vol.~25, no.~8, pp. 3572--3584, 2016.

\end{thebibliography}

\end{document}